\pdfoutput=1
\documentclass[11pt]{article}

\usepackage{amsmath}
\usepackage{amssymb}
\usepackage{amsfonts}
\usepackage{amsthm}
\usepackage{mathtools}
\usepackage[mathscr]{eucal}

\usepackage[margin=1in]{geometry}
\usepackage{graphicx}
\usepackage{url}
\usepackage{multirow}
\usepackage{booktabs}

\usepackage[numbers,sort&compress]{natbib}
\usepackage{hyperref}
\usepackage[nameinlink,noabbrev,capitalize]{cleveref}
\hypersetup{colorlinks=true, linkcolor=blue, citecolor=blue, urlcolor=blue}

\usepackage[shortlabels]{enumitem}
\usepackage{mathrsfs}
\usepackage{thm-restate}

\usepackage[dvipsnames,svgnames]{xcolor}

\usepackage{float,makecell}

\usepackage[ruled,vlined,algo2e]{algorithm2e}
\newtheorem{theorem}{Theorem}[section]
\newtheorem{informalthm}{Informal Theorem}[section]
\newtheorem{lemma}[theorem]{Lemma}
\newtheorem{proposition}[theorem]{Proposition}

\newtheorem{assumption}{Assumption}

\theoremstyle{definition}
\newtheorem{definition}[theorem]{Definition}

\theoremstyle{remark}

\LinesNumbered
\SetKwInput{KwParam}{parameters}
\SetKwInput{KwIn}{input}
\SetKwInput{KwOut}{output}
\DontPrintSemicolon
\SetAlgoVlined

\SetCommentSty{mycommfont}

\crefname{equation}{Eq.}{Eqs.}
\Crefname{equation}{Eq.}{Eqs.}
\crefname{theorem}{Theorem}{Theorems}
\Crefname{theorem}{Theorem}{Theorems}
\crefname{assumption}{Assumption}{Assumptions}
\Crefname{assumption}{Assumption}{Assumptions}
\crefname{claim}{Claim}{Claims}
\Crefname{claim}{Claim}{Claims}
\crefname{lemma}{Lemma}{Lemmas}
\Crefname{lemma}{Lemma}{Lemmas}
\crefname{algocf}{Algorithm}{Algorithms}
\Crefname{algocf}{Algorithm}{Algorithms}
\crefname{table}{Table}{Tables}
\Crefname{table}{Table}{Tables}
\crefname{section}{Section}{Sections}
\Crefname{section}{Section}{Sections}
\crefname{definition}{Definition}{Definitions}
\Crefname{definition}{Definition}{Definitions}
\crefname{appendix}{Appendix}{Appendices}
\Crefname{appendix}{Appendix}{Appendices}
\crefname{proposition}{Proposition}{Propositions}
\Crefname{proposition}{Proposition}{Propositions}
\crefname{informalthm}{Informal Theorem}{Informal Theorems}
\Crefname{informalthm}{Informal Theorem}{Informal Theorems}

\def\ddefloop#1{\ifx\ddefloop#1\else\ddef{#1}\expandafter\ddefloop\fi}
\def\ddef#1{\expandafter\def\csname bb#1\endcsname{\ensuremath{\mathbb{#1}}}}
\ddefloop ABCDEFGHIJKLMNOPQRSTUVWXYZ\ddefloop
\def\ddefloop#1{\ifx\ddefloop#1\else\ddef{#1}\expandafter\ddefloop\fi}
\def\ddef#1{\expandafter\def\csname b#1\endcsname{\ensuremath{\mathbf{#1}}}}
\ddefloop ABCDEFGHIJKLMNOPQRSTUVWXYZ\ddefloop
\def\ddef#1{\expandafter\def\csname sf#1\endcsname{\ensuremath{\mathsf{#1}}}}
\ddefloop ABCDEFGHIJKLMNOPQRSTUVWXYZ\ddefloop
\def\ddef#1{\expandafter\def\csname c#1\endcsname{\ensuremath{\mathcal{#1}}}}
\ddefloop ABCDEFGHIJKLMNOPQRSTUVWXYZ\ddefloop
\def\ddef#1{\expandafter\def\csname h#1\endcsname{\ensuremath{\widehat{#1}}}}
\ddefloop ABCDEFGHIJKLMNOPQRSTUVWXYZ\ddefloop
\def\ddef#1{\expandafter\def\csname hc#1\endcsname{\ensuremath{\widehat{\mathcal{#1}}}}}
\ddefloop ABCDEFGHIJKLMNOPQRSTUVWXYZ\ddefloop
\def\ddef#1{\expandafter\def\csname t#1\endcsname{\ensuremath{\widetilde{#1}}}}
\ddefloop ABCDEFGHIJKLMNOPQRSTUVWXYZ\ddefloop
\def\ddef#1{\expandafter\def\csname tc#1\endcsname{\ensuremath{\widetilde{\mathcal{#1}}}}}
\ddefloop ABCDEFGHIJKLMNOPQRSTUVWXYZ\ddefloop
\def\ddefloop#1{\ifx\ddefloop#1\else\ddef{#1}\expandafter\ddefloop\fi}
\def\ddef#1{\expandafter\def\csname scr#1\endcsname{\ensuremath{\mathscr{#1}}}}
\ddefloop ABCDEFGHIJKLMNOPQRSTUVWXYZ\ddefloop

\DeclarePairedDelimiterX{\dvr}[2]{(}{)}{%
  #1\;\delimsize\|\;#2%
}
\newcommand{\KL}{\mathsf{KL}\dvr*}
\newcommand{\TV}[2]{\mathsf{TV}\left( #1,\; #2 \right)}

\newcommand{\dist}{P}
\newcommand{\seq}{\mathbf}
\newcommand{\tree}{\mathbf}
\newcommand{\oracle}{\mathcal{O}}
\newcommand{\condoracle}{\oracle^{\mathsf{cond}}}
\newcommand{\ball}{\cB}
\newcommand{\off}[1]{\mathfrak{G}_{#1}}
\newcommand{\rade}{\mathfrak{R}}
\newcommand{\nvar}{v_t}

\usepackage{bbm}
\renewcommand{\P}{\mathbb P}
\DeclareMathOperator{\En}{\mathbb{E}}
\DeclareMathOperator{\Var}{\mathsf{Var}}
\newcommand{\ind}{\mathbbm{1}}
\newcommand{\reg}{\mathbf{Reg}}
\newcommand{\rel}{\mathbf{Rel}}
\newcommand{\eps}{\varepsilon}

\newcommand{\norm}[1]{\left\| #1 \right\|}
\DeclareMathOperator*{\argmin}{\arg\!\min}
\DeclareMathOperator{\supp}{supp}
\DeclareMathOperator{\poly}{poly}
\renewcommand{\^}[1]{^{({#1})}}
\newcommand{\Alg}{\mathsf{Alg}}
\newcommand{\tAlg}{\widetilde{\mathsf{Alg}}}
\newcommand{\Unif}{\mathsf{Unif}}
\newcommand{\Ber}{\mathsf{Ber}}

\newcommand{\glmat}{A}
\newcommand{\glvec}{h}
\newcommand{\glfield}{u}

\title{Learning with Simulators:\\ No Regret in a Computationally Bounded World}

\author{
  Sasha Voitovych\footnote{equal contribution}\\
  MIT\\
  \small{\texttt{voitovyc@mit.edu}}
  \and
  Abhishek Shetty\footnotemark[1]\\
  MIT\\
  \small{\texttt{shetty@mit.edu}}
  \and
  Noah Golowich\\
  Microsoft Research\\
  \small{\texttt{nzg@cs.utexas.edu}}
  \and
  Alexander Rakhlin\\
  MIT\\
  \small{\texttt{rakhlin@mit.edu}}
}

\date{}

\begin{document}

\maketitle

\begin{abstract}%
Understanding the minimal assumptions necessary for generalization is the fundamental question in learning theory. Unfortunately, most results rely heavily on independence (or some proxy thereof) of the data-generating process, while results for strongly dependent data are far more limited. Towards addressing this gap, we introduce the framework of \emph{simulatable processes}, where the learner has access to a simulator that approximates the distribution generating the data (which may be an arbitrarily complex and dependent
process). Surprisingly, given access to such a simulator, we show that we can recover
the same learning guarantees as in the classical setting with independent data, namely, error bounds that depend on the VC dimension. Further, we use this framework to study the power of conditional sampling and show strict statistical and computational advantages in this setting.
As a highlight of our framework, we exhibit a single algorithm that \emph{simultaneously} learns any given VC class under all processes samplable in bounded polynomial time, with regret controlled by the time-bounded Kolmogorov complexity of the process. This provides a significant conceptual broadening of the classical PAC model.
\end{abstract}

\section{Introduction}

\label{sec:intro}

A major focus of learning theory has been to understand what structural properties of data are sufficient for efficient learning.
One central assumption that has emerged is that of \emph{independence}, namely, that data points are generated independently from a common population distribution.
This has enabled a rich theory of statistical learning, with celebrated results characterizing learnability under independence, such as via the notion of \emph{VC dimension}~\citep{vapnik1971uniform, vapnik1974theory, blumer1989learnability}. Unfortunately, independence is often hard to verify, or even justify, in practice.
In fact, many natural data-generating processes exhibit rich dependencies across time: today's weather patterns affect tomorrow's; the distribution of a city's population this year depends on that of last year; the spread of disease depends on past infections.

These examples highlight that independence is not a justifiable assumption, but they also point towards a potential resolution to this quagmire.
A tacit assumption in the natural sciences is that natural processes are governed by simple, efficient mechanisms---a belief often articulated as the \emph{extended Church--Turing thesis}~\citep{sep-church-turing, deutsch1985quantum}.
In addition to this abstract belief, we often have explicit simulators for such processes, developed through domain expertise.
This is exemplified by the classical study of simulation-based inference in statistics and physics~\citep{marin2012approximate, sisson2018handbook,deistler2025simulation}, where complex data-generating processes are modeled via simulatable dynamics.
For instance, the evolution of physical systems is often modeled via statistical mechanics processes such as Glauber dynamics~\citep{glauber1963time} and PDEs; opinion dynamics in social networks can be modeled via stochastic processes~\citep{jackson2008social}.
More recently, the success of generative models---large language models~\citep{brown2020language,chowdhery2022palm} and diffusion models~\citep{ho2020denoising, ramesh2022hierarchical, saharia2022photorealistic}---demonstrates that we can build accurate simulators even for complex sequential data such as text and images.
Moreover, the aforementioned belief about efficiently computable nature has been a stated motivation behind recent pushes toward better AI systems~\citep{hassibis_2,hassibis, sutskevar, hutter2024introduction}.
This leads to the question:
\begin{center}
    \refstepcounter{equation}\label{q:main}%
    \makebox[\textwidth]{%
        \hphantom{(\theequation)}\hfill
        \parbox{0.9\textwidth}{
        \centering
            \emph{Can the belief that natural processes are governed by efficient, possibly unknown mechanisms be leveraged for learning?}
        }%
        \hfill(\theequation)%
    }%
\end{center}

To set the stage for learning without independence, we consider the most well-studied model for sequential learning, namely,
\emph{online learning}~\citep{littlestone1988learning, cesa2006prediction}: at each round $t$ the learner observes a data point $X_t$ and predicts its label, with the goal of competing with respect to a prespecified class $\cF$ of labeling functions.
Online learning permits learning even under adversarially generated data, providing the promise of learning with no independence.
However, this robustness incurs a significant cost: even simple hypothesis classes, such as thresholds, which are easy to learn under independence, are not learnable in the fully adversarial model.
The issue is further exacerbated by computational barriers indicating that classes that are efficiently learnable under independence may not be efficiently learnable in the online setting~\citep{hazan2016computational}.
This cements the folk belief that some notion of independence in the data is necessary for efficient learning to be possible.
In particular, this belief suggests that learning should be strictly harder for strongly dependent processes.

In this paper, we introduce a new framework for learning under complex dependencies, which we term \emph{simulatable processes}, and show that this framework suffices to recover learnability at essentially the same statistical and computational complexity as learning under independence.
Our work can be viewed as part of a long line of work on beyond-worst-case analysis of online learning, where the adversarial nature of the data is relaxed.
A number of proposed relaxations allow learning all Vapnik--Chervonenkis (VC) function classes (which are precisely the classes learnable under independence), often efficiently~\citep{rakhlin2011online,block2022smoothed,haghtalab2024smoothed,montasser2025beyond}.
Our framework differs in that it allows learning under \emph{general} dependent processes, provided the learner has access to an (approximate) simulator for the covariate-generating process. The main takeaway of our work is as follows:
\begin{center}
    \parbox{0.9\textwidth}{
    \centering
        \emph{Under simulatable processes, learning has the same statistical and computational complexity as learning under independence.}
    }
\end{center}
A highlight of our framework is that it provides a lens to take advantage of the belief that nature is computationally bounded.
Under this lens, using a connection to time-bounded Kolmogorov complexity, we show that all VC classes are learnable under \emph{unknown} polynomial-time samplable processes, leading to a significant conceptual broadening of the classical PAC model under the widely believed
extended Church--Turing thesis \citep{sep-church-turing}.
This can be seen as a formalization of the statement that, while efficiently computable processes themselves might not be learnable,\footnote{This is indeed the case under widely held beliefs in computational complexity, e.g., existence of one-way functions.} we can learn to predict under these processes.

\paragraph{Notation.} We let $\log$ be the binary logarithm. For an integer $n$, we let $[n]$ denote the set $\{1,\ldots,n\}$. We use boldface notation for sequences. For a sequence $\seq s = (s_1,\ldots,s_n)$, we use $\seq s_{a:b}$ to denote the subsequence $(s_a, s_{a+1},\ldots,s_b)$; for $I \subset [n]$, we use $\seq s_I$ to denote the subsequence $(s_i)_{i \in I}$. For a distribution $P \in \Delta(\cX^n)$ and subsets $I, J \subset [n]$ of indices, we use $P_I$ to denote the marginal of $P$ on the subsequence $\seq X_I$; for any $\seq x_{J}$, we use $P_I(\cdot \mid \seq x_J)$ to denote the conditional distribution of the subsequence $\seq X_I$ given $\seq X_J = \seq x_J$.
We adopt the standard non-asymptotic big-$O$ notation or $f \lesssim g$ to denote that $f \le C g$ for some universal constant $C>0$; similarly, we write $f \ge \Omega(g)$ or $f \gtrsim g$.

\section{Our Framework and Contributions}

We begin by formally introducing our learning setting.
We assume the learner and nature play a $T$-round game. Before the start of the game, nature chooses a covariate-generating distribution $\dist^\star \in \Delta(\cX^T)$ and samples $(X_{1},\ldots,X_T) \sim \dist^\star$. Let us re-emphasize that we make no assumptions on the distribution $\dist^\star$, which can be arbitrarily complex and dependent across time. Then, at each round $t \le T$, nature reveals $X_t$ to the learner. The learner makes a prediction $\hat Y_t \in \{0,1\}$ for the label of $X_t$, after which it learns the true label $Y_t$, and suffers loss $1$ if $\hat Y_t \neq Y_t$.
The objective of the learner is to minimize \emph{expected regret} with respect to a comparator class of labeling functions $\cF$:
\[
\En \reg(\cF,T) = \En \left[\sum_{t=1}^T \ind\{\hat Y_t \neq Y_t\} - \inf_{f\in\cF} \sum_{t=1}^T \ind\{f(X_t) \neq Y_t\} \right].
\]
The expectation above is taken with respect to the generation from $\dist^\star$ and the internal randomness of the learner. Throughout, we operate under the assumption that the labels are chosen \emph{obliviously}.
\begin{assumption}
\label{assumption:f-star}
    There exists a function $f^\star$ such that, for every $t\in [T]$, $Y_t = f^\star(X_t)$ and $f^\star$ is fixed before the draw of $(X_1,\ldots,X_T)$. We refer to $f^\star$ as the ground truth labeling function.
\end{assumption}
The assumption formalizes the belief that, while the data may have complex structure and dependencies, the ground truth is a fixed property of the world: it does not change in response to the learner's predictions or the realization of past covariates.
While some of our results hold under milder conditions, we defer further discussion of this assumption to~\cref{sec:adaptive-labels}.
With this in mind, the question becomes: under what distributions $\dist^\star$ is learning possible? If we let $\dist^\star$ range over all possible distributions over sequences,
the problem reduces to that of classical online learning, and thus, suffers from the aforementioned issues. Intuitively, the goal is to formalize that $\dist^\star$ is not worst-case, while still permitting complex time dependencies.

\paragraph{A motivating problem: learning against computationally bounded nature.}
Consider the most ambitious instantiation of the motivating question in~\eqref{q:main}:
can we hope to learn against \emph{any} unknown nature, provided only that nature is computationally bounded?
Concretely, assume the covariates $X_1,\ldots,X_T$ are produced by an unknown distribution $\dist^\star$ samplable in time $p(T)$ for some polynomial $p$.
Can we learn efficiently in this setting, \emph{simultaneously} against every such $\dist^\star$ and every labeling function $f^\star$? Our first headline contribution answers in the affirmative, as long as $f^\star \in \cF$ and $\cF$ has VC dimension in $O(1)$.
\begin{informalthm}[\cref{thm:samplable-nature}]
\label{informalthm:samplable-nature}
Let $\cF$ be a class of VC dimension $d$. Then, there exists a single algorithm that, given a realizable ERM oracle for $\cF$, runs in time $\poly(T^d, p(T))$ and, for every $p$-time samplable $\dist^\star$ and every $f^\star \in \cF$, achieves $o(T)$ regret, where the rate depends only on $\dist^\star$ and $d$.
\end{informalthm}
The above can be viewed as a learning-theoretic analog of Levin search~\cite{levin1973universal}. Indeed, just as Levin search pays only a bounded overhead relative to the fastest program for any task,
our algorithm pays only a bounded overhead relative to the sampling complexity of $\dist^\star$ for any $p$-time process, without any knowledge of $\dist^\star$ in advance.

The rate in~\cref{informalthm:samplable-nature} depends on the \emph{time-bounded Kolmogorov complexity of $\dist^\star$}. The proof leverages results from the Kolmogorov complexity literature~\citep{antunes2006computational,lu2022optimal,hirahara2023learning} and constructs a \emph{single} distribution $\dist$, which we refer to as the \emph{simulator},
that can be sampled from efficiently, and, in a certain sense, \emph{dominates} every other distribution samplable in $p$-time. To show that access to such simulators suffices to recover learnability of VC classes,
we introduce our framework of learning under simulatable processes.

\paragraph{Learning under simulatable processes.}  A key defining feature of our \emph{simulatable process} framework is that we assume access to a sampling oracle for a distribution $\dist$ that we will tacitly assume is close to $\dist^\star$.
\begin{assumption}[Simulator access]
\label{assumption:p-star}
    The learner has access to a simulatable process $\dist$ that approximates $\dist^\star$. That is, the learner has access to an oracle $\oracle_{\dist}$ that, when queried, returns an independent sample (a trajectory $(X_1,\ldots,X_T)$) from $\dist$. We refer to $\dist$ simply as the simulator.
\end{assumption}

Other than having simulation access to the distribution of covariates, it is important to note that we are making no further restriction on the distributions $\dist$ and $\dist^\star$. In particular, this setting generalizes the standard computational learning theory setting of learning under a known, product distribution such as the uniform distribution or Gaussian distribution.
Further, we note here that we only assume sampling access and not access to the density. In addition to being a natural assumption in many practical, physical settings, we will later see an explicit example of a distribution where we can efficiently sample from but not evaluate the density (\cref{sec:kolmogorov}).

We separate \emph{realizable} ($f^\star \in \cF$) and \emph{agnostic} ($ f^{\star}$ unconstrained) learning settings. Our first contribution is then showing that all VC classes are learnable under simulatable processes, even given an approximate simulator, provided $f^\star$ belongs to the class $\cF$.
\begin{informalthm}[\cref{thm:stat-dist-trans}]
\label{informalthm:stat-dist-trans}
    For any class of VC dimension $d$, given access to $N$ samples from an approximate simulator $\dist$, and any $f^\star \in \cF$, there is an algorithm that achieves expected regret upper bound
    \[
        \En \reg(\cF,T) \lesssim d \log(T)  + \frac{dT}{N} + d \cdot \KL{\dist^\star}{\dist} +\KL{\dist^\star}{\dist} \left( \frac{T}{\log\left(1+ N/T\right)}\right).
    \]
\end{informalthm}
The first term above corresponds to the optimal error of learning under independence, which is achieved by empirical risk minimization (ERM). This already hints at a key characteristic of our setting. In our setting, ERM can fail spectacularly as it does not exploit the simulator, and thus, standard lower bounds for online learning with ERM apply. The remaining three terms above account for the price of sampling (finite $N$) and the price of approximation ($\KL{\dist^\star}{\dist}> 0$) in our setting.

Another satisfactory aspect is the dependence of the quality of the simulator on the Kullback--Leibler (KL) divergence, which is, perhaps, the most natural metric on distributions over sequences due to the tensorization property and connections to density estimation. It is also typically smaller than other measures considered in the literature, such as coverage (i.e., uniform density ratio bounds). Further, KL divergence has practical significance in the modern context of generative models, where LLMs are typically trained with the log loss for perhaps similar reasons.

Two aspects of the result above might seem unsatisfactory at first glance: the number of samples needed (with approximate simulator access) and the realizability requirement.
Unfortunately, both these aspects are inherent to this setting.
In \cref{thm:realizable-lb}, we show that the tradeoff between the number of samples and the quality of approximation in~\cref{informalthm:stat-dist-trans} is essentially tight.
Further, even when the simulator is exact, in \cref{thm:agnostic-impossibility}, we show that agnostic learning is statistically impossible even for simple classes with VC dimension $1$.
This indicates that \cref{informalthm:stat-dist-trans} is the strongest possible result for algorithms in the setting discussed so far. In particular,~\cref{informalthm:samplable-nature} follows from~\cref{informalthm:stat-dist-trans} by instantiating $\dist$ as a universal polytime samplable distribution (\cref{thm:universal-dist}).

\paragraph{A second motivating problem: learning without mixing.} While~\cref{informalthm:samplable-nature} is a very general learnability statement, this generality comes at a cost: the guarantee can be pessimistic when the process has rich structure we could hope to exploit. In particular,~\cref{informalthm:samplable-nature} requires realizability of $f^\star$, and the running time has exponential dependence on the VC dimension. To explore whether more structure can alleviate these concerns, we consider an example of learning under Markov chains. As concrete examples, in this paper we consider two families of processes motivated by evolutions of physical systems: linear dynamical systems (LDS), and Glauber dynamics. Here, we focus on the setting of LDS due to its simplicity, and defer the discussion of Glauber dynamics to~\cref{sec:glauber}. In particular, suppose $X_{1},\ldots,X_T \in \bbR^n$ evolves according to a, possibly unknown, LDS, which can be described as follows:
\begin{align}
    \label{eq:lds}
X_{t+1} = A^\star X_t + \eta_t, \quad \text{where} \quad \eta_t \overset{\text{i.i.d.}}{\sim} \cN(0,\sigma^2 I).
\end{align}
Learning under Markov chains is a well-studied problem~\citep{aldous1995markovian,bartlett1994exploiting,gamarnik1999extension,bshouty2005learning,arpe2008agnostically,kanade2015mcmc,dagan2019learning,cornacchia2026benefits,bresler2017learning}.
Linear dynamical systems is a canonical such model, and it exhibits non-independence if the chain is non-mixing (which happens when the spectral radius of $A^\star$ is close to $1$), disallowing approaches that extract independence structure from mixing~\citep{yu1994rates,mohri2010stability,kuznetsov2017generalization,gamarnik1999extension}.
Similarly, naive applications of oracle-efficient methods from smoothed analysis~\citep{haghtalab2024smoothed} yield bounds with exponential dependence on the dimension. The reason is that a Gaussian density with unknown mean can only be dominated by a fixed density up to a constant that scales exponentially in $n$.
However, the process is parametric with only at most $O(n^2)$ free parameters, so one could hope for complexity that scales with $\poly(n)$.
\begin{informalthm}[\cref{thm:lds}]
\label{informalthm:lds}
    Let $B>0$ be arbitrary, and let $\rho(\cdot)$ denote the spectral radius of a matrix. Then, for any class $\cF$ of VC dimension $d$, there is a $\poly(n,T)$-time algorithm that, given an agnostic ERM oracle for $\cF$, achieves sublinear expected regret in the agnostic setting for all $\dist^\star$ evolving according to~\cref{eq:lds} with $A^\star \in [-B,B]^{n \times n}$ with $\rho(A^\star) \le 1$.
\end{informalthm}
We highlight that the above result has \emph{no dependence on the mixing time} of the chain and avoids exponential dependence on dimension $n$.
This is a qualitatively new guarantee for learning from strongly-dependent Markov chain data, not attainable via prior mixing- or smoothness-based approaches.
The proof of the above result leverages the fact that all distributions in the LDS family can be simultaneously and efficiently simulated in a \emph{conditional sense}.
This motivates the second part of our framework, which works with conditional simulators.

\paragraph{Efficient algorithms with conditional simulators.} Broadly inspired both by autoregressive generative models and simulations for physical processes, we consider the setting when the learner has access to a \emph{conditional} sampling oracle for $\dist$.

\begin{assumption}[Conditional Simulator access]
\label{assumption:p-star-cond}
    The learner has access to a conditionally simulatable process $\dist$ that approximates $\dist^\star$. That is, the learner has access to an oracle $\condoracle_{\dist}$ that, when queried on a prefix $\seq x_{1:t} \in \cX^t$, returns an independent sample from $\dist(\cdot\mid \seq x_{1:t})$.
\end{assumption}

Perhaps surprisingly, this additional power allows us to circumvent the statistical and computational barriers discussed above.
In particular, both realizable and agnostic learning become statistically tractable, and admit oracle-efficient algorithms.

\begin{informalthm}[\cref{thm:cond-oracle-eff}]
\label{thm:informal-efficient}
    Assume conditional sampling access to simulator $\dist$. Then, given access to an agnostic ERM oracle for the class $\cF$ of VC dimension $d$, there exists an algorithm running in $\poly(T)$ time, with regret in the agnostic setting
    \[
    \En \reg(\cF,T) \lesssim \sqrt{d T (1 +\KL{\dist^\star}{\dist})}.
    \]
\end{informalthm}
The above theorem matches PAC-style learnability guarantees, both statistically (regret controlled by the VC dimension) and computationally (ERM oracle suffices). This result is based on relaxations with random playout~\citep{rakhlin2012relax}. Specifically, we extend the results of~\cite{rakhlin2015hierarchies} to the setting where only an approximate conditional simulator access is available. To obtain~\cref{informalthm:lds} (and the analogous result for Glauber dynamics, \cref{thm:glauber}) from the above, we show that there exists a distribution $\dist$ that simulates any LDS or Glauber dynamic up to bounded KL and such that conditional sampling from $\dist$ can be implemented efficiently.

\begin{table}
\centering
\setlength{\tabcolsep}{5pt}
\renewcommand{\arraystretch}{1.0}
\begin{tabular}{@{}llccc@{}}
\toprule
    Setting & & \makecell{Unconditional \\ samples} & \makecell{Conditional \\ samples} & Refs. \\
\midrule
\multirow{2}{*}{\makecell[l]{Realizable, exact}}
& Regret  & $\tilde O(d)$ & $O(\sqrt{dT})$ & \makecell{Thms.~\ref{thm:stat-dist-trans}, \ref{thm:cond-oracle-eff}} \\
& Runtime & $\exp(\Omega(d))$ & $\poly(T,d)$ & \makecell{Thms.~\ref{thm:realizable-erm-oracle-lb}, \ref{thm:cond-oracle-eff}} \\
\midrule
\multirow{2}{*}{\makecell[l]{Agnostic, exact}}
& Regret  & $\Omega(T)$ & $O(\sqrt{dT})$ & \makecell{Thms.~\ref{thm:agnostic-impossibility}, \ref{thm:cond-oracle-eff}} \\
& Runtime & --- & $\poly(T,d)$ & Thm.~\ref{thm:cond-oracle-eff} \\
\midrule
\multirow{2}{*}{\makecell[l]{Realizable, approximate}}
& Regret  & $\tilde O(d + d\cdot \mathsf{KL})$ & $O(\sqrt{dT (1 + \mathsf{KL})})$ & \makecell{Thms.~\ref{thm:stat-dist-trans}, \ref{thm:cond-oracle-eff}} \\
& Runtime & $\exp(\Omega(T))$ & $\poly(T,d)$ & \makecell{Thms.~\ref{thm:realizable-lb}, \ref{thm:cond-oracle-eff}} \\
\midrule
\multirow{2}{*}{\makecell[l]{Agnostic, approximate}}
& Regret  & $\Omega(T)$ & $O(\sqrt{dT (1 + \mathsf{KL})})$ & \makecell{Thms.~\ref{thm:agnostic-impossibility}, \ref{thm:cond-oracle-eff}} \\
& Runtime & --- & $\poly(T,d)$ & Thm.~\ref{thm:cond-oracle-eff} \\
\bottomrule
\end{tabular}
\vspace{5pt}
    \caption{
    Summary of regret and runtime bounds for algorithms for classes of VC dimension $d$, with access to unconditional or conditional samples from the exact or approximate simulator $\dist$, and with $\mathsf{KL}:= \KL{\dist^\star}{\dist}$ denoting the Kullback--Leibler divergence between the true process and the simulator.
    Any regret or runtime upper bound listed under \emph{Unconditional samples} also applies under \emph{Conditional samples}, since the former oracle is strictly weaker.
    }
\label{tab:rates}
\end{table}

\paragraph{Trade-off: conditional vs.\ unconditional simulators.}
Taken together, the results above suggest a trade-off.
Unconditional samples are often easy to obtain, but lead to algorithms that are not efficient in general. Conditional samples are harder to obtain, but give rise to oracle-efficient $\poly(T,d)$ algorithms.
This trade-off is already visible in~\cref{informalthm:samplable-nature}: the universal simulator $\dist$ is hard to sample from conditionally, assuming the existence of one-way functions (see~\cref{sec:kolmogorov}). By assuming extra structure on nature's generative process---for instance, Glauber dynamics or LDS---we can circumvent this obstacle and obtain efficient algorithms using conditional simulators. This trade-off is further exemplified in~\cref{tab:rates}.

\subsection{Related Work}

Our paper is in a long line of work on sequential prediction and beyond worst case analysis thereof \cite{shetty2024learning}.
Further, our particular framework has connections to notions such as learning with predictions, simulators in reinforcement learning, simulation-based inference and computational perspectives on learning.
We defer a thorough discussion to \cref{sec:related_work}.

\section{Learning with Unconditional Simulators}

\label{sec:simulation}

\subsection{VC Classes Are Learnable under Simulatable Processes}

Our results depend on the VC dimension of the class $\cF$, which is a canonical complexity measure in learning theory~\citep{vapnik1998statistical}. In particular, bounded VC is equivalent to learnability under independent processes.  Thus, VC is precisely the ``correct'' complexity when learning under independence. We formally define it in~\cref{def:vc}. Throughout, we reserve $d$ to denote the VC dimension of the class at hand. The main result of this section is that bounded VC dimension suffices to recover logarithmic regret in our setting, even with an approximate simulator. A priori, this result is surprising, since the process $\dist^\star$ can be arbitrarily complex and dependent.

\begin{restatable}{theorem}{StatDistTrans}
\label{thm:stat-dist-trans}
    Let $\cF$ be a VC dimension $d$ class, and let $\dist^\star$ be an arbitrary distribution in $\Delta(\cX^T)$. Let $N \in\bbN$ be arbitrary. Then, there exists an algorithm that makes at most $N$ oracle calls to $\oracle_{\dist}$, at most $N^{O(d)}$ calls to a realizable ERM oracle w.r.t. $\cF$, and achieves regret \[
        \En \reg \lesssim d \log T + \frac{dT}{N} + d \cdot \KL{\dist^\star}{\dist} + \frac{T}{\log\left(1 + N/T\right)} \cdot \KL{\dist^\star}{\dist}.
    \]
\end{restatable}
The bound quantitatively captures the tradeoff between the quality of the simulator, the number of samples, and the regret.
The last three terms can be seen as aspects of the \emph{price of information} in our setting.
To appreciate this, we focus on three regimes.
\begin{enumerate}[(i)]
    \item In the case where we have an exact simulator (i.e., $\dist = \dist^\star$), the
regret is dominated by the first two terms.~\cref{thm:stat-dist-trans} then says that with just $N = T$ samples from the simulator, we can achieve near optimal regret of $\tilde O(d)$, matching the independent covariate setting.
\item In the case where we only have an approximate simulator (i.e. $ \KL{\dist^\star}{\dist} > 0$), with enough samples  ($N \ge T 2^T$), we can still achieve regret
$\En \reg(\cF,T) \lesssim d\log T + d \KL{\dist^\star}{\dist},$
which has mild dependence on KL. In particular, KL $\le o(T)$ suffices for sublinear regret.
\item In the case where we have a poor simulator and only a few samples, the regret is dominated by $\KL{\dist^\star}{\dist}T /\log(N)$ term.
\end{enumerate}
Somewhat surprisingly, the upper bound in~\cref{thm:stat-dist-trans} is essentially tight, and we can show that all the error terms are indeed necessary (see~\cref{thm:realizable-lb}).

\subsection{Proof Sketch for~\texorpdfstring{\cref{thm:stat-dist-trans}}{Theorem \getrefnumber{thm:stat-dist-trans}}}
\begin{algorithm2e}[hbt!]
\DontPrintSemicolon
\caption{$\mathrm{MultiCover}$}
\label{algo:multi-cover}
\KwIn{Sampling oracle $\oracle_{\dist}$ for $\dist$; realizable ERM oracle w.r.t.\ $\cF$}
\KwParam{Learning rate $\eta>0$; number of levels $L$; number of samples per level $\{N_\ell\}_{\ell \in[L]}$}
\tcc*[h]{Phase 1: Build multi-scale covers}

\For{$\ell \gets 1$ \KwTo $L$}{
  Sample $\seq Z^{\ell} \sim \bar{\dist}^{\otimes N_\ell}$ with $\bar \dist := \frac{1}{T} \sum_{t=1}^T \dist_t$ \label{line:Z-samples-def}\tcp*{$N_\ell$ queries to $\oracle_\dist$}
  Let $\cG_\ell := \cG(\seq Z^\ell)$ be an improper cover of $\cF$ on $\seq Z^{\ell}$ as in~\cref{eq:cg-ell} \tcp*{$N_\ell^{O(d)}$ ERM queries}
}

$\cG \gets \bigcup_{\ell=1}^L \cG_\ell$

\tcc*[h]{Phase 2: Exponential weights over $\cG$}

For $\ell \in [L]$, $g\in\cG_{\ell}$, set $w_1(g) \gets \frac{1}{L |\cG_{\ell}|}$ \label{line:init-weights-def} \tcp*{Initialize the weights}
\For{$t \gets 1$ \KwTo $T$}{
Sample $g_t$ according to the probability vector $\{w_t(g)\}_{g\in\cG}$\;
  Observe $X_t$ and play $\hat Y_t \gets g_t(X_t)$\;
  Observe $Y_t$ and suffer loss $\ind\left\{Y_t \neq \hat Y_t\right\}$\;
  Update the weights for each $g$: $ w_{t+1}(g) \gets \frac{w_t(g) \exp(- \eta \ind\{g(X_t) \neq Y_t\})}{\sum_{g'\in\cG} w_t(g')\exp(- \eta \ind\{g'(X_t) \neq Y_t\})}$
}

\end{algorithm2e}

The regret bound in~\cref{thm:stat-dist-trans} is attained by~\cref{algo:multi-cover}. The first phase of the algorithm is based on the following idea: by drawing samples from the uniform mixture of the marginals of $\dist$, $\bar \dist := \frac{1}{T}\sum_{t=1}^T \dist_t$,
we can identify a small subset of functions that represents the behavior of any $f^\star$ on the draw $\seq X \sim \dist$ well. 
Concretely, for each $\ell \in [L]$, our algorithm draws $N_\ell$ i.i.d. samples \[\seq Z^\ell := (Z_1^\ell,\ldots,Z_{N_\ell}^\ell) \sim \bar \dist^{\otimes N_\ell}.\] Then, it constructs an \emph{improper cover} $\cG_\ell$ by iterating through all labelings of $\seq Z^\ell$ realized by $\cF$ (a set we denote by $\cF\vert_{\seq Z^\ell}$) and running an optimal PAC learner on each labeling. Specifically, we use the Majority-of-Three learner of~\cite{aden2024majority}. Then, using the population error guarantee for the PAC learner, we can show that any fixed function $f \in \cF$ is well-represented with respect to the distribution $\bar \dist$, and hence, on average under a draw $\seq X \sim \dist$.
 Formally, let $\Alg_{\mathsf{M3}}(S) \in \{0,1\}^\cX$ be the output of the Majority-of-Three algorithm on a labeled sample $S$; additionally, for functions $f,g$, a probability measure $Q \in \Delta(\cX)$ and a sequence $\seq X \in \cX^T$ of covariates, let \[\norm{f - g}_Q := \En_{X\sim Q} |f(X)-g(X)|, \qquad \norm{f-g}_{\seq X} := \frac{1}{T} \sum_{t=1}^T |f(X_t) - g(X_t)|\] be the average disagreement between $f$ and $g$ under $Q$ and under the empirical measure of $\seq X$, respectively. Then, the preceding discussion is formalized in the following statement.
\begin{restatable}{lemma}{CoverAnalysis}
\label{lemma:cover-analysis}
    Let $\seq Z  =\left(Z_1,\ldots,Z_N\right) \sim \bar \dist^{\otimes N}$, and $\cG(\seq Z)$ be an improper cover of $\cF$ constructed as follows
    \begin{equation}
    \label{eq:cg-ell}
    \cG(\seq Z) := \left\{\Alg_{\mathsf{M3}}\left((Z_i,y_i)_{i \in [N]}\right) \colon \left(y_i\right)_{i \in [N]} \in \left.\cF\right\vert_{\seq Z}\right\}.
    \end{equation}
        Then, for any $f\in\cF$, we have
      \[ \En_{\seq Z} \En_{\seq X \sim \dist}  \min_{g \in \cG(\seq Z)} \norm{f-g}_{\seq X} \lesssim\frac{d}{N}.\]
    \end{restatable}
We note that cover $\cG(\seq Z)$ can be constructed in time $N^{O(d)}$ with access to a realizable ERM oracle. Indeed, from the celebrated Sauer--Shelah lemma, the number of distinct realizable labelings in $\cF\vert_{\seq Z}$ is on the order $N^{O(d)}$ and can be enumerated in time $N^{O(d)}$~\citep{blumer1989learnability}. Also, the Majority-of-Three learner~\citep{aden2024majority} requires $O(1)$ realizable ERM queries per labeling. 

Having obtained such a cover for $\cF$, a natural next step is to run exponential weights over the cover. The bound above indeed guarantees a small approximation error along a draw $\seq X\sim\dist$. However, we need to control the approximation error under $\dist^\star$ instead.
A natural approach is to apply a change of measure argument to~\cref{lemma:cover-analysis} directly by controlling the tail behavior of ${\mathrm{d}\dist^\star/\mathrm{d}\dist}$ (as done in similar contexts, e.g.,~\cite{block2023sample}). In particular, we can find $M$ such that $\mathrm{d}\dist^\star/\mathrm{d}\dist\le M$ w.h.p., and apply change of measure to~\cref{lemma:cover-analysis} as follows:
\begin{align*}
\En_{\seq Z} \En_{\seq X\sim\dist^\star} \min_{g \in \cG(\seq Z)} \norm{f-g}_{\seq X}
&\le \En_{\seq Z}  \En_{\seq X\sim\dist^\star} \min_{g \in \cG(\seq Z)} \norm{f-g}_{\seq X} \ind\left\{\frac{\mathrm{d}\dist^\star}{\mathrm{d}\dist} \le M\right\}\\
&\qquad + \En_{\seq Z} \En_{\seq X\sim\dist^\star} \min_{g \in \cG(\seq Z)} \norm{f-g}_{\seq X} \ind\left\{\frac{\mathrm{d}\dist^\star}{\mathrm{d}\dist} > M\right\} \\
&\lesssim  M \cdot \En_{\seq Z}\En_{\seq X\sim\dist} \min_{g \in \cG(\seq Z)}  \norm{f-g}_{\seq X} +  \P_{\seq X\sim\dist^\star} \left[\frac{\mathrm{d}\dist^\star}{\mathrm{d}\dist} > M\right] \\
&\lesssim M \cdot \frac{d}{N} + \frac{\KL{\dist^\star}{\dist}}{\log (M)},
\end{align*}
where the last step uses~\cref{lemma:kl-markov}. Balancing $M$ in the above gives an upper bound that decays \emph{only logarithmically} in $N$. Intuitively, this is due to the fact that KL only bounds a \emph{log-moment} of the density ratio.
Such scaling would be too weak for fast rates, and would only lead to a $\sqrt{T}$ upper bound on regret.

This is exactly where the multi-scale nature of the algorithm becomes essential. Intuitively, the natural strategy fails because it pays for the tail of the log density ratio between $\dist^\star$ and $\dist$ rather than its expectation. The multi-scale structure in~\cref{algo:multi-cover} allows the regret to depend on the loss of the ``correct'' expert with respect to the \emph{realized} density ratio. In particular, by choosing $\left\{N_\ell\right\}_{\ell\in\left[L\right]}$ appropriately, we can ensure that, with high probability, there exists a random index $\ell^\star$ such that $\log\left(N_{\ell^\star}/T\right)\propto\log\left(\mathrm{d}\dist^\star/\mathrm{d}\dist\right)$.
The expected approximation error under $\dist^\star$ of the best expert in level $\ell^\star$ can be upper bounded \emph{uniformly} as follows:
\begin{align*}
\En_{\seq X\sim\dist^\star}\min_{g\in\cG_{\ell^\star}}\norm{f-g}_{\seq X}
&\lesssim\sum_{\ell\in\left[L\right]}\En_{\seq X\sim\dist^\star}\min_{g\in\cG_\ell}\norm{f-g}_{\seq X}\ind\left\{\frac{\mathrm{d}\dist^\star}{\mathrm{d}\dist}\le\frac{N_{\ell}}{T}\right\}\\
&\lesssim\sum_{\ell\in\left[L\right]}\frac{N_\ell}{T}\cdot\En_{\seq X\sim\dist}\min_{g\in\cG_\ell}\norm{f-g}_{\seq X}
\lesssim\frac{dL}{T},
\end{align*}
where the last step follows from~\cref{lemma:cover-analysis}. With this control of the approximation term, we apply standard results for exponential weights with non-uniform priors (\hyperref[line:init-weights-def]{Line~\ref{line:init-weights-def}}), and get:
\begin{align*}
\reg(\cF,T)
&\lesssim \log\left(L\left|\cG_{\ell^\star}\right|\right)+T\cdot\min_{g\in\cG_{\ell^\star}}\norm{f-g}_{\seq X}\\
&\lesssim d\log\left(N_{\ell^\star}\right)+T\cdot\min_{g\in\cG_{\ell^\star}}\norm{f-g}_{\seq X}\\
&\approx d\log\left(T\frac{\mathrm{d}\dist^\star}{\mathrm{d}\dist}\right)+T\cdot\min_{g\in\cG_{\ell^\star}}\norm{f-g}_{\seq X},
\end{align*}
where the first transition uses Sauer--Shelah lemma.
The first term becomes a KL term after taking expectations, and the second term is bounded in expectation by the derivation above.

\subsection{Limitations of Unconditional Samples}

\cref{algo:multi-cover} has several limitations. First, we make a realizability assumption to prove a regret bound in~\cref{thm:stat-dist-trans}. Second,~\cref{algo:multi-cover} is inefficient in the approximate simulator regime, as it requires drawing $\exp(\Omega(T))$ samples. We show that, surprisingly, both these limitations are inherent to learning with only unconditional samples from $\dist$.
\begin{restatable}{theorem}{RealizableLB}
    \label{thm:realizable-lb}
    For any $d \in \bbN$, $N \ge 2$, and $D \in \bbN \cup \{0\}$, there exists a class $\cF$ of VC dimension $d$, such that, for any learning algorithm that draws at most $N$ unconditional samples from $\oracle_\dist$, there exists a choice of distributions $(\dist, \dist^\star)$ with $\KL{\dist^\star}{\dist} \le D$ and $f^\star \in \cF$ such that
    \[
    \En \reg(\cF, T) \gtrsim \left(d + \frac{d T}{N} + D + D  \frac{T}{\log N} \right)\land T
    \]
\end{restatable}
In particular, the lower bound above implies that even in the moderate KL regime (for instance, $\KL{\dist^\star}{\dist} \le 0.1$), an \emph{exponential} in $T$ number of samples is necessary to achieve sub-polynomial regret for all VC classes. As mentioned above, with exact simulator access, $T$ unconditional samples suffice, and~\cref{algo:multi-cover} runs in $T^{O(d)}$ time, which is polynomial in the regime $d \le O(1)$. A natural question to ask is whether $\poly(T,d)$ runtime is also possible. We answer this in the negative. The following result can be viewed as an extension of~\cite{hazan2016computational} to our setting.
\begin{restatable}{theorem}{RealizableERMOracleLB}
\label{thm:realizable-erm-oracle-lb}
There is a sufficiently small constant $c$ so that the following holds. For any $d, N, T\in \bbN$ satisfying $TN \leq 2^{cd}$, and any learning algorithm $\Alg$ that draws at most $N$ unconditional samples from $\dist^\star$ and makes at most $N$ agnostic ERM oracle queries, there exists a class $\cF$ of VC dimension at most $d$ and a distribution $\dist^\star$ so that $\Alg$ suffers regret $\En \reg(\cF, T) \gtrsim T$.
\end{restatable}
Together,~\cref{thm:realizable-lb,thm:realizable-erm-oracle-lb} highlight computational challenges in learning with unconditional samples, and show that~\cref{algo:multi-cover} is essentially statistically and computationally optimal in the realizable setting with unconditional samples. We now turn our attention to agnostic learning. Typically, agnostic learning and realizable learning are qualitatively similar in statistical complexity, though the rates may differ.  Surprisingly, this is not the case in our setting: we show that we can force linear regret in the agnostic setting
on the class of thresholds for any learner that
draws finitely many unconditional samples in the agnostic learning setting.
\begin{restatable}{theorem}{AgnosticImpossibility}
\label{thm:agnostic-impossibility}
    Let $\cF$ be the class of thresholds on the $[0,1]$ interval. For any $N \ge 1$, and for any learning algorithm that draws at most $N$ unconditional samples from $\dist^\star$, there exists a choice of a distribution $\dist^\star$ and $f^\star \not \in \cF$ such that $\En \reg(\cF,T) \gtrsim T.$
\end{restatable}
This further highlights the difference between the present setting and conceptually similar settings such as transductive learning and smoothed online learning. In both those cases, agnostic learning is possible for all VC classes, and the rates are similar to the rates under independence.

\section{Learning with Conditional Simulators}
\label{sec:cond}
    Thus far, we have established that all VC classes are learnable under simulatable processes in the realizable setting.
    In order to extend these results beyond the realizable setting, we consider a stronger notion of simulatable processes that allows \emph{conditional} sampling. Specifically, given a conditional sampling oracle for $\dist$ that approximates $\dist^\star$, we show that agnostic learning of all VC classes is possible. The main result of this section follows.
    It is achieved by~\cref{algo:dist-transductive-reduction}.
    \begin{restatable}{theorem}{CondOracleEff}
    \label{thm:cond-oracle-eff}
    Let $\cF$ be a VC dimension $d$ class. Then, there exists an algorithm that runs in $\poly(T)$ time and, given access to an agnostic ERM oracle over $\cF$ and a conditional sampling oracle $\condoracle_{\dist}$, achieves the following expected regret bound in the agnostic setting:
    \[
    \En \reg(\cF,T) \lesssim \sqrt{dT (1 + \KL{\dist^\star}{\dist})}.
    \]
    \end{restatable}
When $\dist=\dist^\star$ (i.e., exact conditional sampling access), it is
known that one can achieve the optimal $O(\sqrt{dT})$ regret for VC dimension $d$ classes using
\emph{relaxations with random playout}~\citep{rakhlin2012relax,rakhlin2015hierarchies}. Our contribution is to extend these results to a setting where the playout distribution
$\dist$ is only a coarse approximation for $\dist^\star$.

Recall that in the unconditional setting with an inexact simulator, achieving $T^{1-\Omega(1)}$ regret provably requires exponentially many samples~(\cref{thm:realizable-lb}), and even with exact simulator, $\exp(\Omega(d))$ runtime is necessary. Remarkably,~\cref{thm:cond-oracle-eff} attains a $\sqrt{T}$ rate using only $\poly(T)$ conditional queries and ERM calls, which qualitatively matches the complexity of learning under independence. The dependence on $\KL{\dist^\star}{\dist}$ is also qualitatively tight, and the above result can be complemented with a $\sqrt{T \cdot \KL{\dist^\star}{\dist}}$ lower bound for the agnostic setting against even inefficient algorithms (\cref{prop:approx-agnostic-lb}). Additionally, $\sqrt{T}$-type regret is known to be optimal for oracle-efficient algorithms even in the realizable transductive setting~\citep{hazan2016computational}. 

\paragraph{Warm-up: conditional Rademacher relaxation.}
A convenient way to design oracle-efficient algorithms is via the
\emph{relaxation} framework~\citep{rakhlin2012relax}. Intuitively, relaxations maintain at each time $t$ an upper bound on the ``remaining difficulty'' of the problem. As a warm-up, consider a natural strategy based on the classical conditional Rademacher complexity relaxation~\cite{rakhlin2012relax,rakhlin2015hierarchies}. Formally, for a dataset $\seq S_{1:t} := (\seq X_{1:t},\seq Y_{1:t})$, we let $L_t(f) := \sum_{s=1}^t \ind\{f(X_s)\neq Y_s\}$, and define: 
\begin{equation}
\label{eq:rel-def}
\rel_0(\seq S_{1:t})
:=
\En_{\substack{\seq X_{t+1:T} \sim \dist(\cdot \mid \seq X_{1:t}),\\
\boldsymbol{\varepsilon}_{t+1:T}\sim \Unif(\{\pm 1\}^{T-t})}}
\sup_{f\in \cF}\left[ \sum_{s=t+1}^T \varepsilon_s \frac{2f(X_s)-1}{2} - L_t(f)\right].
\end{equation}
Note that the above can be evaluated using conditional samples from $\dist$, potentially inefficiently. 
For a history $\seq S_{1:t-1}$ and a new covariate $X_t$, a strategy of the learner is defined via a solution to the following minimax optimization problem:
\begin{equation}
\argmin_{q \in \Delta([0,1])}\ \sup_{Y_t \in \{0,1\}}
\left[
\En_{\hat Y\sim q}\ind\{\hat Y\neq Y_t\}
+ \rel_0(\seq S_{1:t}) \right]
\end{equation}
In the setting where the simulator is exact, the standard analyses apply~\cite{rakhlin2012relax,rakhlin2015hierarchies}, and yield $\sqrt{dT}$ expected regret upper bound for this strategy. The case of inexact simulator with $\KL{\dist^\star}{\dist} > 0$ is more interesting. A natural strategy that can be used to upper bound regret under $\dist^\star \neq \dist$ is to upper bound regret under $\dist$ and perform a change-of-measure argument to $\dist^\star$ using \emph{Donsker--Varadhan} variational formula for KL~\cite{polyanskiy2025information}: 
\begin{equation}
\label{eq:dv}
\En_{\dist^\star} \reg(\cF,T) \le \inf_{\eta >0} \left\{\frac{1}{\eta} \log \En_{\dist} \exp(\eta \reg(\cF,T)) + \frac{1}{\eta} \KL{\dist^\star}{\dist}\right\}.
\end{equation}
Using simple boundedness $|\reg(\cF,T)| \le T$ and Hoeffding's inequality, we can upper bound 
\[
\frac{1}{\eta} \log \En_{\dist} \exp(\eta \reg(\cF,T)) \lesssim \En_{\dist} \reg(\cF,T) + \eta T^2 \lesssim \sqrt{dT} + \eta T^2.
\]
Optimizing $\eta$ then yields: 
\[
\En_{\dist^\star} \reg(\cF,T) \lesssim \sqrt{dT} + T \sqrt{\KL{\dist^\star}{\dist}},
\]
which is only meaningful in the small-KL regime $\KL{\dist^\star}{\dist} \ll 1$.  To improve upon the above, we would need to show that $\reg(\cF,T)$ satisfies meaningful concentration properties under $\dist$ beyond simple boundedness. Unfortunately, it is unclear whether $\reg(\cF,T)$ concentrates meaningfully, with the main obstacle being an absence of any independence structure in $\dist$, which disallows naive applications of martingale concentration inequalities. However, this gives us insight in how to change the definition of~\cref{eq:rel-def} to enable regret concentration.

\paragraph{Log-MGF relaxation.} We consider an alternative relaxation definition, which allows us to control the moment generating function (MGF) of regret directly. The relaxation is parameterized by an inverse temperature parameter $\eta > 0$:
\begin{equation}
\label{eq:rel-log-mgf}
\rel(\seq S_{1:t})
:=
\frac{1}{\eta} \log \En_{\substack{\seq X_{t+1:T} \sim \dist(\cdot \mid \seq X_{1:t}),\\
\boldsymbol{\varepsilon}_{t+1:T}\sim \Unif(\{\pm 1\}^{T-t})}} \exp\left(\eta
\sup_{f\in \cF}\left[ \sum_{s=t+1}^T \varepsilon_s \frac{2f(X_s)-1}{2} - L_t(f)\right] \right).
\end{equation}
We fix the inverse temperature $\eta > 0$ and suppress it from the notation. Note that $\lim_{\eta \to 0} \rel = \rel_0$. The strategy of the learner is defined as a solution to:
\begin{equation}
    \label{eq:relaxation-mgf-strategy}
    q_t(\seq S_{1:t-1}, X_t) = \argmin_{q \in \Delta([0,1])}\ \sup_{Y_t \in \{0,1\}} 
\left[
\En_{\hat Y\sim q}\ind\{\hat Y\neq Y_t\}
+ \rel(\seq S_{1:t}) \right]
\end{equation}
Note that it is non-trivial to compute the above from conditional samples from $\dist$ and agnostic ERM oracle calls, which is why the above is the \emph{idealized} strategy of the learner. We consider the following potential function:
\begin{equation}
\label{eq:mgf-potential}
\Phi_t(\seq S_{1:t}) := \exp\left( \eta \bar L_t + \eta \rel (\seq S_{1:t})\right),
\end{equation}
where $\bar L_t := \sum_{s \le t} \En_{\hat Y_s \sim q_s} \ind\{\hat Y_s \neq Y_s\}$ is the predictable loss process. Note that $\Phi_T = \exp(\eta (\bar L_T - \inf_{f\in\cF} L_T(f)))$. Then, it can be shown that~\cref{eq:relaxation-mgf-strategy} makes $\Phi_t$ a \emph{supermartingale}.
\begin{restatable}{lemma}{RelAdmissibility}
    \label{lemma:rel-admissibility}
    Let $\cF$ be an arbitrary function class, and consider any $t \le T$. Suppose the simulator is exact, that is $\dist = \dist^\star$, and let $q_t$ be as in~\cref{eq:relaxation-mgf-strategy}. Then, for any history $\seq S_{1:t-1}$, we have
\[
\En_{X_t \sim \dist_t^\star(\cdot \mid \seq X_{1:t-1})} \exp\left( \sup_{Y_t}\left[\eta \En_{\hat  Y_t \sim  q_t} \ind\{ \hat Y_t  \neq Y_t\} + \eta \rel(\seq S_{1:t})\right] \right) \le \exp\left( \eta \rel(\seq S_{1:t-1}) \right).
\]
Thus, conditioned on $\seq S_{1:t-1}$, $\En_{X_t} \sup_{Y_t} \Phi_t \le \Phi_{t-1}$, and therefore, $\{\Phi_t\}_{t \le T}$ is a supermartingale.
\end{restatable}
The above allows us to rectify the change-of-measure strategy from~\cref{eq:dv}, since $\Phi_T$ controls the exponential moment of regret, and since $\En \Phi_T \le \Phi_0$ from the supermartingale property.
We arrive at the following lemma.
\begin{restatable}{lemma}{IdealizedRegret}
    \label{lemma:idealized-regret}
    Let $\cF$ be a VC dimension $d$ class. Then, for every $\eta > 0$, an algorithm that plays according to~\cref{eq:relaxation-mgf-strategy} achieves the following expected regret bound in the agnostic setting:
\[
\En_{\dist^\star} \reg(\cF,T) \lesssim \sqrt{dT} + \frac{1}{\eta} \KL{\dist^\star}{\dist} + \eta T.
\]
\end{restatable}
After balancing $\eta$, we obtain $\sqrt{T \KL{\dist^\star}{\dist}}$-type dependence on $\KL{\dist^\star}{\dist}$, which is optimal in general (\cref{prop:approx-agnostic-lb}). It remains to show how the rule in~\cref{eq:relaxation-mgf-strategy} can be implemented efficiently.

\paragraph{Efficient implementation.} We first express the solution to the minimax problem in~\cref{eq:relaxation-mgf-strategy} in closed form in terms of exponential moments of a certain stochastic process. For notational brevity, let $\seq W^t = (\seq X_{t+1:T}, \boldsymbol{\varepsilon}_{t+1:T})$ be the random variable summarizing the randomness left after $X_t$ is revealed. 
\begin{restatable}{proposition}{ClosedFormStrategy}
    \label{prop:closed-form}
    For a history $\seq S_{1:t}$ and a covariate $X_t$, let
    \begin{equation}
        \label{eq:ayw}
    \off{y}(\seq W^t) :=
    \sup_{f\in \cF}\left[ \sum_{s=t+1}^T \varepsilon_s \frac{2f(X_s)-1}{2} - L_{t-1}(f) - \ind \{f(X_t) \neq y\} \right].
    \end{equation}
    Then, $q_t$ as in~\cref{eq:relaxation-mgf-strategy} can be written as
    \[
    q_t(\seq S_{1:t-1}, X_t) = \frac{1}{2} + \frac{1}{2\eta} \log \left(\frac{\En \exp(\eta \off{1}(\seq W^t))}{\En \exp(\eta \off{0}(\seq W^t))}\right)
    \]
\end{restatable}
We will attempt to estimate $q_t$ using polynomially many conditional samples from $\dist$. A priori, it is non-obvious that this is possible since $\exp(\eta \off{1}(\seq W^t))$ can be as large as $\exp(\eta T) \asymp \exp(\sqrt{T})$ in the relevant temperature regime $\eta \asymp 1/\sqrt{T}$. However, we will show that unboundedness is not an issue for two reasons: (1) in the relevant temperature regime, the $\exp(\eta \off{1}(\seq W^t))$ has favorable concentration properties, (2) we only need to estimate the difference of logs of expectations, so the common offset term in the estimation cancels out.

Concretely, we consider a natural plug-in estimator for $\hat q_t$: for i.i.d. samples $\widehat{\seq W}_1^t,\ldots, \widehat{\seq W}_N^t$, we define our estimator as:
\begin{equation}
\label{eq:plug-in}
\hat q_t(\seq S_{1:t-1}, X_t) = \frac{1}{2} + \frac{1}{2\eta} \log \left(\frac{\frac{1}{N} \sum_{j = 1}^N \exp(\eta \off{1}(\widehat{\seq W}^t_j))}{\frac{1}{N} \sum_{j = 1}^N \exp(\eta \off{0}(\widehat{\seq W}^t_j))}\right).
\end{equation}
Note that the above can be computed using $N$ conditional oracle queries to $\dist$ and $O(N)$ agnostic ERM oracle calls. Then, we have the following bound on the MSE of the estimator. 
\begin{restatable}{lemma}{MSEBound}
    \label{lemma:mse-bound}
    For an arbitrary history $\seq S_{1:t-1}$, an arbitrary covariate $X_t$, and $\eta\le 1$, we have
    \[
    \En\left[(\hat q_t - q_t)^2\right] \lesssim \frac{\nvar}{N},
    \quad \text{where} \quad
    \nvar := \max_{y \in \{0,1\}} \frac{\En \exp(2\eta \off{y}(\seq W^t))}{\left(\En \exp(\eta \off{y}(\seq W^t))\right)^2},
    \]
    is the normalized variance of the exponential moment.
\end{restatable}
It now remains to upper bound $\nvar$. Intuitively, both $\off{1}$ and $\off{0}$ have a form of an Rademacher process with an offset of $L_{t-1}(f)$. Since $\nvar$ is normalized, the impact of the offset in the numerator and the denominator cancels out, and we can upper bound $\nvar$ as follows.
\begin{restatable}{lemma}{VEtaBound}
\label{lemma:v-eta}
    For an arbitrary history $\seq S_{1:t-1}$, an arbitrary covariate $X_t$, and $\eta\le 1$, we have
    \[
    \nvar \le \exp(2\eta) \En \exp\left(2\eta \sup_{f} \sum_{s= t+1}^T \varepsilon_s \frac{2f(X_s) - 1}{2}\right) \lesssim 
    \exp\left(O(\eta \sqrt{dT} + \eta^2 T)\right).
    \]
\end{restatable}
Combining everything, we have the following regret bound.
\begin{restatable}{lemma}{RegMGFIntermediate}
\label{lemma:reg-mgf-intermediate}
Let $\cF$ be a VC dimension $d$ class. Then, an algorithm that plays according to~\cref{eq:plug-in} with $N$ samples per step and inverse temperature $\eta$ achieves the following regret bound
\begin{align*}
\En_{\dist^\star} \reg(\cF,T) \lesssim \sqrt{dT} + \eta T + \frac{1}{\eta} \KL{\dist^\star}{\dist} + T \sqrt{\frac{\exp(O(\eta \sqrt{dT} + \eta^2 T))}{N}}.
\end{align*}
\end{restatable} 
Setting $\eta \asymp 1/\sqrt{dT}$ and $N = \poly(T)$ large enough, the above gives: 
\[
\En_{\dist^\star} \reg(\cF,T) \lesssim \sqrt{dT} (1 + \KL{\dist^\star}{\dist}).
\]
The above is short of the regret bound claimed in~\cref{thm:cond-oracle-eff}, due to the linear dependence on $\KL{\dist^\star}{\dist}$. In particular, the above is non-vacuous only when $\KL{\dist^\star}{\dist} \le o(\sqrt{T})$. The above rate can be further sharpened using epoching.

\paragraph{Epoching.}
The bound in~\cref{lemma:reg-mgf-intermediate} is obtained by performing rollouts from the distribution $\dist(\cdot \mid \seq X_{1:t})$ all the way to horizon $T$. When $\dist\neq \dist^\star$, this can be problematic: long rollouts from $\dist$ may quickly drift away from the typical future under $\dist^\star$, making the simulated continuation
uninformative. Intuitively, nature has roughly $\KL{\dist^\star}{\dist}$ bits of freedom, and thus, every $T/\KL{\dist^\star}{\dist}$ steps, the distribution $\dist^\star$ may deviate from $\dist$ significantly. With this intuition, the prediction rule in~\cref{eq:relaxation-mgf-strategy} effectively overfits to noise in the tail of the generation $\seq X_{t+1:T}$.

Epoching is a natural strategy for mitigating compounding errors and also appears in works on predictable sequences~\citep{raman2024online}.
Let $B_1,\ldots,B_K$ be a partition of $[T]$ into $K$ consecutive epochs of equal length $L:=T/K$, and let
$H_k := \bigcup_{i<k} B_i$. The key point is that KL
\emph{tensorizes} across time blocks: by the chain rule for KL divergence,
\[
\frac{1}{K}\sum_{k=1}^K
\En\KL{\dist^\star_{B_k}(\cdot \mid \seq X_{H_k})}{\dist_{B_k}(\cdot \mid \seq X_{H_k})}
=
\frac{1}{K}\KL{\dist^\star}{\dist},
\]
Thus, the average KL discrepancy across epochs is small. We now run the same relaxation-based update within each epoch, using rollouts from $\dist$ that only extend to the end of the current epoch. Applying~\cref{lemma:reg-mgf-intermediate} to each epoch separately gives the following upper bound on expected regret:
\begin{align*}
\sum_{k=1}^K \sqrt{dL}
+
\sum_{k=1}^K
\sqrt{dL} \cdot
\En\KL{\dist^\star_{B_k}(\cdot \mid \seq X_{H_k})}{\dist_{B_k}(\cdot \mid \seq X_{H_k})}=
\sqrt{dTK}
+
\sqrt{d T/K} \cdot \KL{\dist^\star}{\dist}.
\end{align*}
It remains to optimize over $K$. This can be done without the knowledge of $\KL{\dist^\star}{\dist}$ by running exponential weights over different choices of $K$. Combining these ideas yields a prediction strategy given in~\cref{algo:dist-transductive-reduction}.

\begin{algorithm2e}[hbt!]
\caption{Log-MGF relaxations with epoching}
\label{algo:dist-transductive-reduction}
\KwIn{Conditional sampling oracle $\condoracle_\dist$, number of epochs $K$, inverse temperature $\eta > 0$}
Let $L \gets T/K$ be the length of each epoch \tcp*{WLOG, assume $T$ is a multiple of $K$}
\For{$t \gets 1$ \KwTo $T$}{
  Let $k \gets \lceil t/L \rceil$ \tcp*{Index of the current epoch}
  Observe $X_{t}$\;
  Sample $\widehat{\seq W}^t_{i} \sim \dist_{t+1:kL}(\cdot \mid \seq X_{1:t})\otimes \Unif(\{\pm 1\}^{kL-t})$ for $i \in [N]$ \tcp*{$N$ calls to $\condoracle_\dist$}
  Compute $\off{1}(\widehat{\seq W}^t_i)$ and $\off{0}(\widehat{\seq W}^t_i)$ for $i \in [N]$ as in~\cref{eq:ayw} for epoch $k$ \tcp*{$2N$ calls to ERM oracle}
  Compute $\hat q_t$ as in~\cref{eq:plug-in} \;
  Predict $\hat Y_t \sim \hat q_t$\;
  Observe $Y_t$ and suffer loss $\ind\left[\hat Y_t \neq Y_t\right]$\;
}
\end{algorithm2e}

\section{Learning under Families of Processes}

\label{sec:sim-families}

Given our results thus far, the next natural question is: What if no natural simulator for $\dist^\star$ is available? Learning under all $\dist^\star$ simultaneously is impossible (it reduces to classical online learning), but in most settings some information about $\dist^\star$ is available. We model this by assuming $\dist^\star$ lies in some family $\cP$ of processes. Then, the learner can \emph{construct} a simulator $\dist$ for $\cP$ and run either~\cref{algo:multi-cover} or~\ref{algo:dist-transductive-reduction}. A natural choice of $\dist$ is the (approximate) solution to the following minimax optimization problem:
\begin{equation}
\label{eq:univ-compression}
\cR(\cP) := \min_{\dist} \max_{\dist^\star\in \cP} \KL{\dist^\star}{\dist}.
\end{equation}
This quantity is known as \emph{redundancy} in the information theory literature~\citep{polyanskiy2025information}, and it characterizes the universal compression rate with sources coming from the family $\cP$.
This reconceptualizes our results through the lens of compression:
\begin{center}
    \parbox{0.8\textwidth}{\centering \emph{Learning is possible under universally compressible families of sources.}}
\end{center}
The quantity $\cR(\cP)$ has well-known closed forms for many families of parametric processes~\cite{polyanskiy2025information,atteson1999asymptotic}. Combining~\cref{thm:cond-oracle-eff} and~\cref{lemma:linear-kl-lb}, we have the following learnability result in terms of $\cR(\cP)$.
\begin{restatable}{corollary}{LearnabilitySublinearKL}
\label{cor:sim-families}
Let $\cP$ be a family of processes with $\cR(\cP) \le o(T)$, and $\cF$ a bounded VC dimension class. Then, there exists a (possibly inefficient) algorithm that achieves sublinear agnostic regret for every $\dist^\star \in \cP$. Moreover, for every $D \in \bbN$, there exists $\cP$ with $\cR(\cP) \le D \land T$ and a VC class of dimension $1$, such that any algorithm suffers regret $\Omega(D \land T)$ even in the realizable setting.
\end{restatable}
To achieve sublinear regret in the above, it suffices to instantiate~\cref{algo:dist-transductive-reduction} with the simulator $\dist$ taken to be the (approximate) minimizer of~\cref{eq:univ-compression}; in the realizable case, we may instead run~\cref{algo:multi-cover} with $\dist$ as the simulator. Sampling such $\dist$ need not be efficient, however: to obtain efficient algorithms, our framework requires efficient sampling (either conditional or unconditional) from $\dist$. As we detail below, this requirement can be satisfied for many common families of processes.

\subsection{Learning Against Computationally Bounded Nature}

\label{sec:kolmogorov}

Consider an instantiation of the above problem when $\cP$ consists of all polynomial-time samplable processes. This assumption can be broadly motivated from the beyond-worst-case theory of online learning. Indeed, positing that nature runs in bounded polynomial time is one way to quantify the intuitive belief that nature is simple and not worst-case.
A priori this class seems too large to say anything interesting about.
Surprisingly, we show that this is not the case.
Specifically, in this section we show that, for any bound on nature's running time, there exists an algorithm that, for any fixed VC dimension $d$, runs in time polynomial in $T$ and in nature's running time, and achieves an optimal bound with respect to all such strategies of nature. To accomplish this, at a conceptual level, we use complexity-theoretic analogues of coding theorems that establish existence of ``universal'' compressors for time-bounded polynomial processes.

Specifically, consider an online learning game over a countable domain $\mathcal{X}$, which we take to be the set of all finite binary strings (i.e., $\mathcal{X}=\{0,1\}^*$). We let $\Sigma := \{0,1,\#\}$ be the alphabet, where $\#$ is a special delimiter symbol. In this section, we encode elements $\seq x \in \cX^\infty$ as delimiter-separated infinite binary strings. Formally, for $x \in \cX$, let $\langle x \rangle := x \#$, and, for any $\seq x \in \cX^\infty$, let $\langle \seq x \rangle := \langle x_1\rangle \langle x_2\rangle \ldots$ be the encoding.
Accordingly, we view any distribution $\dist \in \Delta(\mathcal{X}^\infty)$ as a distribution on $\Sigma^\infty$ and let $\dist_{1:t}$ denote the distribution of the encoded length-$t$ prefix $\langle \seq X_{1:t} \rangle$.

The complexity-theoretic language in which we state our results is similar to one used in~\cite{hirahara2023learning}. We begin by defining a notion of polytime samplable distributions. Let $1^t$ denote an all-$1$ string of length $t$. Intuitively, a distribution $\dist$ is $p$-time samplable, if the first $t$ elements can be sampled in $p(t)$ time. In the remainder, fix a universal prefix-free Turing machine $U$ that has a write-only tape where the output is recorded.
For a program $\pi \in \{0,1\}^*$ (with $|\pi| \ge 1$), and input $s \in \Sigma^*$, let $U^t(\pi, s)$ denote the contents of the output tape of $U$ after simulating $\pi$ on $s$ for $t$ steps.
\begin{definition}[Time-bounded samplable distribution]
\label[definition]{def:samplable}
    Let $p \colon \mathbb N \to \mathbb N$ be a time bound. Then, distribution $P \in \Delta(\Sigma^\infty)$ is $p$-time samplable iff there exists a program $\pi$ such that, for every $t \ge 1$ and $s \sim \Unif(\{0,1\}^{p(t)})$, we have $ U^{p(t)}(\pi, 1^{t}\# s) \sim P_{1:t}.$
    For a distribution $P \in \Delta(\Sigma^\infty)$, let
    \[
    K^p(P) = \min\left\{|\pi| \colon U^{p(t)}(\pi, 1^{t}\# s) \sim P_{1:t} \text{ for every } t \ge 1, \text{ where } s \sim \Unif(\{0,1\}^{p(t)}) \right\}
    \]
    be the (time-bounded) Kolmogorov complexity of $P$ if $P$ is $p$-time samplable.
\end{definition}
It is well-appreciated in the Kolmogorov complexity literature that there exist \emph{universal} distributions that dominate every other time-bounded samplable distribution~\citep{antunes2006computational,antunes2009worst,lu2022optimal}. In particular, we prove a version of such a statement below, with Lemma 6.9 from~\cite{hirahara2023learning} being the most direct analog (a similar statement is also implicit in~\cite{impagliazzo1990no}).
\begin{restatable}[Universal samplable distribution]{theorem}{UniversalDist}
\label{thm:universal-dist}
    Let $p \colon \mathbb N \to \mathbb N$ be a time-constructible bound. There exists a time bound $p'(t) =\poly(p(t))$ and a sequence of measures $\{\mu^t\}_{t \ge 1}$, with each $\mu^t \in \Delta (\Sigma^*)$ samplable in time $p'(t)$ that satisfy the following. For any $p$-time samplable distribution $P$ and any $\seq x \in \cX^t$ we have
    \[
    \log \frac{P_{1:t}(\langle \seq x \rangle)}{\mu^t(\langle \seq x\rangle)} \le K^p(P) + 2\log K^p(P) + O(1).
    \]
\end{restatable}
For every $T$, the distributions $\mu^T$ defined above can be used as an approximate simulator for \emph{all} nature's strategies  that run in $p(T)$-time to generate $\seq X_{1:T}$. Moreover, note that this simulator controls the worst-case density ratio with respect to any time-bounded process, which is a stronger guarantee than the bounded-KL assumption of~\cref{sec:simulation}. This can be further leveraged for efficiency. As such, it suffices to run~\cref{algo:multi-cover} with $\mu^T$ as the simulator and only a \emph{single layer} ($L=1$) to obtain the following result.
\begin{restatable}{theorem}{SamplableNature}
\label{thm:samplable-nature}
    Suppose the distribution $\dist^\star$ of nature is $p$-time samplable (as per~\Cref{def:samplable}). Let $\cF$ be any class of VC dimension $d$. Then, there exists an algorithm that, given access to a realizable ERM oracle w.r.t. $\cF$, runs in time $\poly(T^d, p(T))$ and for any $f^\star \in \cF$, achieves regret
    \[
    \En \reg (\cF, T ) \lesssim d \log(T) + d \cdot 2^{K^p(\dist^\star) + 2\log K^p(\dist^\star)}.
    \]
\end{restatable}
We briefly mention that~\cref{thm:stat-dist-trans,thm:universal-dist} also imply the existence of an $\exp(O(T))$-time algorithm that achieves regret $\En \reg \lesssim d \log(T) + d K^p(\dist^\star)$.
However, an interesting feature of~\cref{thm:samplable-nature} is that the running time of our algorithm depends only polynomially on nature's running time.
One might hope that the results for conditional samplers (\cref{thm:cond-oracle-eff}) can be used to further improve the dependence on VC dimension. As it turns out, it is computationally hard to conditionally sample from $\mu$. Indeed, the universal distribution $\mu$ is closely related to a time-bounded variant of the probabilistic Kolmogorov complexity~\citep{goldberg2022probabilistic}, a connection made precise in~\cite{hirahara2023learning}. In particular, $\log \mu$ closely approximates probabilistic Kolmogorov complexity. Unfortunately, under standard cryptographic assumptions, it is hard to evaluate time-bounded Kolmogorov complexity, and hence to conditionally sample from universal distributions~\citep{liu2020owf}. Thus, our simulatable processes framework sits in the sweet spot where the efficiency of coding theorems can be algorithmically exploited for learning.

\subsection{Learning under Markov Chains}
\label{sec:markov-chains}

To exploit favorable runtime of algorithms that rely on conditional simulators from~\cref{sec:cond}, we now consider examples where the family $\cP$ has additional structure. A natural example of structured families that are also ubiquitous in nature are \emph{Markov chains}.

\subsubsection{Linear Dynamical Systems}
\label{sec:lds}

Suppose the family $\cP$ consists of all processes that describe the evolution of an unknown \emph{linear dynamical system}. More concretely, suppose $X_1,\ldots,X_{T} \in \bbR^n$ evolve as
\[
X_{t+1} = A X_t + \eta_t, \qquad \eta_t \overset{\text{i.i.d.}}{\sim} \cN(0, \sigma^2 I_n),
\]
with unknown parameter $A \in \Theta := [-B, B]^{n \times n}$, starting point $X_0\in \ball(0, R) := \{x \in \bbR^n \colon \norm{x}_2 \le R\}$, and known noise scale $\sigma > 0$. Without loss of generality, we can assume that the starting point $X_0$ is known. Indeed, if $X_0$ is unknown, we can treat $X_1$ as a starting point at a cost of incurring extra loss $\le 1$ in the first round, and a polynomial blow-up in the magnitude of the starting point $\norm{X_1}_2$. Other than the basic assumption of entry-wise boundedness of $A$, we will additionally assume the \emph{marginal stability} of the system throughout, namely, $\rho(A) \le 1$, where $\rho(\cdot)$ denotes the spectral radius. This assumption ensures that the system is well-behaved~\citep{simchowitz2018learning}. The marginal-stability assumption includes the regime $\rho(A) = 1$, in which the chain has no stationary distribution and the mixing time is undefined~\cite{simchowitz2018learning}; this regime is of particular interest, since naive approaches that extract independence structure from mixing would fail here. Let $\dist^{A}$ denote the law of $\seq X_{1:T}$ under these dynamics. Finally, let \[\pi = \cN(0, B^2)^{n\times n}\] be the i.i.d. Gaussian prior. Then, we instantiate our simulator to be the following mixture distribution:
\begin{equation}
\label{eq:universal-dist-lds}
\dist(\cdot \mid X_0) := \En_{A \sim \pi} \dist^{A}(\cdot \mid X_0).
\end{equation}
Then, it can be shown (see~\cref{lemma:lds-compressor} and~\cref{prop:lds-cond-sampler}) that the above distribution controls the worst-case KL within the family of LDS described above, and, moreover, that conditional sampling from $\dist$ can be implemented efficiently.
\begin{restatable}{theorem}{ThmLDS}
\label{thm:lds}
    Let $\cF$ be a class of VC dimension $d$, and let $R, B \ge 1$ and $\sigma > 0$ be arbitrary. There is a $\poly(n, T)$-time algorithm that, given access to an agnostic ERM oracle for $\cF$, the noise scale $\sigma$, and the magnitude bound $B$, achieves expected regret
    \[
    \En \reg(\cF, T) \lesssim \sqrt{dT} \left(1 + \sqrt{n^3 \log\left(nBT(R/\sigma + 1)\right)}\right)
    \]
    in the agnostic setting, simultaneously for every $\dist^\star = \dist^{A^\star}$ with $A^\star \in \Theta$ satisfying $\rho(A^\star) \le 1$ and initial state $X_0 \in \ball(0,R)$.
\end{restatable}
The result above is related to, although distinct from, the line of work of~\citet{hazan2017learning,hazan2018spectral} on online prediction in linear dynamical systems. There, the learner observes a stream of inputs and predicts the outputs of a latent LDS whose hidden state evolves linearly, competing with the best such system. In our setting the state $X_t$ is instead observed, and the learner competes with the best predictor in any given VC class. As such, the resulting guarantees are incomparable.

It is also instructive to compare the above regret guarantee to works considering estimation of the parameters of an LDS~\citep{simchowitz2018learning,sarkar2019near,faradonbeh2018finite}. An attractive feature of our result is that learning under an unknown LDS is possible even when the parameters of the system are hard to estimate. Indeed, prior works that estimate the parameters of the LDS rely on additional identifiability conditions, e.g., requiring the least eigenvalue of the system's Gramian or the least singular value of $A^\star$ to be bounded away from zero, whereas our result requires marginal stability of the system (i.e., $\rho(A^\star)\le1$).

\subsubsection{Glauber Dynamics}
\label{sec:glauber}

We consider the family of Glauber dynamics for the Ising model on $n$ vertices~\cite{glauber1963time}, which can be defined as follows. The Glauber dynamics trajectory consists of an initial configuration $X_0 \in \{0,1\}^n$ and a sequence of pairs $(\seq I, \seq X) \in ([n] \times \{0,1\}^n)^T$, generated by
\begin{align*}
&I_t \sim \Unif([n]),\\
&X_{t, I_t} \sim \Ber(\sigma(\glfield_{I_t}^\theta(X_{t-1}))),\\
&X_{t, j} = X_{t-1, j} \quad \text{for } j \neq I_t.
\end{align*}
Here $\glmat$ is a symmetric, zero-diagonal $n \times n$ matrix and $\glvec \in \bbR^n$, so that $\theta = (\glmat, \glvec) \in \Theta := [-B, B]^p$ with $p = \binom{n}{2} + n$. The local field is $\glfield_i^\theta(x) = \sum_{j \neq i} \glmat_{ij} x_j + \glvec_i$.
Following the same logic as in~\cref{sec:lds}, we can, without loss of generality, assume that the initial state $X_0$ is known. Then, let $\dist^{\theta}$ denote the joint law of $(\seq I, \seq X)$ under these dynamics. Following the large body of literature on learning under Glauber dynamics~\citep{bresler2017learning}, we will assume that the transition site $I_t$ is observed even when the state is unchanged ($X_t = X_{t+1}$). Let $\pi$ be the uniform distribution on $\Theta$, and define the \emph{simulator}
\[
\dist := \En_{\theta \sim \pi} \dist^{\theta}(\cdot \mid X_0).
\]
Similarly to LDS, it can be shown (see~\cref{lemma:glauber-compressor} and~\cref{prop:glauber-cond-sampler}) that the above distribution controls the worst-case KL within the family of Glauber dynamics, and that conditional sampling from $\dist$ can be implemented efficiently via log-concave sampling~\cite{chewi2024logconcave}.

\begin{restatable}{theorem}{ThmGlauber}
\label{thm:glauber}
    Let $\cF$ be a class of VC dimension $d$, and let $B \ge 1$. There is a $\poly(n, T, B)$-time algorithm that, given access to an agnostic ERM oracle for $\cF$, achieves expected regret
    \[
    \En \reg(\cF, T) \lesssim \sqrt{dT} \cdot \left(1 + \sqrt{n^2 \log(BTn)}\right)
    \]
    in the agnostic setting, simultaneously for every $\dist^\star = \dist^{\theta^\star}$ with $\theta^\star \in \Theta$ and initial state $X_0 \in \{0,1\}^n$.
\end{restatable}
To our knowledge, this work is the first one which considers learning arbitrary VC function classes when covariates evolve according to unknown Glauber dynamics. A conceptually similar work is that of~\citet{chandrasekaran2026learning}, which studies learning under covariates generated by a graphical model; there, however, the examples are drawn i.i.d., whereas in our setting the covariates form a single, temporally dependent trajectory. As with learning under LDS, we highlight that the above result permits learning even when the parameters of the Glauber dynamics are difficult to estimate. In particular, approaches to estimating the parameters of Glauber dynamics have exponential dependence on the $\ell_1$-width, defined as $\max_{k} \sum_{i \neq k} |A_{ik}| + |h_k|$~\citep{bresler2017learning,gaitonde2024,gaitonde2024unified}; such dependence is moreover information theoretically necessary for estimation~\cite{santhanam2012information}. In contrast, our results depend only \emph{logarithmically} on the magnitudes of entries of $A$ and $h$.

We recognize the site observability assumption (observability of $I_t$) to be slightly unnatural in our context. While the latest work on learning under Glauber dynamics removes it~\citep{gaitonde2024}, the sample complexity of parameter identification still has exponential dependence on the $\ell_1$-width as defined above. It is an interesting open question whether it is possible to establish a similar guarantee to~\cref{thm:glauber} without the site observability assumption.

\section{Related Work} \label{sec:related_work}

\subsection{Beyond-Worst-Case Online Learning}

Our work is broadly related to the line of work on beyond-worst-case analysis of sequential decision making \citep{shetty2024learning}, notably smoothed analysis of online learning~\citep{rakhlin2011onlinelearning,haghtalab2020smoothed,haghtalab2022oracle,bhatt2023smoothed,haghtalab2024smoothed,block2022smoothed,block2024performance,blanchard2025agnostic}.
In smoothed online learning, the distribution of the adversarial covariates is restricted to have bounded density (w.r.t. some potentially unknown reference measure).
Under this assumption, it is shown that all VC classes are learnable, often efficiently, in the same ERM oracle model as ours. The key conceptual difference is that smoothed online learning imposes a per-round constraint on the adversary via the density bound, whereas our model imposes a global constraint via the existence of an approximate simulator. In this sense, smoothed online learning should be thought of as \emph{an approximation to independence}, while our framework allows for strongly dependent processes. On the other hand, smoothed online learning allows for adversarial labels, which our results generally do not (see~\cref{app:adaptive-labels} for a detailed discussion). Overall, the two models are complementary, capturing different ways of relaxing the adversarial nature of online learning.
A similar comparison can be made with other recent models, such as learning with abstention~\citep{goel2023abstention} and learning with relaxed benchmarks~\citep{montasser2025beyond}.

\subsection{Learning with Side Information}

A number of settings leverage \emph{side information}, most notably transductive online learning, predictable sequences, and online algorithms with predictions. A common denominator of these settings is that the learner is typically given \emph{deterministic} side information about the future, whereas the key technical challenge in our setting is that the learner must account for the \emph{stochasticity} of the covariate sequence, which requires new algorithmic ideas based on the relaxation framework~\citep{rakhlin2012relax,rakhlin2015hierarchies}. This challenge arises from the belief that it is often easier to find a simulator that captures the distribution of a process than to predict its individual future elements, which is the broad principle behind probabilistic modeling from statistical mechanics~\citep{glauber1963time,metropolis1953equation} to generative AI~\citep{brown2020language,ho2020denoising}.

\paragraph{Transductive online learning.} Our framework can be viewed as a distributional generalization of transductive online learning~\citep{ben1997online,kakade2005batch}. In classical transductive online learning, the learner receives the entire sequence of unlabeled covariates $(X_1, \ldots, X_T)$ at the start of the game and then makes predictions sequentially as labels are revealed.
This additional information allows all VC classes to be learned with regret $O(d \log T)$, circumventing the Littlestone-dimension barrier of adversarial online learning~\citep{littlestone1988learning}.
Our setting generalizes this by replacing the deterministic sequence with \emph{distributional} access: rather than knowing the exact sequence, the learner has access to a simulator for the covariate-generating distribution.
Classical transductive learning corresponds to the degenerate case in which the simulator places all its mass on a single sequence.

\paragraph{Predictable sequences.}
Online learning under the assumption that there is some side information about future instances is typically studied under the name of predictable sequences. These arise naturally in settings such as learning in games and optimization \citep{rakhlin2013predictable,rakhlin2013optimization,chiang2012online,syrgkanis2015fast,jadbabaie2015online,raman2024online,bhaskara2020online}. Although conceptually related, these settings do not capture learning under simulatable processes, since they typically assume that \emph{the realization} of future inputs is predictable in some sense, while we assume only \emph{distributional} access. Furthermore, most results focus on improving the dependence of regret on the horizon $T$; in our setting, it is not obvious a priori which function classes are \emph{learnable}, and our key contribution is showing that bounded VC dimension suffices for learnability.

\paragraph{Online algorithms with predictions.}
A closely related area is the emerging literature on online algorithms with predictions~\citep{mitzenmacher2022algorithms,lykouris2021competitive,purohit2018improving}.
In this setting, the algorithm is given (possibly imperfect) predictions about future inputs, and the goal is to design algorithms that perform well when predictions are accurate while maintaining worst-case guarantees.
The key difference is that these predictions are typically point estimates of future inputs, whereas in our setting the learner has access to a distributional simulator.
Furthermore, this line typically targets competitive-ratio guarantees for optimization problems, whereas we seek learning-theoretic regret bounds for classification.

\subsection{Learning under Stochastic Processes}

\paragraph{Dynamical systems and control.}
Learning under dynamical processes is a rich area with close ties to reinforcement learning and control~\citep{abbasi2011regret,cohen2018online}. A notable direction is the work on learning under linear dynamics~\citep{hazan2017learning,hazan2018spectral}, which studies online prediction of the outputs of a latent LDS. We instead focus on the classification setting with general VC classes, and assume access to a simulator rather than imposing structural assumptions on the dynamics. A related use of simulators is sim-to-real transfer~\citep{tobin2017domain,zhao2020simtoreal}, where simulators are used to train policies that are then deployed in the real world. The central challenge there is the ``reality gap'' between simulator and reality, which our framework quantifies via KL divergence.

\paragraph{Learning under mixing processes.}
A long line of work instead learns under processes that satisfy strong mixing properties~\citep{yu1994rates,mohri2010stability,kuznetsov2017generalization,gamarnik1999extension}. The key idea is that if the process mixes sufficiently fast, the dependencies between observations decay, and one can recover generalization bounds similar to the i.i.d.\ setting. For example, under $\beta$-mixing or $\phi$-mixing assumptions, PAC-style bounds can be established with rates depending on the mixing coefficients~\citep{mohri2010stability}. \citet{dagan2019learning} relax mixing via a Dobrushin-style condition, which is another way to quantify weak dependence in the data by bounding pairwise influences of the samples in the dataset. Our framework differs significantly: the data can have arbitrary dependence structure and we impose no mixing assumptions.

\paragraph{Learning without mixing.}
A complementary line of work learns from a single trajectory of a structured (and potentially non-mixing) process, such as a random walk or Markov chain over the instance space, typically for specific function classes such as DNF or juntas~\citep{aldous1995markovian,bartlett1994exploiting,bshouty2005learning,roch2007learning,arpe2008agnostically,kanade2015mcmc,cornacchia2026benefits}. Notably, in this line of work, the temporal dependence is often a \emph{resource} rather than an obstacle (e.g., for $k$-juntas, consecutive covariates that differ in one bit can reveal the relevant coordinates). Our work differs in scope: rather than a fixed class such as DNF or juntas under a specific process, we compete with an arbitrary VC class under a general simulatable process. Relatedly, a large literature studies \emph{estimation} of parameters of structured processes, such as Glauber dynamics~\citep{bresler2017learning,gaitonde2024,gaitonde2024unified} and linear dynamical systems~\citep{simchowitz2018learning,sarkar2019near}; our work is again orthogonal, since it concerns learning a labeling function that acts \emph{on top of} structured dynamics.

\subsection{Learning and Kolmogorov Complexity}

The connection between learning and Kolmogorov complexity has a rich history~\citep{li2019introduction}, dating back to Kolmogorov and Solomonoff.
Solomonoff introduced his method, \emph{Solomonoff induction}, which uses a prior based on (unbounded) Kolmogorov complexity, and argued that it can be used to learn computable processes. For a detailed exposition of these methods, see~\citep{hutter2024introduction,li2019introduction,lu2022optimal}.
Though these techniques are conceptually related to the approach we take, a key difference is that methods based on Solomonoff induction typically require \emph{evaluating} Kolmogorov complexity, which can be computationally prohibitive (in the time-bounded version~\citep{liu2020owf}) or even uncomputable (in the unbounded version).
We circumvent these issues using the fact that we only need to \emph{approximately sample} from the (time-bounded) universal distribution as opposed to computing its density.
In fact, computing the density of universal distributions is closely related to cryptography: roughly speaking, being efficiently computable is equivalent to the nonexistence of one-way functions~\citep{liu2020owf}.

In this work, we use the time-bounded variant of Kolmogorov complexity~\citep{sipser1983complexity,levin1984randomness}. \citet{hirahara2023learning} made the link between time-bounded Kolmogorov complexity and learning explicit, showing that under ``Pessiland'' assumptions (most notably, that one-way functions \emph{do not} exist) arbitrary sequences generated by time-bounded processes are learnable. Their proof crucially relies on the nonexistence of one-way functions to obtain an efficient conditional sampler. As discussed above, under cryptographic assumptions, such a conditional sampler is provably infeasible.

\subsection{Simulation-Based Inference}

Simulation-based inference (SBI), also known as likelihood-free inference, is a classical area in statistics concerned with performing inference when the likelihood function is intractable but a simulator for the data-generating process is available~\citep{marin2012approximate,sisson2018handbook,cranmer2020frontier,deistler2025simulation}.
The canonical example is Approximate Bayesian Computation (ABC)~\citep{beaumont2002approximate,marjoram2003markov}, which performs approximate posterior inference by comparing simulated data to observed data via summary statistics.
More recently, neural-network-based approaches have emerged that learn density estimators or likelihood ratios from simulated data~\citep{papamakarios2016fast,lueckmann2017flexible,greenberg2019automatic,hermans2020likelihood}.
Our work shares SBI's core assumption, which is access to a simulator but not the likelihood, and the KL divergence in our bounds can be seen as analogous to the ``simulation gap'' in SBI. However, the goals in the two settings are different: SBI performs parameter inference, whereas we perform online prediction with the objective of minimizing regret.

\section*{Acknowledgments}
We acknowledge support from NSF through awards DMS-2031883 and PHY-2019786, the DARPA AIQ program, and AFOSR FA9550-25-1-0375. SV acknowledges support from Amazon AI Research Innovation Fellowship. AS thanks Ankur Moitra for valuable discussions. 

\paragraph{LLM usage.} We used Claude Opus (4.6--4.8) to polish the exposition and to check the technical arguments for errors. In addition, some technical claims in~\cref{sec:cond} (in particular~\cref{lemma:mse-bound,lemma:v-eta}) were developed through iterative work with the model.

\bibliographystyle{plainnat}
\bibliography{refs}

\crefalias{section}{appendix}

\appendix

\section{Further Preliminaries}

\subsection{Learning-Theoretic Preliminaries}
We begin by recalling the VC dimension, a central combinatorial complexity measure in statistical learning theory~\citep{vapnik1998statistical}.
\begin{definition}[VC dimension] \label{def:vc}
For a binary class $\cF \subset \{0,1\}^{\cX}$, the VC dimension of $\cF$ is the largest integer $d$ for which there exists a subset $\{x_1,\ldots,x_d\} \subset \cX$ such that, for any labels $y_1,\ldots,y_d \in \{0,1\}$, there exists $f \in \cF$ such that $f(x_t) = y_t$ for every $t \in [d]$.
\end{definition}

A classical theorem states that PAC learnability is equivalent to finiteness of the VC dimension.
For our purposes, a useful consequence of bounded VC dimension is the celebrated Sauer--Shelah lemma. It upper bounds the number of realizable labelings on a subset of a domain in terms of the VC dimension.
\begin{lemma}[Sauer--Shelah Lemma]
\label{lemma:sauer-shelah}
    For a binary class $\cF$ with VC dimension $d$, and any finite set $S \subset \cX$, the number of distinct labelings of $S$ by functions in $\cF$ is at most \[
    \sum_{i=0}^d \binom{|S|}{i} \leq \left(\frac{e |S|}{d}\right)^d.
    \]
\end{lemma}

We now turn to the online learning setting and recall the Littlestone dimension, an online analogue of the VC dimension.
A binary tree $\tree x$ of depth $d$ on the domain $\cX$ is a sequence of mappings $\tree x_{t} \colon \{\pm 1\}^{t-1} \to \cX$ for $t \in [d]$.
    \begin{definition}[Littlestone Dimension]
        Let $\cF \subset \{0,1\}^\cX$ be a binary class of functions. The Littlestone dimension is the largest integer $d$ such that the following holds. There exists a tree $\mathbf x$ such that, for any choice of signs $\varepsilon_1,\ldots,\varepsilon_d \in\{\pm 1\}$ there exists a function $f \in \cF$ such that
        \[
        f(\tree x_t(\varepsilon_{1},\ldots,\varepsilon_{t-1})) = \frac{\varepsilon_t + 1}{2},
        \]
        for every $t \in [d]$.
    \end{definition}
    Finite Littlestone dimension characterizes learnability in classical online learning.
    Conversely, when the Littlestone dimension is infinite, one can embed arbitrarily deep Littlestone trees and force any learner to incur linear regret; we record a standard lower bound in the realizable case.
    \begin{proposition}
    \label{prop:littlestone-hardness}
        Let $\cF \subset \{0,1\}^\cX$ be a class of infinite Littlestone dimension, and let $\tree x$ be a depth-$T$ Littlestone tree. Consider nature's strategy that first samples $\varepsilon_t \sim \Unif(\{\pm 1\})$ independently for each $t \in [T]$. Then, at step $t \in [T]$, the nature presents the learner with $X_t = \tree x_t(\varepsilon_1,\ldots,\varepsilon_{t-1})$ labeled by $Y_t := (\varepsilon_t + 1)/2$. Then, any learner suffers $T/2$ regret in this setting.
    \end{proposition}
    The proof is based on the realizability of every root-to-leaf path in a Littlestone tree: for each path, there exists a function in $\cF$ that matches the induced labels along that path.
    We will refer to such a function as being \emph{consistent} with the path.
    \begin{definition}
    \label[definition]{def:littlestone-consistent-function}
        Let $\tree x$ be a depth-$T$ tree. We say that a function $f$ is consistent with the path $(\varepsilon_1,\ldots,\varepsilon_T)$ in $\tree x$ iff
        \[
        f(\tree x_t(\varepsilon_1,\ldots,\varepsilon_{t-1})) = \frac{\varepsilon_t + 1}{2}.
        \]
    \end{definition}

\subsection{Learning with ERM Oracles} We adopt an ERM oracle computation model inspired by~\cite{hazan2016computational}. We distinguish between realizable and agnostic ERM oracles. Both oracles accept as input a labeled dataset, $S \subset \cX \times \{0,1\}$. When called on $S$, the realizable ERM oracle returns a function $f \in \cF$ that realizes $S$ (that is, $f(x) = y$ for every $(x,y)\in S$) with ties broken adversarially; if no such function exists, the oracle returns ``not realizable.'' When called on $S$, the agnostic ERM oracle returns a function $f$ that makes the least number of mistakes, that is, any function in the set:
\[
\argmin_{f \in \cF} \sum_{(x,y) \in S} \ind\left\{f(x) \neq y\right\}.
\]
We assume both oracles run in unit time, but it costs time to construct the input set $S$. We also assume that, given a function class $\cF$, any $f \in \cF$ can be evaluated efficiently on any $x \in \cX$.

    \subsection{Miscellaneous Lemmata}
    We collect a few technical inequalities for KL divergence that will be used repeatedly.
    The first is a simple change of measure bound (Example 7.3 of~\cite{polyanskiy2025information}).
    \begin{lemma}
    \label{lemma:mult-pinkser}
        For any $P, Q\in\Delta([0,1])$, we have
        \[
        \En_{Q} X \le 3 \En_P X + 4 \KL{P}{Q}.
        \]
    \end{lemma}

    Next, we prove a Markov-type inequality for the density ratio, with a bound expressed in terms of $\KL{P}{Q}$.
    \begin{lemma}
    \label{lemma:kl-markov}
    For any $P,Q$, $\lambda > e$ we have
    \[
    \P_{P}\left[\frac{\mathrm{d} P}{\mathrm{d} Q} > \lambda\right] \le \frac{\KL{P}{Q}}{\log(\lambda/e)}
    \]
    \end{lemma}
    \begin{proof}
        By definition of KL, we have
        \begin{align}
            \KL{P}{Q} &= \En_{P} \log \frac{\mathrm{d} P}{\mathrm{d} Q} \notag \\
            &= \En_{P} \log \frac{\mathrm{d} P}{\mathrm{d} Q} \ind \left\{ \frac{\mathrm{d} P}{\mathrm{d} Q} \ge \lambda \right\} + \En_{P} \log \frac{\mathrm{d} P}{\mathrm{d} Q} \ind \left\{ \frac{\mathrm{d} P}{\mathrm{d} Q} < \lambda \right\} \notag\\
            &\ge \log(\lambda) \P_P\left[\frac{\mathrm{d} P}{\mathrm{d} Q} \ge \lambda\right] + \En_{P} \log \frac{\mathrm{d} P}{\mathrm{d} Q} \ind \left\{ \frac{\mathrm{d} P}{\mathrm{d} Q} < \lambda \right\}\label{eq:kl-markov-estimate}.
        \end{align}
        It remains to lower bound the second term above. By change of measure, and using $u \log(u) \ge \log(e) (u-1)$, we have:
        \begin{align*}
        \En_{P} \log \frac{\mathrm{d} P}{\mathrm{d} Q} \ind \left\{ \frac{\mathrm{d} P}{\mathrm{d} Q} < \lambda \right\}
        &= \En_{Q} \frac{\mathrm{d} P}{\mathrm{d} Q} \log \frac{\mathrm{d} P}{\mathrm{d} Q} \ind \left\{ \frac{\mathrm{d} P}{\mathrm{d} Q} < \lambda \right\}\\
        &\ge \En_{Q} \log(e) \left( \frac{\mathrm{d} P}{\mathrm{d} Q} - 1 \right) \ind \left\{ \frac{\mathrm{d} P}{\mathrm{d} Q} < \lambda \right\}\\
        &= \log(e) \P_P\left[ \frac{\mathrm{d} P}{\mathrm{d} Q} < \lambda \right] - \log(e) \P_Q\left[ \frac{\mathrm{d} P}{\mathrm{d} Q} < \lambda \right]\\
        &= \log(e) - \log(e) \P_P\left[ \frac{\mathrm{d} P}{\mathrm{d} Q} \ge \lambda \right] - \log(e) \P_Q\left[ \frac{\mathrm{d} P}{\mathrm{d} Q} < \lambda \right]\\
        &\ge - \log(e) \P_P\left[ \frac{\mathrm{d} P}{\mathrm{d} Q} \ge \lambda \right].
        \end{align*}
        Plugging this into~\cref{eq:kl-markov-estimate}, we have
        \begin{align*}
        \KL{P}{Q} &\ge \log(\lambda) \P_P\left[\frac{\mathrm{d} P}{\mathrm{d} Q} \ge \lambda\right] - \log(e) \P_P\left[\frac{\mathrm{d} P}{\mathrm{d} Q} \ge \lambda\right] \\
        &= \log(\lambda/e) \P_P\left[\frac{\mathrm{d} P}{\mathrm{d} Q} \ge \lambda\right].
        \end{align*}
        Rearranging the above concludes the proof.
    \end{proof}

    Finally, we bound the expectation of the positive part of the log density ratio by a KL term.
    \begin{lemma}
    \label{lemma:log-plus}
        For any $P,Q$ we have
        \[
        \En_{P} \log^+ \left(\frac{\mathrm{d} P}{\mathrm{d} Q}\right) \le \KL{P}{Q} + 1.
        \]
    \end{lemma}
    \begin{proof}
        By definition of KL, we have
        \begin{equation}
        \label{eq:log-plus-estimate}
        \En_{P} \log^+ \left(\frac{\mathrm{d} P}{\mathrm{d} Q}\right) = \KL{P}{Q} - \En_{P} \log \left(\frac{\mathrm{d} P}{\mathrm{d} Q}\right) \ind\left\{  \frac{\mathrm{d} P}{\mathrm{d} Q} \le 1\right\}.
        \end{equation}
        Thus, it suffices to upper bound the second term. By a change of measure argument and using $u \log(u) \ge - \log(e)/e$ for any $u\ge 0$, we have
        \begin{align*}
        \En_{P} \log \left(\frac{\mathrm{d} P}{\mathrm{d} Q}\right) \ind\left\{  \frac{\mathrm{d} P}{\mathrm{d} Q} \le 1\right\} &= \En_{Q} \frac{\mathrm{d} P}{\mathrm{d} Q} \log \left(\frac{\mathrm{d} P}{\mathrm{d} Q}\right) \ind\left\{  \frac{\mathrm{d} P}{\mathrm{d} Q} \le 1\right\} \\
        &\ge -\frac{\log(e)}{e}.
        \end{align*}
        Plugging this into~\cref{eq:log-plus-estimate}, we obtain:
        \begin{align*}
        \En_{P} \log^+ \left(\frac{\mathrm{d} P}{\mathrm{d} Q}\right) &\le \KL{P}{Q} + \frac{\log(e)}{e} \\
        &\le \KL{P}{Q} + 1,
        \end{align*}
        as desired.
    \end{proof}

\section{Proofs from~\texorpdfstring{\cref{sec:simulation}}{Section \getrefnumber{sec:simulation}}}

\subsection{Proof of~\texorpdfstring{\cref{thm:stat-dist-trans}}{Theorem \getrefnumber{thm:stat-dist-trans}}}

We begin by proving~\cref{lemma:cover-analysis}.
\CoverAnalysis*

\begin{proof}
    Let $f\in \cF$ be arbitrary. From the construction of $\cG(\seq Z)$, it contains a function: 
    \[
    g := \Alg_{\mathsf{M3}}\left((Z_i, f(Z_i))_{i \in [N]}\right).
    \]
    Using the fact that $\seq Z$ is sampled in an i.i.d. fashion from $\bar \dist$, and using a PAC-learning guarantee for the Majority-of-Three learner from~\cite{aden2024majority}, we have: 
    \[
    \En_{\seq Z} \P_{X \sim \bar \dist} \left[f(X)\neq g(X)\right] \lesssim \frac{d}{N}.
    \]
    Thus, 
    \[
    \En_{\seq Z} \min_{g \in \cG(\seq Z)} \P_{X \sim \bar \dist} \left[f(X)\neq g(X)\right] \lesssim \frac{d}{N}.
    \]
    It remains to note that 
    \begin{align*}
    \En_{\seq Z} \min_{g \in \cG(\seq Z)} \En_{\seq X \sim \dist} \norm{f -g}_{\seq X} &= \En_{\seq Z} \min_{g \in \cG(\seq Z)} \En_{\seq X \sim \dist} \frac{1}{T} \sum_{t=1}^T |f(X_t) - g(X_t)| \\
    &= \En_{\seq Z} \min_{g \in \cG(\seq Z)} \En_{X \sim \bar \dist} |f(X) - g(X)| \\
    &= \En_{\seq Z} \min_{g \in \cG(\seq Z)} \P_{X \sim \bar \dist} \left[f(X)\neq g(X)\right] \\
    &\lesssim \frac{d}{N}.
    \end{align*}
    The proof is concluded via Jensen's inequality:
    \[
    \En_{\seq Z} \En_{\seq X \sim \dist} \min_{g \in \cG(\seq Z)}  \norm{f -g}_{\seq X}\le \En_{\seq Z} \min_{g \in \cG(\seq Z)} \En_{\seq X \sim \dist} \norm{f -g}_{\seq X} \lesssim \frac{d}{N}.
    \]
\end{proof}

The proof of~\cref{thm:stat-dist-trans} follows.
\StatDistTrans*
\begin{proof} Note that we may WLOG assume that $N \le 2^{O(T)}$ and $N$ is larger than some prespecified constant. Indeed, for $N \le O(1)$, the theorem statement is vacuous, and for $N$ larger than $2^{O(T)}$, the statement follows from $N = T 2^T$. Set $L := \lfloor \log \log (N/2) \rfloor $, and let
\[
N_\ell = \lfloor (N/2)^{\frac{1}{2^{L - \ell}}} \rfloor,
\]
for each $\ell \in [L]$. Then, $L \lesssim \log(T).$ Also, as long as $N$ is larger than some constant, we have the following properties:
(i) $\sum_{\ell \in [L]} N_\ell \le N$, (ii) $N_L \asymp N$, (iii) $N_1 \asymp 1$, and (iv) $N_{\ell+1} \lesssim N_\ell^2$. Let $\{E_\ell\}_{\ell \in [L+1]}$ be the following collection of events that partition the probability space. Set
\[
E_1 := \left\{\frac{\mathrm{d} \dist^\star}{\mathrm{d} \dist} \le \frac{N_1}{T}\right\},
\]
for each $2 \le \ell \le L$, set
\[
E_\ell := \left\{\frac{N_{\ell - 1}}{T} < \frac{\mathrm{d} \dist^\star}{\mathrm{d} \dist} \le \frac{N_\ell}{T} \right\},
\]
and let
\[
E_{L+1} := \left\{\frac{N_{L}}{T} < \frac{\mathrm{d} \dist^\star}{\mathrm{d} \dist} \right\}.
\]
We use small-loss bounds for the experts algorithms with non-uniform prior. Recall that we initialize the weights in \hyperref[line:init-weights-def]{Line~\ref{line:init-weights-def}} of~\cref{algo:multi-cover} as $w_1(g) = 1/(L|\cG_\ell|)$ for each $g \in \cG_\ell$. Then, by Theorem 2.4 in~\cite{arora2012multiplicative}, for $\eta = 1/2$, the regret of our algorithm can be upper bounded by, with probability $1$,
\begin{align}
\reg &\lesssim \min_{g\in\cG} \left(\log (1/w_1(g)) + \sum_{t=1}^T \ind\left\{f^\star(X_t) \neq g(X_t)\right\}\right) \notag\\
&= \min_{\ell \in [L]} \min_{g\in\cG_\ell} \left(\log (1/w_1(g)) + \sum_{t=1}^T \ind\left\{f^\star(X_t) \neq g(X_t)\right\}\right)\notag\\
&= \min_{\ell \in [L]} \min_{g\in\cG_\ell} \left(\log (L |\cG_\ell|) + \sum_{t=1}^T \ind\left\{f^\star(X_t) \neq g(X_t)\right\}\right)\notag\\
&\lesssim \min_{\ell \in [L]} \min_{g\in\cG_\ell} \left(\log (L N_\ell^d) + \sum_{t=1}^T \ind\left\{f^\star(X_t) \neq g(X_t)\right\}\right)\notag\\
&= \min_{\ell \in [L]} \left(\log (L N_\ell^d) + \min_{g\in\cG_\ell} \sum_{t=1}^T \ind\left\{f^\star(X_t) \neq g(X_t)\right\}\right) \notag\\
&\le \min_{\ell \in [L]} \left(d \log (N_\ell) + T \cdot \min_{g \in \cG_\ell} \norm{f^\star - g}_{\seq X}\right) + \log L\notag
\end{align}
where the second inequality uses the Sauer--Shelah lemma~(\cref{lemma:sauer-shelah}). Now, since $\{E_\ell\}_{\ell \in [L+1]}$ partition the probability space, we have: 
\begin{align}
\reg &\le \sum_{\ell \in [L+1]} \left(d \log (N_{\ell \land L}) + T \cdot \min_{g \in \cG_{\ell\land L}} \norm{f^\star - g}_{\seq X}\right) \ind\{E_{\ell}\} + \log L\notag\\
&\le \sum_{\ell \in [L+1]} d \log (N_{\ell \land L})\ind\{E_{\ell}\} + T \sum_{\ell \in [L+1]} \min_{g \in \cG_{\ell \land L}} \norm{f^\star - g}_{\seq X} \ind\{E_{\ell}\} + \log L. \label{eq:multi-cover-reg-decomp}
\end{align}
Recall from \hyperref[line:Z-samples-def]{Line~\ref{line:Z-samples-def}} of~\cref{algo:multi-cover} that $Z^\ell_i$ are sampled from $\bar \dist$ independently for each $\ell\in[L]$ and $i \in [N_\ell]$. Let $N_{\mathsf{tot}} := \sum_{\ell \in [L]} N_{\ell} \le N$, and let $\seq Z := \{Z^{\ell}_i\}_{\ell \in [L], i \in [N_\ell]}$. Then, $\seq Z \sim \bar \dist^{\otimes N_\mathsf{tot}}$, and $\seq Z^\ell \sim \bar\dist^{\otimes N_\ell}$ for each $\ell \in [L]$. Using~\cref{lemma:cover-analysis}, for any $\ell \in [L]$, we have
\[
\En_{\seq Z^\ell} \En_{\seq X \sim \dist} \min_{g \in\cG_\ell} \norm{f^\star - g}_{\seq X} \lesssim \frac{d}{N_\ell}.
\]
In the remainder, we suppress the expectation under $\seq Z$, and write $\En_{\dist^\star}$ and $\En_{\dist}$ to denote expectations over $(\seq X, \seq Z) \sim \dist^\star \otimes \bar\dist^{\otimes N_\mathsf{tot}}$ and $(\seq X, \seq Z) \sim \dist \otimes \bar\dist^{\otimes N_\mathsf{tot}}$ respectively for notational clarity. Then, note that, for all $\ell \in [L]$, we have
\begin{align*}
\En_{\dist^\star} \min_{g \in \cG_\ell} \norm{f^\star - g}_{\seq X} \ind\{E_\ell\}
&\le  \frac{N_\ell}{T} \En_{\dist} \min_{g \in \cG_\ell} \norm{f^\star - g}_{\seq X} \ind\{E_\ell\} \\
&\le  \frac{N_\ell}{T} \En_{\dist} \min_{g \in \cG_\ell} \norm{f^\star - g}_{\seq X} \\
&\lesssim  \frac{N_\ell}{T} \cdot \frac{d}{N_\ell} \\
&= \frac{d}{T}.
\end{align*}
For the final approximation term, we have by~\cref{lemma:kl-markov}, for any $\lambda > e$
\begin{align*}
    \En_{\dist^\star} \min_{g \in \cG_L} \norm{f^\star - g}_{\seq X} \ind\{E_{L+1}\} &\le \En_{\dist^\star} \min_{g \in \cG_L} \norm{f^\star - g}_{\seq X} \\
    &\le   \En_{\dist^\star} \min_{g \in \cG_L} \norm{f^\star - g}_{\seq X} \ind \left\{ \frac{\mathrm{d}\dist^\star}{\mathrm{d}\dist} \le \lambda \right\} + \P_{\dist^\star} \left[ \frac{\mathrm{d}\dist^\star}{\mathrm{d}\dist} > \lambda \right] \\
    &\le \lambda \En_{\dist} \min_{g \in \cG_L} \norm{f^\star - g}_{\seq X} + \frac{\KL{\dist^\star}{\dist}}{\log(\lambda/e)}\\
    &\lesssim \lambda \frac{d}{N_{L}} + \frac{\KL{\dist^\star}{\dist}}{\log(\lambda/e)}.
\end{align*}
Set $\lambda = (N_L/T) + e$. Then,
\begin{align*}
\En_{\dist^\star} \min_{g \in \cG_{L}} \norm{f^\star - g}_{\seq X} \ind\{E_{L+1}\} &\lesssim \frac{d}{T} + \frac{d}{N_L} + \frac{\KL{\dist^\star}{\dist}}{\log\left(1 + \frac{N_L}{e T}\right)}
\end{align*}
Thus, the total approximation term is upper bounded as:
\begin{align}
\sum_{\ell \in [L+1]} \En_{\dist^\star} \min_{g \in \cG_{\ell\land L}} \norm{f^\star - g}_{\seq X} \ind\{E_\ell\} &\lesssim \frac{dL}{T} + \frac{d }{N_L} + \frac{\KL{\dist^\star}{\dist}}{\log\left(1 + \frac{N_L}{eT}\right)} \notag \\
&\lesssim \frac{dL}{T} + \frac{d}{N} + \frac{\KL{\dist^\star}{\dist}}{\log\left(1 + \frac{N}{T}\right)} \label{eq:total-approx-err},
\end{align}
where in the last line we used $N_L \asymp N.$
Recall that, from the way $N_\ell$ are selected, we have $N_{\ell} \lesssim N_{\ell-1}^2$ for every $2 \le \ell \le L$. Thus, on event $E_\ell$ for $2 \le \ell \le L$, we have
\[
\log(N_{\ell}) \lesssim \log(N_{\ell-1}) + 1 \le \log \left(T \frac{\mathrm{d}\dist^\star}{\mathrm{d}\dist}\right) + 1
\]
Also, on event $E_{L+1}$, we have 
\[
\log N_L \le \log \left(T \frac{\mathrm{d}\dist^\star}{\mathrm{d}\dist}\right).
\]
Thus, for every $\ell \ge 2$, we have
\[
d \log (N_{\ell\land L})\ind\{E_\ell\} \lesssim d \log \left(T \frac{\mathrm{d} \dist^\star}{\mathrm{d} \dist} \right)\ind\{E_\ell\} + d \ind\{E_\ell\}
\]
For $\ell = 1$, we simply have:
\[
d \log (N_1) \ind\{E_1\}\lesssim d \ind\{E_1\}.
\]
Then, summing over all $\ell$, we have:
\begin{align*}
\sum_{\ell \in [L+1]} d \log (N_{\ell\land L})\ind\{E_\ell\} &\lesssim d\sum_{\ell \ge 1} \ind\{E_\ell\} + d \log \left(T \cdot \frac{\mathrm{d} \dist^\star}{\mathrm{d} \dist} \right) \ind \left\{\bigcup_{\ell \ge 2} E_\ell\right\}\\
&\lesssim d + d \log^+ \left(T \cdot \frac{\mathrm{d} \dist^\star}{\mathrm{d} \dist} \right) \\
&\lesssim d\log(T) + d \log^+\left(\frac{\mathrm{d} \dist^\star}{\mathrm{d} \dist} \right).
\end{align*}
Taking expectations, we have
\begin{align}
\En_{\dist^\star}\sum_{\ell \in [L+1]} d \log (N_{\ell\land L})\ind\{E_\ell\} &\lesssim d\log(T) + d\En_{\dist^\star} \log^+\left(\frac{\mathrm{d} \dist^\star}{\mathrm{d} \dist} \right) \notag \\
&\lesssim d\log(T) + d \KL{\dist^\star}{\dist} \label{eq:total-expert-err},
\end{align}
where in the second step we used~\cref{lemma:log-plus}.
To conclude the proof, we combine~\cref{eq:multi-cover-reg-decomp,eq:total-expert-err,eq:total-approx-err}:
\begin{align*}
\En_{\dist^\star}\reg (\cF, T)
&\lesssim d\log T + d \KL{\dist^\star}{\dist} + dL + \frac{d T}{N} + T \frac{\KL{\dist^\star}{\dist}}{\log\left(1 + \frac{N}{T}\right)}\\
&\lesssim d\log T + d \KL{\dist^\star}{\dist} + \frac{d T}{N} + T \frac{\KL{\dist^\star}{\dist}}{\log\left(1 + \frac{N}{T}\right)},
\end{align*}
where the second inequality holds since $L \lesssim \log(T)$. This concludes the proof.
\end{proof}

\subsection{Proof of~\texorpdfstring{\cref{thm:realizable-lb}}{Theorem \getrefnumber{thm:realizable-lb}}}

\RealizableLB*

The $\Omega(d)$ lower bound is trivial from the definition of VC dimension. It suffices to prove the remaining lower bounds separately. Throughout, we work with $d$-dimensional thresholds, defined as follows. We let $\cX_d := [0,1] \times [d]$ be the domain, and let
\begin{equation}
\label{eq:d-dim-thresholds}
\cF_d := \{ (x,i) \mapsto \ind\{ x \ge \theta_i \},  (\theta_1,\ldots,\theta_d) \in [0,1]^d\}.
\end{equation}
It is easy to see that the VC dimension of the above class is exactly $d$. We proceed with the proof of the ``price of sampling'' lower bound of $d T/N$. It holds even in the case $\dist^\star = \dist$.

\begin{lemma}
\label{lemma:price-of-sampling}
    Let $\cF_d$ be the class as defined in~\cref{eq:d-dim-thresholds}. For any $N \in \mathbb N$ and for any learning algorithm that draws at most $N$ unconditional samples from $\dist^\star$, there exists a choice of $\dist^\star$ and $f^\star \in \cF_d$ such that
    \[
    \En \reg(\cF_d, T) \gtrsim \frac{T (d \land N)}{N}.
    \]
\end{lemma}
\begin{proof}
    For each $i \in [d]$, let $\tree z_i$ denote a depth-$T$ Littlestone tree in coordinate $i$. Then, let $\rho_i$ be a distribution that samples
    a uniform path in $\tree z_i$. Let $\rho$ be a uniform mixture of $\rho_i$ as $i\sim\mathsf{Unif}([d])$.

    Note that we may assume WLOG that $N \ge d$ (otherwise, we can consider $\cF_N$ instead of $\cF_d$ and prove a linear lower bound on regret).
    To define $\dist^\star$, we first sample $d$ paths $\seq x^i := (x_1^i,\ldots,x_T^i) \sim \rho_i$ for $i \in [d]$; then, we let
    \[
    \dist^\star = \frac{1}{2N} \sum_{i \in [d]} \delta_{\seq x^i} + \left(1 - \frac{d}{2N}\right) \rho.
    \]
    Let $f^\star$ be a labeling function that is consistent with each of the paths $\seq x^i$ for $i \in [d]$ (see~\cref{def:littlestone-consistent-function}). Note that, WLOG, we may assume that the learner draws all $N$ samples before the first round of the online learning game. With probability at least $d/2N$, the distribution $\dist^\star$ samples a path $\seq x^{i_\star}$ for some $i_\star \in [d]$. Since the learner only draws $N$ samples, with probability at least:
    \[\left(1 - \frac{1}{2N}\right)^{N} \ge \frac{1}{2},\] none of the sampled paths are equal to $\seq x^{i_\star}$ and are either sampled from $\rho$, or are equal to $\seq x^j$ for some $j \neq i_\star$. Condition on this constant probability event. Since the randomness used in the construction of $\seq x^{i_\star}$ is independent of other paths and of samples from $\rho$, the learner will incur expected regret of at least $T/2$ on this path by standard Littlestone tree hardness argument (see~\cref{prop:littlestone-hardness}). Thus, total expected regret of the learner can be lower bounded as:
    \[
    \En \reg(\cF_d, T) \ge \frac{d}{2N}\cdot \frac{1}{2} \cdot \frac{T}{2} = \frac{dT}{8N}.
    \]
    This concludes the proof.
\end{proof}

Next, we prove a ``price of approximation'' lower bound: even with full knowledge of $\dist$, the regret of any algorithm must scale linearly with $\KL{\dist^\star}{\dist}$.
\begin{lemma}
\label[lemma]{lemma:linear-kl-lb}
    Let $\cF_1$ be as in~\cref{eq:d-dim-thresholds}. For every $D \in \mathbb N$, there exist a family $\cP_D$ of distributions over $\cX^T$ and a simulator $\dist$ such that:
    \begin{enumerate}[(i)]
        \item $\max_{\dist^\star \in \cP_D} \KL{\dist^\star}{\dist} \le D \land T$.
        \item For any learning algorithm, even one granted full knowledge of $\dist$, there exists $\dist^\star \in \cP_D$ and a labeling $f^\star \in \cF_1$ for which
        \[
        \En \reg(\cF_1, T) \ge \frac{D \land T}{2}.
        \]
    \end{enumerate}
\end{lemma}
\begin{proof}
    Let $m := D \land T$. Consider a Littlestone tree of depth $m$, and let $\dist$ be a distribution that samples a uniform path in that Littlestone tree, padding the remaining $T - m$ rounds with a fixed $x^0 \in \cX$. Let $\cP_D$ be the family of all such distributions, that is, $\dist^\star \in \cP_D$ iff $\dist^\star = \delta_{\seq x}$ for some leaf $\seq x$ (with the same padding). We note that $\dist(\seq x) = \frac{1}{2^m}$, and thus, for any $\dist^\star \in \cP_D$,
    \[
    \KL{\dist^\star}{\dist} = \log(2^m) = m \le D,
    \]
    which establishes (i). For (ii), we construct $\dist^\star$ randomly by selecting a path $\seq x$ in the Littlestone tree uniformly at random and setting $\dist^\star = \delta_{\seq x}$, with the same padding. Set $f^\star \in \cF_1$ to be a threshold function consistent with the path $\seq x$ (see~\cref{def:littlestone-consistent-function}). Since the construction of $\dist^\star$ is independent of $\dist$, the setup is equivalent to the standard Littlestone hardness construction (see~\cref{prop:littlestone-hardness}). Thus, any learner must suffer $m/2 = (D \land T)/2$ expected regret, which concludes the proof.
\end{proof}

Next, we quantify how access to a bounded number of samples from $\dist$ worsens the price of approximation.
\begin{lemma}
\label{lemma:kl-sampling-lb}
    Let $\cF_1$ be as in~\cref{eq:d-dim-thresholds}. For any $N \ge 2$, $D \in \mathbb N$, and for any learning algorithm that draws at most $N$ unconditional samples from $\dist$, there exists a choice of $(\dist, \dist^\star)$ with $\KL{\dist^\star}{\dist} \le D$ and $f^\star \in \cF_1$ such that
    \[
    \En \reg(\cF_1, T) \gtrsim T \left( \frac{D}{\log(N)} \land 1 \right)
    \]
\end{lemma}
\begin{proof}
    Consider a Littlestone tree of depth $T$. Consider a distribution $\rho$ that samples a path in the Littlestone tree at random, and let $\seq x \sim \rho$ be a random path in the Littlestone tree.
    Then, we define $\dist, \dist^\star$ as mixtures:
    \[\dist^\star = \alpha \delta_{\seq x} + (1-\alpha) \rho, \qquad \dist = \beta \delta_{\seq x} + (1-\beta) \rho,\] for some $1 \ge \alpha > \beta > 0$ to be specified later. Then, we have
    \[
    \KL{\dist^\star}{\dist} = \KL{\Ber(\dist^\star(\seq x))}{\Ber(\dist(\seq x))},
    \]
    since $\dist^\star$ and $\dist$ have a constant likelihood ratio $(1-\alpha)/(1-\beta)$ on the support off $\seq x$, so the full KL collapses to the binary KL between the Bernoulli marginals $\Ber(\dist^\star(\seq x))$ and $\Ber(\dist(\seq x))$. Furthermore, a simple computation shows that $p \mapsto \KL{\Ber(\alpha + (1-\alpha)p)}{\Ber(\beta + (1-\beta)p)}$ is non-increasing for $p \in [0,1]$, therefore,
    \[
    \KL{\Ber(\dist^\star(\seq x))}{\Ber(\dist(\seq x))}\le \KL{\Ber(\alpha)}{\Ber(\beta)} \le \alpha \log(1/\beta).
    \]
    Together, these inequalities give:
    \[
    \KL{\dist^\star}{\dist}\le \alpha \log(1/\beta).
    \]
    Let $f^\star$ be the labeling function consistent with the path $\seq x$ (see~\cref{def:littlestone-consistent-function}).
    Recall that $N$ is the almost sure upper bound on the number of samples drawn by the algorithm. Let $\beta = 1/N$. Then, with at least constant probability, all of the samples drawn by the algorithm will come from $\rho$. Now, let $\alpha = (D/\log(N)) \land 1$, which gives $\KL{\dist^\star}{\dist} \le D$.

    Now, note that, with probability at least $\alpha$, the learner will be presented with $\seq x$ at runtime, on which it incurs linear regret by standard Littlestone hardness (see~\cref{prop:littlestone-hardness}). Thus,
    \[
    \En \reg(\cF_1, T) \gtrsim \alpha T = T \left(\frac{D}{\log(N)} \land 1\right).
    \]
    This concludes the proof.
\end{proof}

\begin{proof}[Proof of~\cref{thm:realizable-lb}]
The bound follows by combining the four lower bounds established above, each holding against \emph{any} algorithm drawing at most $N$ samples. The $d \land T$ term is the trivial lower bound from the definition of VC dimension. The ``price of sampling'' bound (\cref{lemma:price-of-sampling}) forces regret $\gtrsim \frac{T(d \land N)}{N} = \frac{dT}{N} \land T$. The ``price of approximation'' bound (\cref{lemma:linear-kl-lb}) forces regret $\gtrsim D \land T$. Finally, ``price of approximate sampling'' (\cref{lemma:kl-sampling-lb}) forces regret $\gtrsim T \left( \frac{D}{\log N} \land 1 \right) = \frac{DT}{\log N} \land T$. Each of these constructions satisfies $\KL{\dist^\star}{\dist} \le D$ and can be realized within $\cF_d$ (\cref{eq:d-dim-thresholds}), which has VC dimension $d$. Indeed, the one-dimensional constructions ($\cF_1$) can be trivially embedded into $\cF_d$.

Note that each of the four bounds holds against any algorithm. Then, after fixing the algorithm at hand, we may adversarially choose the construction that yields the largest lower bound. Thus, writing $R_{\mathsf{max}} = \max\left\{d, \frac{dT}{N}, D, \frac{DT}{\log N}\right\}$, the corresponding construction forces
\[
\En \reg(\cF_d, T) \gtrsim R_{\mathsf{max}} \land T \ge \frac{1}{4}\left(d + \frac{dT}{N} + D + \frac{DT}{\log N}\right) \land T,
\]
for class $\cF_d$ of VC dimension $d$ and some $(\dist^\star, \dist)$ with $\KL{\dist^\star}{\dist} \le D$, which gives the desired claim.
\end{proof}

\subsection{Proof of~\texorpdfstring{\cref{thm:agnostic-impossibility}}{Theorem \getrefnumber{thm:agnostic-impossibility}}}

\AgnosticImpossibility*
\begin{proof}
    WLOG, assume $N \ge 2$. Divide the interval $[0,1]$ into $N$ sub-intervals $\{A_{i}\}_{i \in [N]}$, and consider $N$ Littlestone trees in each respective interval. For each $i \in[N]$, let $\rho_i$ be a distribution that samples a random path in the Littlestone tree in interval $A_i$. For each $i \in [N],$ sample a path $\seq x^i \sim \rho_i$. Now, let $\dist^\star$ and $f^\star$ be defined in the following way. Set
    \[
    \dist^\star := \frac{1}{N} \sum_{i \in [N]} \delta_{\seq x^i}.
    \]
    On each interval $A_i$, $i \in [N]$, let $f_i$ be a threshold function consistent with the path $\seq x^i$ (see~\cref{def:littlestone-consistent-function}). Then, we let
    \[
    f^\star(x) := \sum_{i \in [N]} \ind\{x \in A_i\} f_i(x).
    \]
    Then, with probability $1$ over the draw from $\seq X \sim \dist^\star$, the labels of $f^\star$ are realizable w.r.t. $\cF$ on $\seq X$, that is, with probability $1$:
    \[
    \inf_{f \in\cF} \sum_{t=1}^T \ind\{f^{\star}(X_t) \neq f(X_t)\} = 0.
    \]
    Let $i_\star \in [N]$ be a random index such that $\seq X = \seq x^{i_\star}$. Then, with probability at least
    \[
    \left(1 - \frac{1}{N}\right)^{N} \ge \frac{1}{4},
    \]
    none of the samples the learner draws are equal to $\seq x^{i_\star}$. Let us condition on this event. Then, for any $t \in[T]$ the randomness of $\seq x^{i_\star}$ is independent of the randomness of the learner in round $t$, that is, learner suffers expected regret of $T/2$ by standard Littlestone hardness result (see~\cref{prop:littlestone-hardness}). Thus,
    \[
    \En \left[\sum_{t=1}^T  \ind\{f^{\star}(X_t) \neq \hat{y}_t\} - \inf_{f \in\cF} \sum_{t=1}^T  \ind\{f^{\star}(X_t) \neq f(X_t)\} \right] \ge \frac{1}{4} \cdot \frac{T}{2} = \frac{T}{8},
    \]
    as desired.
\end{proof}

\subsection{Proof of~\texorpdfstring{\cref{thm:realizable-erm-oracle-lb}}{Theorem \getrefnumber{thm:realizable-erm-oracle-lb}}}
Below we show that in the realizable setting, if the learning algorithm only receives unconditional samples from $\dist^\star$, then any sublinear-regret learning algorithm requires a number of ERM oracle queries which is at least exponential in the VC dimension. We assume that when making queries to the ERM oracle, the learning algorithm must choose covariates that belong to either (a) the set of previous covariates revealed by nature (i.e., $X_1, \ldots, X_t$) or (b) the set of covariates amongst the unconditional samples drawn from $\dist^\star$. This assumption is essentially without loss of generality, as one may enlarge the domain $\cX$ with a large number of ``dummy'' points; to keep the notation simple, however, we explicitly make this assumption in this section. Moreover, we assume that the learning algorithm $\Alg$ must succeed for \emph{any} valid choice of ERM oracle; in particular, the lower bound below establishes the existence of \emph{some} ERM oracle with respect to which $\Alg$ suffers linear regret.

\RealizableERMOracleLB*
\begin{proof}
Fix $T, N \in \mathbb{N}$, and let $M = 10N$. We define $\cX := [M] \times [T]$, and for $i \in [M]$, $\cX_i := \{ i \} \times [T]$. Write $K = CT^4N^{10}$ for a sufficiently large constant $C$ (to be specified below).  We construct a random class $\cF$, as follows:
\begin{itemize}
\item First, we generate $K+1$ random functions, which we denote by $h_1, \ldots, h_{K+1}$ (i.e., for each such $h_k$, $h_k(x) \in \{0,1\}$ is drawn uniformly at random for each $x \in \cX$).
\item We set $f^\star = h_{K+1}$.
\item For each $i \in [M]$, $t \in [T]$, $k_0 \in [K+1]$, $k_1 \in [K]$, and $b \in \{0,1\}$, we define the function $g_{t,i,b,k_0, k_1}$ as follows:
\begin{align}
g_{t,i,b,k_0, k_1}(x) = \begin{cases}
h_{k_0}(x) &: x \in \{(i,1), \ldots, (i,t-1) \} \\
 b &: x = (i, t) \\
h_{k_1}(x) &: \mbox{otherwise}.
\end{cases}\label{eq:define-gs}
\end{align}
\item We set $\cF = \{ f^\star \} \cup \bigcup_{t \in [T], i \in [M], k_0 \in [K+1], k_1 \in [K]} \{ g_{t,i,b,k_0, k_1 } \}$.
\end{itemize}
It is straightforward to see that realizability holds (i.e., $f^\star \in \cF$). Note that $|\cF| \leq O(MTK^2)$, and so the VC dimension of $\cF$ is $O(\log MTK) \leq O(\log TN)$. This is bounded above by $d$ as long as $TN \leq 2^{cd}$ for a sufficiently small constant $c$.
The distribution $\dist^\star$ is defined as $\mathrm{Unif}\{ ((i,1), \ldots, (i,T)) \}_{i \in [M]}$.

Now consider any algorithm $\Alg$, which functions as follows: first, it draws $N$ samples from $\dist^\star$, which we denote by $((I_n, t))_{t \in [T]}$, for $n \in [N]$; in particular, $I_n \sim \Unif([M])$ independently for each $n \in [N]$. Moreover, nature chooses $I^\star \sim \Unif([M])$, and the sequence of covariates that $\Alg$ observes is $(X_1, \ldots, X_T)$, where $X_t = (I^\star, t)$. In particular, the following procedure repeats for $T$ time steps:
\begin{enumerate}
\item At each time step $t$, write $\cW_t := \bigcup_{n \in [N], t \in [T]} \{ (I_n, t) \} \cup \bigcup_{s \in [t]} \{ (I^\star, s) \}$.
\item At time step $t$, $\Alg$ makes some number of queries to the ERM oracle, each of the form
\begin{align}
    \label{eq:alg-min-g}
\argmin_{f \in \cF} \sum_{x \in \cW_t} y_{j,x} \cdot f(x),
\end{align}
where $y_{j,x} \in \mathbb{R}$ are real numbers which can depend on the results of previous queries and the labels $Y_1, \ldots, Y_{t-1}$. Here we use $j$ to denote a \emph{global} index over all queries, i.e., \cref{eq:alg-min-g} denotes the $j$th ERM query over the entire execution of the algorithm. If there are multiple minimizers to \cref{eq:alg-min-g}, we assume that the ERM oracle picks one in $\tcF_t$ (defined below), if possible; conditioned on this, it returns one uniformly at random.
With a slight abuse of notation, we use the expression in \cref{eq:alg-min-g} (in particular, the use of $\argmin$) to denote the result returned by this ERM oracle.

After making each such query, which returns some $f\^{j} \in \cF$, we allow $\Alg$ to query any of the functions $f\^{j'}$, $j' \leq j$, at any number of points in $\cX$ before making the subsequent query.
\item $\Alg$ then predicts a label $\hat Y_t$ for $X_t$, and observes the true label $Y_t = f^\star(X_t)$.
\end{enumerate}
Let $t\^j\in[T]$ denote the time step at which the $j$th oracle query is made, and let $J$ denote the total number of oracle queries (we have $J \leq N$). Next, we define (random) subclasses $\tcF_t \subset \cF$ as follows:
\begin{align}
\cF \backslash \tcF_t :=&  \{ f^\star \} \cup \{ g_{s, I^\star,b, K+1, k_1} \colon s \geq t+1, b \in \{0,1\}, k_1 \in [K]\}\nonumber\\
& \cup \{ g_{s,i,b,K+1, k_1} \colon s \in [T], i \neq I^\star, b \in \{0,1\}, k_1 \in [K]\}\nonumber.
\end{align}
In words, functions in $\cF \backslash \tcF_t$ are those which reveal information about $f^\star$ which has not already been observed (as of step $t$) in the form of the labels $Y_1, \ldots, Y_{t-1}$.
Finally, we let $\tAlg$ denote the algorithm whose execution is identical to that of $\Alg$, except that for the $j$th oracle call, the minimization in \cref{eq:alg-min-g} is performed over $f \in \tcF_t$.

\begin{lemma}
    \label{lem:talg-regret}
Let $\tilde Y_1, \ldots, \tilde Y_T$ denote the predictions of $\tAlg$. Then it holds that
\begin{align}
\En\left[ \sum_{t=1}^T \ind\{ \tilde Y_t \neq Y_t \}  \right]  = \frac{T}{2},\nonumber
\end{align}
where the expectation is over the randomness of $\tAlg$ and the choice of $\cF$.
\end{lemma}
\begin{proof}
The claim follows immediately from the observation that the random variable $\tilde Y_t$ is independent of $Y_t$ (which is distributed as $\Unif\{0,1\}$), due to the definition of $\tAlg$ and the fact that $f^\star \colon \cX \to \{0,1\}$ is chosen to be a uniformly random function.
\end{proof}

The next lemma shows that, with high probability over the joint probability space including all random variables defined above, $\Alg$ and $\tAlg$ behave identically.
\begin{lemma}
    \label{lem:talg-alg}
With probability at least $1-6J^{3/2}TN/K^{1/4} - N/M$ (over the execution of $\Alg, \tAlg$, and the choice of $\cF$), the execution of $\Alg$ is identical to that of $\tAlg$.
\end{lemma}
\begin{proof}
Note that $I^\star \not \in \{I_1, \ldots, I_N \}$ with probability at least $1-\frac{N}{M}$. Let us henceforth condition on this event. We now consider the execution of $\tAlg$: we will show that with high probability, performing the minimization in \cref{eq:alg-min-g} over $\cF$ as opposed to $\tcF_t$ at each oracle query $j \in [J]$, would not change any of the answers returned by the ERM oracle. Thus, on this high probability event, the execution of $\Alg$ agrees with that of $\tAlg$.

Recall that $\tAlg$ makes a total of $J$ ERM oracle queries; we let the response of the $j$th ERM oracle query be denoted $ f\^j\in \cF$. For each $j \in [J]$, suppose that $ f\^j$ is of the form $g_{t,i, b,k_0\^j, k_1\^j}$ for some $k_0\^j, k_1\^j \in [K+1]$.  Write $\cK\^j := \{ k_0\^1, \ldots, k_0\^{j-1}, k_1\^1, \ldots, k_1\^{j-1} \}$.
Let $\scrF\^j$ denote the sigma-algebra generated by $I_1, \ldots, I_N, I^\star, f\^1, \ldots, f\^{j-1}, \cK\^j, Y_1, \ldots, Y_{t\^j-1}$.
Note that the transcript of responses that $\tAlg$ has received from the ERM oracle queries prior to the $j$th one is $\scrF\^j$-measurable. We state the following lemma, whose proof is deferred to the end of the section.
\begin{lemma}
    \label{lem:close-uniform}
Fix any subset $\cK \subset [K]$ and step $j \in [J]$, and set $\eps= \frac{2j |\cK|}{K}$. Then with probability at least $1-\sqrt{\eps}$ over the execution of $\tAlg$ up to step $j$ and the choice of $\cF$, the distribution of the tuple of functions $(h_k)_{k \in \cK}$ conditioned on $\scrF\^j$ is $\sqrt{\eps}$-close to uniform in total variation distance.
\end{lemma}
Let us write $\bar\eps:= \frac{2J}{\sqrt K}$.
At any step $j$ of the execution of $\tAlg$, note that $f^\star(\cW_{t\^j} \backslash \{ X_1, \ldots, X_{t\^j-1}\})$ is uniformly random conditioned on $\scrF\^j$. (Indeed, the functions in $\tilde \cF_1, \ldots, \tilde \cF_{t\^j}$ contain no information about $f^\star$ on $\cW_{t\^j} \backslash \{ X_1, \ldots, X_{t\^j-1} \}$.) By \cref{lem:close-uniform}, the distribution of $(h_1, \ldots, h_{\sqrt{K}})$ is $\sqrt{\bar\eps}$-close to uniform in total variation distance, under some event $\cE_j$ that occurs with probability $1-\sqrt{\bar\eps}$. Thus, under $\cE_j$, the joint distribution of $f^\star, h_1, \ldots, h_{\sqrt{K}}$ is $\sqrt{\bar\eps}$-close to uniform on $\cW_{t\^j} \backslash \{ X_1, \ldots, X_{t\^j-1} \}$ in total variation distance. Under the event $\cE_j$, we have
\begin{align}
& \P \left( \argmin_{f \in \cF} \sum_{x \in \cW_{t\^j}} y_{j,x} \cdot f(x) = f^\star \middle| \scrF\^j \right) \label{eq:ft-part-1}\\
\leq & \P \left(\sum_{x \in \cW_{t\^j}} y_{j,x} \cdot f^\star(x) < \min_{\substack{k \leq \sqrt{K},\\ b \in \{0,1\}}} \sum_{x \in \cW_{t\^j}} y_{j,x} \cdot g_{t\^j, I^\star,b, K+1, k}(x) \middle| \scrF\^j\right)
 \leq  \sqrt{\bar\eps} +  \frac{1}{\sqrt{K}}\nonumber.
\end{align}
In the first inequality above, we have used that $\Alg$ will only return $f^\star$ if none of $g_{t\^j, I^\star,b, K+1, k}$, $k \leq \sqrt{K}, b \in \{0,1\}$ achieve the minimum (as all of the latter functions belong to $\tcF_{t\^j}$).
The second inequality above uses that, $f^\star$ and $g_{t\^j, I^\star, b,K+1, k}$ ($k \leq \sqrt{K}$) agree on $\{ X_1, \ldots, X_{t\^j-1} \}$, and the joint distribution of these $\sqrt{K}+1$ functions, over $b \sim \Unif\{0,1\}$, is $\sqrt{\bar\eps}$ close to uniform on $\cW_{t\^j} \backslash \{ X_1, \ldots, X_{t\^j-1} \}$ (conditioned on $\scrF\^j$, under the event $\cE_j$).

Similarly, we have that under $\cE_j$, for any $t \in [T]$ and $I \in \{ I_1, \ldots, I_N \}$,
\begin{align}
&     \P \left( \argmin_{f \in \cF} \sum_{x \in \cW_{t\^j}} y_{j,x} \cdot f(x) \in \{ g_{t,I, b,K+1, k_1} \colon k_1 \in [K],b \in \{0,1\}\} \middle| \scrF\^j \right) \label{eq:ft-part-2}\\
\leq & \P \left( \min_{\substack{k_1 \in [K],\\ b \in \{0,1\}}} \sum_{x \in \cW_{t\^j}} y_{j,x} \cdot g_{t, I,b, K+1, k_1}(x) < \min_{\substack{k_1 \in [K], \\ k_0 \in [\sqrt{K}], \\ b \in \{0,1\}}} \sum_{x \in \cW_{t\^j}} y_{j,x} \cdot g_{t, I,b,k_0, k_1}(x) \middle| \scrF\^j \right)\nonumber\\
 \leq &  \sqrt{\bar\eps} + \frac{1}{\sqrt{K}}\nonumber.
\end{align}
Again, in the first inequality, we have used that $\Alg$ will only return one of $g_{t,I,b,K+1, k_1}$ if none of $g_{t,I,b,k_0, k_1}$, $k_1 \in [K], k_0 \in [\sqrt{K}], b \in \{0,1\}$ achieve the minimum, as all of the latter functions belong to $\tcF_{t\^j}$.
In the second inequality, we have used that $g_{t,I,b,K+1, k_1}$ and $g_{t,I,b,k_0, k_1}$ differ only on $(I, 1), \ldots, (I,t-1)$, where their values are given by those of $h_{K+1}$ and $h_{k_0}$, respectively. The second inequality then follows from the fact that under $\cE_j$, conditioned on $\scrF\^j$, the functions $h_1, \ldots, h_{\sqrt{K}}, h_{K+1}$, restricted to $(I, 1), \ldots, (I,t-1)$, have distribution which is $\sqrt{\bar\eps}$-close to uniform in total variation distance (as argued above using \cref{lem:close-uniform}).

Finally, note that $g_{s, I^\star, b, K+1, k_1}$, for any $s \geq t+1$, agrees with $g_{t, I^\star, f^\star(X_t), K+1, k_1}$ on all of $\cW_t$, meaning that the ERM oracle will never return $g_{s, I^\star, b, K+1, k_1}$ for any $s \geq t+1, b \in \{0,1\}, k_1 \in [K]$. Using this fact as well as \cref{eq:ft-part-1,eq:ft-part-2}, we see that under $\bigcap_{j \in [J]} \cE_j$, with probability at least $1 - 2JTN \cdot (\sqrt{\bar\eps} + 1/\sqrt{K})$ under the execution of $\tAlg$, for all $J$ ERM oracle queries $j \in [J]$, the oracle response in \cref{eq:alg-min-g} is the same if the minimization is done over $\cF$ as if it is done over $\tcF_{t\^j}$. Since $\bigcap_{j \in [J]} \cE_j$ occurs with probability at least $1-J\sqrt{\bar\eps}$, the result follows.

\end{proof}
We may now conclude the proof of \cref{thm:realizable-erm-oracle-lb}. First, \cref{lem:talg-regret} gives that
\begin{align}
\P_{\tAlg} \left( \reg(\cF, T) \geq T/2 \right) \geq 1/2\nonumber.
\end{align}
As long as we choose $M \geq 10N$ and $K \geq (600 J^{3/2} TN)^4$, \cref{lem:talg-alg} gives that the execution of $\Alg$ and $\tAlg$ have total variation distance at most $1/5$. It follows that $\P_{\Alg}(\reg(\cF,T) \geq T/2) \geq 3/10$, i.e., $\En_{\Alg}[\reg(\cF,T)] \geq \Omega(T)$.

\end{proof}

Finally, it remains to prove \cref{lem:close-uniform}.
\begin{proof}[Proof of \cref{lem:close-uniform}]
We couple the execution of $\tAlg$ to that of a modified algorithm that differs as follows: the ERM oracle returns functions $g'_{t,i,b,k_0, k_1}$ defined as follows (contrast with \cref{eq:define-gs}):
\begin{align}
g'_{t,i,b,k_0, k_1}(x) = \begin{cases}
h'_{k_0}(x) &: x \in \{(i,1), \ldots, (i,t-1) \} \\
 b &: x = (i, t) \\
h'_{k_1}(x) &: \mbox{otherwise},
\end{cases}\label{eq:define-gs-alg}
\end{align}
where $h_k = h_k'$ if $k \not \in \cK$, and $h_k'$ is a fresh uniformly random function for $k \in \cK$. We denote the resulting algorithm by $\Alg'$.

Under the execution of $\Alg'$, with probability $1$, conditioned on $\cK\^j$, the distribution of $(h_k)_{k \in \cK}$ is uniformly random (as the functions $h_k, k \in \cK$ are never used by the ERM oracle). The probability that the two executions differ is bounded above by the probability that $\cK \cap \cK\^j \neq \emptyset$, which is at most $\eps= \frac{2j | \cK|}{K}$, by symmetry. (In particular, at each oracle call $j$ of $\tAlg$, conditioned on $\cK \cap \cK\^j = \emptyset$, there is at most a $2|\cK|/K$ chance that the ERM oracle returns $g_{t,i,b,k_0,k_1}$ with $k_0 \in \cK$ or $k_1 \in \cK$.)

In particular, the distribution of the execution of $\Alg'$ is $\eps$-close to the distribution of the execution of $\Alg$ in total variational distance. The conclusion follows by the general fact that for random variables $(X,Y), (X', Y')$, we have $\TV{(X,Y)}{(X', Y')} \geq \En_{X}[\TV{Y|X}{Y'|X}]$; in particular, we use this fact with $X, X'$ being the random variables generating $\scrF\^j$ under $\tAlg, \Alg'$, respectively, and $Y,Y'$ being the random variables $(h_k)_{k \in \cK}$ under $\tAlg, \Alg'$, respectively.
\end{proof}

\section{Proofs from~\texorpdfstring{\cref{sec:cond}}{Section \getrefnumber{sec:cond}}}

\subsection{Proof of~\texorpdfstring{\cref{thm:cond-oracle-eff}}{Theorem \getrefnumber{thm:cond-oracle-eff}}}
\ClosedFormStrategy*
\begin{proof}
    Writing $\seq S_{1:t} = (\seq S_{1:t-1}, (X_t, Y_t))$, the  objective of~\cref{eq:relaxation-mgf-strategy} is minimized at the unconstrained minimizer
    \begin{equation}
    \label{eq:qt-unconstrained}
    \frac{1}{2} + \frac{\rel(\seq S_{1:t-1}, (X_t,1)) - \rel(\seq S_{1:t-1}, (X_t,0))}{2}.
    \end{equation}
    Since $L_t(f) = L_{t-1}(f) + \ind\{f(X_t) \neq y\}$ when $Y_t = y$, the relaxation~\cref{eq:rel-log-mgf} and the offset process $\off{y}$ in~\cref{eq:ayw} satisfy
    \[
    \rel(\seq S_{1:t-1}, (X_t, y)) = \frac{1}{\eta} \log \En \exp(\eta \off{y}(\seq W^t)),
    \]
    so that the unconstrained minimizer from~\cref{eq:qt-unconstrained} can be equivalently rewritten as
    \begin{equation}
    \label{eq:qt-unconstrained-2}
    \frac{1}{2} + \frac{1}{2\eta} \log \left(\frac{\En \exp(\eta \off{1}(\seq W^t))}{\En \exp(\eta \off{0}(\seq W^t))}\right),
    \end{equation}
    as claimed. It remains to show that this value lies in $[0,1]$. First, note that, with probability $1$:
    \begin{multline*}
        \Bigg|\sup_{f\in \cF}\left[ \sum_{s=t+1}^T \varepsilon_s \frac{2f(X_s) - 1}{2} - L_{t-1}(f) - \ind\{f(X_t) \neq 0\} \right] \\- \sup_{f\in \cF}\left[ \sum_{s=t+1}^T \varepsilon_s \frac{2f(X_s) - 1}{2} - L_{t-1}(f) - \ind\{f(X_t) \neq 1\} \right]\Bigg| \le 1,
    \end{multline*}
    and thus, for each $y \in\{0,1\}$
    \[
    \exp(\eta \off{y}(\seq W^t) \le \exp(\eta) \cdot \exp(\eta \off{1-y}(\seq W^t)).
    \]
    Taking expectations of both sides gives:
    \[
    \En \exp(\eta \off{y}(\seq W^t)) \le \exp(\eta) \En \exp(\eta \off{1-y}(\seq W^t)).
    \]
    This yields: 
    \[
    -\eta \le \log \left(\frac{\En \exp(\eta \off{1}(\seq W^t))}{\En \exp(\eta \off{0}(\seq W^t))}\right) \le \eta,
    \]
    which implies that~\cref{eq:qt-unconstrained-2} is bounded in $[0,1]$. Thus,
    \[
    q_t(\seq S_{t-1}, X_t) =     \frac{1}{2} + \frac{1}{2\eta} \log \left(\frac{\En \exp(\eta \off{1}(\seq W^t))}{\En \exp(\eta \off{0}(\seq W^t))}\right),
    \]
    as desired.
\end{proof}

\RelAdmissibility*
\begin{proof}
In the remainder, we let $\seq X_{t:T} \sim \dist(\cdot \mid \seq X_{1:t-1})$ and $\boldsymbol{\varepsilon}_{t+1:T} \sim \Unif(\{\pm 1\}^{T-t})$ and suppress the distributions from the expectations for notational brevity.
Using the closed-form expression for $q_t$ from~\cref{prop:closed-form} (namely, from \cref{eq:qt-unconstrained}), it suffices to show
\begin{align*}
    \En_{X_t} \exp\left(\frac{\eta}{2} + \frac{\eta}{2} \left[\rel(\seq S_{1:t-1}, (X_t,1)) + \rel(\seq S_{1:t-1}, (X_t,0))\right] \right) \le \exp\left( \eta \rel(\seq S_{1:t-1}) \right),
\end{align*}
or, equivalently,
\begin{align}
    \En_{X_t} \exp\left(\eta \En_{\varepsilon_t}\left[ \rel\left(\seq S_{1:t-1}, \left(X_t, \frac{1+\varepsilon_t}{2}\right)\right) + \frac{1}{2} \right]\right) \le \exp\left( \eta \rel(\seq S_{1:t-1}) \right).
    \label{eq:exp-relaxation-cond}
\end{align}
For every $\varepsilon_t \in \{\pm 1\}$, we have
\begin{multline*}
    \rel\left(\seq S_{1:t-1}, \left(X_t, \frac{1+\varepsilon_t}{2}\right)\right) + \frac{1}{2}  \\ = \frac{1}{\eta} \log
\En_{\substack{\seq X_{t+1:T}\\
\boldsymbol{\varepsilon}_{t+1:T}}}
\exp\left( \eta \sup_{f\in \cF}\left[ \sum_{s=t+1}^T \varepsilon_s \frac{2f(X_s)-1}{2} - \left(\ind\left\{f(X_t) \neq \frac{1+\varepsilon_t}{2}\right\} - \frac{1}{2}\right) - L_{t-1}(f)\right]\right) \\
= \frac{1}{\eta} \log
\En_{\substack{\seq X_{t+1:T},\\
\boldsymbol{\varepsilon}_{t+1:T}}}
\exp\left( \eta \sup_{f\in \cF}\left[\sum_{s=t+1}^T \varepsilon_s \frac{2f(X_s)-1}{2} + \varepsilon_t \frac{2f(X_t)-1}{2} - L_{t-1}(f)\right]\right) \\
= \frac{1}{\eta} \log
\En_{\substack{\seq X_{t+1:T},\\
\boldsymbol{\varepsilon}_{t+1:T}}}
\exp\left( \eta \sup_{f\in \cF}\left[\sum_{s=t}^T \varepsilon_s \frac{2f(X_s)-1}{2} - L_{t-1}(f)\right]\right).
\end{multline*}
Then,~\cref{eq:exp-relaxation-cond} can be restated as
\begin{multline*}
    \En_{X_t} \exp\left( \En_{\varepsilon_t} \log
\En_{\substack{\seq X_{t+1:T},\\
\boldsymbol{\varepsilon}_{t+1:T}}}
\exp\left( \eta \sup_{f\in \cF}\left[ \sum_{s=t}^T \varepsilon_s \frac{2f(X_s) - 1}{2} - L_{t-1}(f)\right]\right)\right) \\ \le
\En_{\substack{\seq X_{t:T},\\
\boldsymbol{\varepsilon}_{t:T}}}
\exp\left( \eta \sup_{f\in \cF}\left[ \sum_{s=t}^T \varepsilon_s \frac{2f(X_s) - 1}{2} - L_{t-1}(f)\right]\right).
\end{multline*}
The proof is concluded using Jensen's inequality $\En_{\varepsilon_t} \log \le \log \En_{\varepsilon_t}$.
\end{proof}

\IdealizedRegret*
\begin{proof}
    Recall the predictable loss process $\bar L_t$ and the potential $\Phi_t$ from~\cref{eq:mgf-potential}. By~\cref{lemma:rel-admissibility}, $\{\Phi_t\}_{t \le T}$ is a supermartingale, and moreover \[\Phi_T = \exp\left(\eta \left( \bar L_T - \inf_{f} L_T(f)\right)\right).\] Thus, using the tower rule and the Donsker--Varadhan formula~\citep{polyanskiy2025information}, we have
    \begin{align*}
        \En_{\dist^\star} \reg(\cF,T) &=
        \En_{\dist^\star} \left(\bar L_T - \inf_{f} L_T(f)\right) \\
        &\le \frac{1}{\eta} \log \En_{\dist} \exp\left(\eta \left(\bar L_T - \inf_{f} L_T(f)\right)\right) + \frac{1}{\eta} \KL{\dist^\star}{\dist} \\
        &= \frac{1}{\eta} \log \En_{\dist} \Phi_T + \frac{1}{\eta} \KL{\dist^\star}{\dist} \\
        &\le \frac{1}{\eta}\log \Phi_0 + \frac{1}{\eta} \KL{\dist^\star}{\dist}.
    \end{align*}
    Now, we upper bound $\frac{1}{\eta}\log \Phi_0$. By~\cref{eq:mgf-potential} and the definition of the relaxation in~\cref{eq:rel-log-mgf}, we have
    \begin{align*}
        \frac{1}{\eta}\log \Phi_0 &= \frac1\eta \log
\En_{\substack{\seq X \sim \dist,\\
\boldsymbol{\varepsilon} \sim \{\pm 1\}^T}}
\exp\left( \eta \sup_{f\in \cF} \sum_{s=1}^T \varepsilon_s \frac{2f(X_s) - 1}{2}\right).
    \end{align*}
    Note that, conditional on $\seq X$, the function
    \[
    \boldsymbol{\varepsilon}\mapsto \sup_{f} \sum_{s=1}^T \varepsilon_s \frac{2f(X_s) - 1}{2}
    \]
    satisfies the bounded differences condition. Thus, conditional on $\seq X$, we have
    \begin{align*}
    \En_{\boldsymbol{\varepsilon} \sim \{\pm 1\}^T}
\exp\left( \eta \sup_{f\in \cF} \sum_{s=1}^T \varepsilon_s \frac{2f(X_s) - 1}{2}\right) &\lesssim \exp\left(O(\eta^2 T) + \eta \En_{\boldsymbol\varepsilon}\sup_{f\in \cF} \sum_{s=1}^T \varepsilon_s \frac{2f(X_s) - 1}{2}\right) \\
&\lesssim \exp\left(O(\eta^2 T + \eta \sqrt{dT}) \right).
    \end{align*}
    Thus, we have
    \[
    \frac{1}{\eta} \log \Phi_0 \lesssim \sqrt{dT} + \eta T.
    \]
    Putting everything together, we have
    \[
    \En_{\dist^\star} \reg(\cF,T) \lesssim \sqrt{dT} + \eta T + \frac{1}{\eta} \KL{\dist^\star}{\dist}.
    \]
\end{proof}

\MSEBound*
\begin{proof}
Throughout, $q_t$ is as in~\cref{prop:closed-form} and $\hat q_t$ is the plug-in estimator in~\cref{eq:plug-in}. Note that, for any $j$, we have \[\exp(-\eta) \le \frac{\exp(\eta \off{0} (\seq W^t_j))}{\exp(\eta \off{1}(\seq W^t_j))} \le  \exp(\eta),\] and \[\exp(-\eta) \le \frac{\En \exp(\eta \off{0} (\seq W^t))}{\En \exp(\eta \off{1}(\seq W^t))} \le \exp(\eta).\] Also, $\log$ is $\exp(\eta)$-Lipschitz on $[\exp(-\eta), \exp(\eta)]$, thus
\begin{align*}
(\hat q_t- q_t)^2 &=  \left( \frac{1}{2\eta} \log\left( \frac{\frac{1}{N} \sum_{j = 1}^N \exp(\eta \off{1}(\seq W^t_j))}{\frac{1}{N} \sum_{j = 1}^N \exp(\eta \off{0}(\seq W^t_j))}\right) - \frac{1}{2\eta} \log\left( \frac{\En \exp(\eta \off{1}(\seq W^t))}{\En \exp(\eta \off{0}(\seq W^t))}\right) \right)^2\\
&\le \frac{\exp(2\eta)}{4\eta^2} \left(\frac{\frac1N\sum_j \exp(\eta \off{1}(\seq W^t_j))}{\frac1N\sum_j \exp(\eta \off{0}(\seq W^t_j))}-\frac{\En \exp(\eta \off{1}(\seq W^t))}{\En \exp(\eta \off{0}(\seq W^t))} \right)^2 \\ &=\frac{\exp(2\eta)}{4\eta^2} \left(\frac{\frac1N\sum_j\left(\exp(\eta \off{1}(\seq W^t_j))-\frac{\En \exp(\eta \off{1}(\seq W^t))}{\En \exp(\eta \off{0}(\seq W^t))}\exp(\eta \off{0}(\seq W^t_j))\right)}{\frac1N\sum_j \exp(\eta \off{0}(\seq W^t_j))}\right)^2.
\end{align*}
Let \[\cE:=\left\{\frac {1}{N}\sum_j \exp(\eta \off{0}(\seq W^t_j))\ge\frac{1}{2}\En \exp(\eta \off{0}(\seq W^t))\right\}.\] Note that \[\En \frac1N\sum_j \exp(\eta \off{0}(\seq W^t_j)) = \En \exp(\eta \off{0}(\seq W^t)),\]
and
\begin{align*}
\Var\left(\frac1N\sum_j \exp(\eta \off{0}(\seq W^t_j))\right) &\le \frac{1}{N} \En \left(\exp(\eta \off{0}(\seq W^t))\right)^2 \\ &\le \frac{\nvar}{N} \left(\En \exp(\eta \off{0}(\seq W^t))\right)^2.
\end{align*}
Thus, by Chebyshev's inequality, we have \[\P[\cE^c]\le \frac{\Var\left(\frac1N\sum_j \exp(\eta \off{0}(\seq W^t_j))\right)}{\left(\frac{1}{2} \En \exp(\eta \off{0}(\seq W^t))\right)^2} \le \frac{4\nvar}{N}.\]
On $\cE$, we have
\begin{align*}
    (\hat q_t -  q_t)^2 \ind\{\cE\} &\le \frac{\exp(2\eta)}{4\eta^2} \left(\frac{\frac1N\sum_j\left(\exp(\eta \off{1}(\seq W^t_j))-\frac{\En \exp(\eta \off{1}(\seq W^t))}{\En \exp(\eta \off{0}(\seq W^t))}\exp(\eta \off{0}(\seq W^t_j))\right)}{\frac{1}{2} \En \exp(\eta \off{0}(\seq W^t))}\right)^2 \\
    &= \frac{\exp(2\eta)}{\eta^2} \left(\frac{\frac1N\sum_j\left(\exp(\eta \off{1}(\seq W^t_j))-\frac{\En \exp(\eta \off{1}(\seq W^t))}{\En \exp(\eta \off{0}(\seq W^t))}\exp(\eta \off{0}(\seq W^t_j))\right)}{\En \exp(\eta \off{0}(\seq W^t))}\right)^2.
\end{align*}
Consider the numerator in the above. The summands are i.i.d., mean zero,
and each is bounded as
\begin{multline*}
\left|\exp(\eta \off{1}(\seq W^t_j))-\frac{\En \exp(\eta \off{1}(\seq W^t))}{\En \exp(\eta \off{0}(\seq W^t))}\exp(\eta \off{0}(\seq W^t_j)) \right| \\ = \exp(\eta \off{0}(\seq W^t_j)) \left|\frac{\exp(\eta \off{1}(\seq W^t_j))}{\exp(\eta \off{0}(\seq W^t_j))}-\frac{\En \exp(\eta \off{1}(\seq W^t))}{\En \exp(\eta \off{0}(\seq W^t))} \right|\\  \le (\exp(\eta) - \exp(-\eta)) \exp(\eta \off{0}(\seq W^t_j)),
\end{multline*}
almost surely, since both ratios lie in $[\exp(-\eta), \exp(\eta)]$. Thus,
\begin{multline*}
\En \left(\exp(\eta \off{1}(\seq W^t_j))-\frac{\En \exp(\eta \off{1}(\seq W^t))}{\En \exp(\eta \off{0}(\seq W^t))}\exp(\eta \off{0}(\seq W^t_j)) \right)^2 \\ \le (\exp(\eta) - \exp(-\eta))^2 \En \exp(2\eta \off{0}(\seq W^t)) \\
\le (\exp(\eta) - \exp(-\eta))^2 \cdot \nvar \left(\En \exp(\eta \off{0}(\seq W^t))\right)^2.
\end{multline*}
Using that the samples $\seq W^t_1, \ldots, \seq W^t_N$ are i.i.d., this implies
\begin{align*}
    \En [(\hat q_t -  q_t)^2 \ind\{\cE\}] &\le \frac{\exp(2\eta)}{\eta^2} \frac{\frac{1}{N^2}  \sum_{j} \En \left(\exp(\eta \off{1}(\seq W^t_j))-\frac{\En \exp(\eta \off{1}(\seq W^t))}{\En \exp(\eta \off{0}(\seq W^t))}\exp(\eta \off{0}(\seq W^t_j))\right)^2}{\left(\En \exp(\eta \off{0}(\seq W^t))\right)^2} \\
    &\le \frac{\exp(2\eta)}{\eta^2} \cdot \frac{\nvar (\exp(\eta) - \exp(-\eta))^2}{N}\\
    &\lesssim \frac{\nvar}{N},
\end{align*}
where the last step uses $\exp(\eta) - \exp(-\eta) \lesssim \eta$ and $\exp(2\eta) \lesssim 1$ for $\eta \le 1$.
On $\cE^{c}$, we simply use that $\hat q_t,  q_t \in [0,1]$, and get:
\[
\En [(\hat q_t -  q_t)^2 \ind\{\cE^c\}] \le \P[\cE^c] \lesssim \frac{\nvar}{N}.
\]
Summing the two bounds concludes the proof.
\end{proof}

\VEtaBound*
\begin{proof}
Fix $y \in \{0,1\}$, and set $\rade(\seq W^t) := \sup_f\sum_{s\ge t+1}\varepsilon_s\frac{2f(X_s)-1}{2}$ and $L^\star_{t-1} := \min_f L_{t-1}(f)$. With $\off{y}$ as in~\cref{eq:ayw}, we have
\[
\off{y}(\seq W^t) + L^\star_{t-1} \le \rade(\seq W^t), \qquad \En_{\seq W^t} [\off{y}(\seq W^t)] + L^\star_{t-1} \ge - 1.
\]
Thus,
\begin{align*}
\En \exp(2\eta \off{y}(\seq W^t)) &= \exp(-2\eta L^\star_{t-1})\En \exp(2\eta(\off{y}(\seq W^t)+L^\star_{t-1})) \\ &\le \exp(-2\eta L^\star_{t-1})\En \exp(2\eta \rade(\seq W^t)).
\end{align*}
Also,
\begin{align*}
\En \exp(\eta \off{y}(\seq W^t)) &= \exp(-\eta L^\star_{t-1})\En \exp(\eta(\off{y}(\seq W^t)+L^\star_{t-1}))\\
&\ge \exp(-\eta L^\star_{t-1}) \exp(\En \eta(\off{y}(\seq W^t)+L^\star_{t-1}))\\
&\ge \exp(-\eta L^\star_{t-1})\exp(-\eta),
\end{align*}
so the common factor $\exp(-2\eta L^\star_{t-1})$ cancels and
\[
\frac{\En \exp(2\eta \off{y}(\seq W^t))}{\left(\En \exp(\eta \off{y}(\seq W^t))\right)^2} \le \exp(2\eta) \cdot \En \exp(2\eta \rade(\seq W^t)).
\]
It remains to control the exponential moment of $\rade(\seq W^t)$. Note that, conditional on $\seq X_{t+1:T}$, $\rade(\seq W^t)$ satisfies the bounded differences condition w.r.t. the independent signs $\boldsymbol{\varepsilon}_{t+1:T}$. Thus,
\[
\log \En \left[\exp\left(2\eta \rade(\seq W^t) \right) \mid \seq X_{t+1:T} \right] \lesssim \eta \En\left[\rade(\seq W^t) \mid \seq X_{t+1:T}\right] + \eta^2 T.
\]
At the same time, almost surely,
\[
\En\left[\rade(\seq W^t) \mid \seq X_{t+1:T}\right] \lesssim \sqrt{dT}.
\]
Therefore,
\[
\log \En \exp\left(2\eta \rade(\seq W^t) \right) \lesssim \eta \sqrt{dT} + \eta^2 T.
\]
This implies
\[
\log \nvar \lesssim \eta \sqrt{dT} + \eta^2 T,
\]
as desired.
\end{proof}

\RegMGFIntermediate*
\begin{proof}
    Let $\reg^\dagger(\cF,T)$ be the regret of the idealized algorithm that plays~\cref{eq:relaxation-mgf-strategy}. From~\cref{lemma:idealized-regret}, we have
    \[
    \En_{\dist^\star} \reg^\dagger(\cF,T) \lesssim \sqrt{dT} + \eta T + \frac{1}{\eta} \KL{\dist^\star}{\dist}.
    \]
    Combining~\cref{lemma:mse-bound,lemma:v-eta}, we have
    \[
    \En |\hat q_t - q_t| \lesssim \sqrt{\frac{\exp(O(\eta \sqrt{dT} + \eta^2 T))}{N}}.
    \]
    Using the fact that regret is Lipschitz in the learner's strategy, we have
    \begin{align*}
    \En_{\dist^\star} \reg(\cF,T) 
    &\lesssim \En_{\dist^\star} \reg^\dagger(\cF,T) + T \sqrt{\frac{\exp(O(\eta \sqrt{dT} + \eta^2 T))}{N}} \\
    &\lesssim \sqrt{dT} + \eta T + \frac{1}{\eta} \KL{\dist^\star}{\dist} + T \sqrt{\frac{\exp(O(\eta \sqrt{dT} + \eta^2 T))}{N}},
    \end{align*}
    as desired.
\end{proof}

\CondOracleEff*
\begin{proof}
    Consider~\cref{algo:dist-transductive-reduction} with some $K \in [T]$, and set $\eta = 1/\sqrt{dL}$ and $N = L$. For every $k \in[K]$, let $B_k := [(k-1)L + 1, kL]$ be the rounds in $k^\text{th}$ epoch, and let $H_k = \bigcup_{i < k} B_i$. Let $\reg_k$ be the regret of the algorithm in epoch $k$:
    \[
    \reg_k(\cF, L) = \sum_{t \in B_k} \ind\{\hat Y_{t} \neq f^\star(X_{t})\} - \inf_{f}\sum_{t \in B_k} \ind\{f(X_{t}) \neq f^\star(X_{t})\}.
    \]
    Then, we have, with probability $1$,
    \begin{equation}
    \label{eq:reg-ub-per-epoch}
    \reg(\cF,T) \le \sum_{k=1}^{K} \reg_k(\cF, L).
    \end{equation}
    Conditionally on $X_{H_k}$, $\reg_k$ is upper bounded by~\cref{lemma:reg-mgf-intermediate} as:
    \[
    \En_{\dist^\star(\cdot \mid X_{H_k})} \reg_k(\cF, L) \lesssim \sqrt{dL} + \sqrt{dL} \cdot \KL{\dist^\star_{B_k}(\cdot \mid X_{H_k})}{\dist_{B_k}(\cdot \mid X_{H_k})}
    \]
    Thus, averaging out $X_{H_k}$, we have
    \[
    \En_{\dist^\star} \reg_k(\cF, L) \lesssim \sqrt{dL} + \sqrt{dL} \cdot \En_{\dist^\star}\KL{\dist^\star_{B_k}(\cdot \mid X_{H_k})}{\dist_{B_k}(\cdot \mid X_{H_k})}
    \]
    Plugging this into~\Cref{eq:reg-ub-per-epoch}, we have
    \begin{align}
        \En_{\dist^\star} \reg(\cF,T) &\le \sum_{k=1}^{K} \En_{\dist^\star} \reg_k(\cF, L) \notag \\
        &\le \sum_{k=1}^{K} \left[\sqrt{dL} + \sqrt{dL} \cdot \En_{\dist^\star}\KL{\dist^\star_{B_k}(\cdot \mid X_{H_k})}{\dist_{B_k}(\cdot \mid X_{H_k})}\right] \notag \\
        &= K \cdot \sqrt{dL} + \sqrt{dL} \KL{\dist^\star}{\dist} \notag \\
        &= \sqrt{dTK} + \sqrt{d \cdot \frac{T}{K}} \cdot \KL{\dist^\star}{\dist},
        \label{eq:reg-ub-with-k}
    \end{align}
    where in the penultimate equality we used the chain rule for KL.

    Now, to tune the value of $K$ in the above, we will run an adaptive variant of exponential weights algorithm over choices of $K = 2^j$, for $j \in \{0,1,\ldots,\lfloor \log T \rfloor\}$ with initial weights $w_K \propto 2^{-K}$ for each $K$ (so that $\log(1/w_K) = O(K)$). In particular, we use $\mathrm{AdaNormalHedge}$ from~\cite{luo2015achieving}. Note that this results in at most $\log T$ factor of computational overhead. Then, from~\cref{eq:reg-ub-with-k} and~\cite{luo2015achieving}, we get:
    \begin{align*}
    \En_{\dist^\star} \reg(\cF,T) &\lesssim \min_{K} \left\{ \sqrt{dTK} + \sqrt{d \cdot \frac{T}{K}} \cdot \KL{\dist^\star}{\dist} + \sqrt{T \log(1/w_K)}\right\} \\
    &\lesssim \min_{K} \left\{ \sqrt{dTK} + \sqrt{d \cdot \frac{T}{K}} \cdot \KL{\dist^\star}{\dist} + \sqrt{T K}\right\} \\
    &\lesssim \min_{K} \left\{ \sqrt{dTK} + \sqrt{d \cdot \frac{T}{K}} \cdot \KL{\dist^\star}{\dist} \right\} \\
    &\lesssim \sqrt{dT (1 + \KL{\dist^\star}{\dist})}.
    \end{align*}
    This concludes the regret bound proof. It remains to observe that the algorithm runs in $\poly(T)$ time in the ERM oracle model.
\end{proof}

\subsection{Proof of the Agnostic Lower Bound with Conditional Sampling Access}

\begin{restatable}{proposition}{ApproxAgnosticLB}
\label{prop:approx-agnostic-lb}
    Consider the class $\cF$ of thresholds on the unit interval $[0,1]$. Then, for any $D \in [T]$,
    and for any algorithm, there exists a pair $(\dist, \dist^\star)$ such that
    $\KL{\dist^\star}{\dist} = D$, and the algorithm suffers regret $\ge \sqrt{D T}$ in the
    agnostic setting.
\end{restatable}
\begin{proof}
The construction below is similar to the standard Littlestone dimension lower bound in agnostic learning \citep{ben2009agnostic}. Consider a depth-$D$ Littlestone tree $\seq z$, and let $N=\lfloor T/D \rfloor$. For each $\boldsymbol{\varepsilon}\in\left\{\pm 1\right\}^{\le D}$, create $N$ distinct copies of the point $\seq z\left(\boldsymbol{\varepsilon}\right)$. Specifically, let $\seq z^{i}\left(\boldsymbol{\varepsilon}\right)$ for $i\in\left[N\right]$ lie in a $\delta$-neighborhood of $\seq z\left(\boldsymbol{\varepsilon}\right)$, with $\delta$ small enough that relative order is preserved across all copies, meaning
\[
\seq z\left(\boldsymbol{\varepsilon}\right)\le \seq z\left(\boldsymbol{\sigma}\right) \Longleftrightarrow \seq z^{i}\left(\boldsymbol{\varepsilon}\right)\le \seq z^{j}\left(\boldsymbol{\sigma}\right)\text{ for all }i,j\in\left[N\right], \boldsymbol{\sigma}, \boldsymbol{\varepsilon} \in \{\pm1\}^{\le D}.
\]
Let $f^\star$ assign independent uniformly random labels to the points $\{\seq z^{i}\left(\boldsymbol{\varepsilon}\right):\boldsymbol{\varepsilon}\in\left\{\pm 1\right\}^{\le D},\, i\in\left[N\right]\}$. Define $\dist$ and $\dist^\star$ as follows. Under $\dist$, the sequence is generated in $D$ consecutive blocks: first output $\seq z^{1}\left(\emptyset\right),\dots,\seq z^{N}\left(\emptyset\right)$, then sample $\varepsilon_1\sim\left\{0,1\right\}$ uniformly and output $\seq z^{1}\left(\varepsilon_1\right),\dots,\seq z^{N}\left(\varepsilon_1\right)$, and continue analogously for $D$ levels. Under $\dist^\star$, the process is identical except that after outputting a block it chooses the next child by majority vote of the labels in the current block: after outputting $\seq z^{1}\left(\emptyset\right),\dots,\seq z^{N}\left(\emptyset\right)$, it sets $\varepsilon_1'\in\left\{0,1\right\}$ to be the majority label among $f^\star\left(\seq z^{1}\left(\emptyset\right)\right),\dots,f^\star\left(\seq z^{N}\left(\emptyset\right)\right)$ and then outputs $\seq z^{1}\left(\varepsilon_1'\right),\dots,\seq z^{N}\left(\varepsilon_1'\right)$, and so on. The two distributions differ only in the $D$ branch choices, hence
\[
\KL{\dist^\star}{\dist}= D.
\]
Since the labels are independent unbiased coins, any learner suffers expected loss $T/2$. On the other hand, by the Littlestone-tree property there exists a threshold $f^\dagger\in\cF$ that realizes the branch labels chosen along the $\dist^\star$ path, so within each visited block it predicts the majority label. If $S\sim\mathsf{Bin}\left(N,1/2\right)$ denotes the number of ones in a block, then the number of mistakes made by $f^\dagger$ on that block is $\min\left\{S,N-S\right\}$, and therefore
\[
\En\left[\min\left\{S,N-S\right\}\right]=\frac{N}{2}-\En\left[\left|S-\frac{N}{2}\right|\right]\le \frac{N}{2}-\Omega\left(\sqrt{N}\right).
\]
Summing over the $D$ blocks yields
\[
\En\left[\sum_{t=1}^{T}\ind\left\{f^\dagger\left(X_t\right)\neq f^\star\left(X_t\right)\right\}\right]\le \frac{DN}{2}-\Omega\left(D\sqrt{N}\right)=\frac{T}{2}-\Omega\left(D\sqrt{N}\right).
\]
The learner suffers expected loss $T/2$, since the labels at each new point presented are uniformly random conditioned on the history. Hence, the expected regret is at least $\Omega(D\sqrt{N})=\Omega(\sqrt{DT})$, which concludes the proof.
\end{proof}

\section{Proofs from~\texorpdfstring{\cref{sec:sim-families}}{Section \getrefnumber{sec:sim-families}}}

\subsection{Proof from~\texorpdfstring{\cref{sec:kolmogorov}}{Section \getrefnumber{sec:kolmogorov}}}

\label{appendix:kolmogorov}

\UniversalDist*
\begin{proof}
    Suppose $1^t$ is given as input. Then, we implement sampling from $\mu^t$ as follows. First, we describe sampling from an auxiliary distribution $\tilde \mu^t$ supported on encodings of elements of $\cX^t$. We sample a random program length $n$ from the distribution supported on $[3p(t)]$ with probabilities proportional to $1/n^2$, sample $\tilde \pi \sim \mathsf{Unif}(\{0,1\}^n)$, and sample a seed $\tilde s \sim \mathsf{Unif}(\{0,1\}^{p(t)})$. Then, we execute $U$ on $(\tilde \pi, 1^t \# \tilde s)$ for $p(t)$ steps. If the output is a valid encoding of a string in $\cX^t$ (that is, it is a delimiter-separated sequence of $t$ binary strings), we output whatever it outputs. Otherwise, we output an encoding of an arbitrary prespecified $\seq x^0 \in \cX^t$. Let $\tilde \mu^t$ be the distribution over the output. Also, let $\nu^t$ be the uniform distribution on strings in $\Sigma^{\le p(t)}$ with exactly $t$ delimiters ending in a delimiter. Note that any such string is a valid encoding of an element of $\cX^t$. Then, we let
    \[
    \mu^t := \frac{1}{2} \tilde \mu^t + \frac{1}{2} \nu^t.
    \]
    Note that sampling from $\mu^t$ can be implemented in $\poly(p(t))$ time.

    Now, we show the desired density ratio upper bound. Note that, for any $\seq x$ such that $\langle \seq x \rangle \in \supp(P_{1:t}),$ we have $\langle \seq x \rangle \in \Sigma^{\le p(t)}$ (since it takes time $\le p(t)$ to output $\seq x$) and $\langle \seq x \rangle$ has exactly $t$ delimiters. Consider two cases. First, if $K^p(P) \ge 3p(t)$, we have
    \[
    \log \frac{P_{1:t}(\langle \seq x \rangle)}{\mu^t(\langle \seq x\rangle)} \le \log \frac{1}{\nu^t(\langle \seq x\rangle)} + O(1) \le \log|\supp(\nu^t)| + O(1).
    \]
    Using elementary counting, we have:
    \[
    |\supp(\nu^t)| = \sum_{\ell = t}^{p(t)} \binom{\ell - 1}{t-1} 2^{\ell - t} \le p(t) 2^{2p(t)} \le 2^{3p(t)}.
    \]
    Thus,
    \[
    \log \frac{P_{1:t}(\langle \seq x \rangle)}{\mu^t(\langle \seq x\rangle)} \le 3p(t) + O(1) \le K^p(P) + O(1),
    \]
    which concludes the proof in this case. Now, consider the case $K^p(P) < 3p(t)$. Since $P$ is $p$-time samplable, by \cref{def:samplable}, there exists a program $\pi^\star$ with $|\pi^\star| = K^p(P)$ such that $U^{p(t)}(\pi^\star, 1^t \# s) \sim P_{1:t}$ when $s \sim \mathsf{Unif}(\{0,1\}^{p(t)})$. Recall that the program $\tilde \pi$ and the seed $\tilde s$ are drawn independently in the construction of $\tilde \mu^t$, and since $|\pi^\star| < 3p(t)$, we have
    \[
    \P[\tilde \pi = \pi^\star] \gtrsim \frac{1}{|\pi^\star|^2 \cdot 2^{|\pi^\star|}} = \frac{1}{(K^p(P))^2 \cdot 2^{K^p(P)}}.
    \]
    Conditioned on the event $\{\tilde \pi = \pi^\star\}$, the output of the construction is exactly $U^{p(t)}(\pi^\star, 1^t \# \tilde s)$ with $\tilde s \sim \mathsf{Unif}(\{0,1\}^{p(t)})$, which by the above is distributed as $P_{1:t}$. Consequently, for any $\langle \seq x \rangle \in \supp(P_{1:t})$ we have $\tilde \mu^t(\langle \seq x\rangle) \ge \P[\tilde \pi = \pi^\star] \cdot P_{1:t}(\langle \seq x\rangle)$, and since $\mu^t \ge \frac{1}{2}\tilde \mu^t$,
    \[
    \log \frac{P_{1:t}(\langle \seq x \rangle)}{\mu^t(\langle \seq x \rangle)} \le \log \frac{2}{\P[\tilde \pi = \pi^\star]} \le K^{p}(P) + 2 \log K^p(P) + O(1),
    \]
    as desired. Moreover, it is easy to see that $\mu^t$ is only supported on valid encodings of $\cX^t.$
    This concludes the proof.
\end{proof}

\SamplableNature*
\begin{proof}
    We run~\cref{algo:multi-cover} with number of layers $L = 1$, using $\mu^T$ as the simulator, and the total number of samples $N_1 = T$. Then, from~\cref{thm:universal-dist}, it takes time $O(T \cdot p'(T))$ to draw $T$ samples from $\mu^T$, and it takes time $T^{O(d)}$ to construct a cover $\cG_1$ given access to a realizable ERM oracle w.r.t. $\cF$. The rest of the proof is concluded similarly to~\cref{thm:stat-dist-trans}. By~\cref{lemma:cover-analysis}, we have:
    \begin{equation}
    \label{eq:kolmogorov-approx-err}
    \En_{\seq Z} \En_{\seq X \sim \mu^T} \min_{g \in \cG_1} \norm{f-g}_{\seq X} \lesssim \frac{d}{N}.
    \end{equation}
    Letting $\eta = 1$ be the learning rate, we then have the following regret guarantee from Corollary 2.3 in~\cite{cesa2006prediction}:
    \begin{align*}
    \reg &\lesssim \log |\cG_1| + T \cdot \min_{g \in \cG_1} \norm{f-g}_{\seq X} \\
    &\lesssim d \log T + T \cdot \min_{g \in \cG_1} \norm{f-g}_{\seq X}.
    \end{align*}
    Thus,
    \[
        \En_{\seq X \sim \dist^\star_{1:T}} \reg \lesssim d \log T + T \cdot \En_{\seq X \sim \dist^\star_{1:T}} \min_{g \in \cG_1} \norm{f-g}_{\seq X}.
    \]
    \cref{thm:universal-dist} tells us that log density ratio between $\dist^\star$ and $\mu$ can be upper bounded uniformly:
    \[
    \log \frac{\mathrm{d}\dist^\star_{1:T}}{\mathrm{d}\mu^T} \le K^p(\dist^\star) + 2\log K^p(\dist^\star) + O(1).
    \]
    Let
    \[
    C_{\dist^\star}:= 2^{K^p(\dist^\star) + 2\log K^p(\dist^\star)}.
    \]
    Then, using~\cref{eq:kolmogorov-approx-err} and a change of measure from $\mu^T$ to $\dist^\star_{1:T}$, we have
    \begin{align*}
        \En_{\seq X \sim \dist^\star_{1:T}} \reg &\lesssim d \log T + C_{\dist^\star} \cdot T \cdot \En_{\seq X \sim \mu^T} \min_{g \in \cG_1} \norm{f-g}_{\seq X} \\
        &\lesssim d \log T + C_{\dist^\star} \cdot T \cdot \frac{d}{N} \\
        &= d \log T + d \cdot C_{\dist^\star},
    \end{align*}
    as desired. Note that it takes time $O(|\cG|) \le \poly(T^d)$ per round to run the exponential weights update. Thus, the algorithm runs in time $\poly(T^d, p(T))$. This concludes the proof.
\end{proof}

\subsection{Proofs from~\texorpdfstring{\cref{sec:lds}}{Section \getrefnumber{sec:lds}}}

Throughout this subsection, we use the notation introduced in~\cref{sec:lds}. We begin by analyzing the worst-case KL divergence between our simulator and the class of LDS.

\begin{lemma}\label{lemma:lds-compressor}
    Let $\dist$ be as defined in~\cref{eq:universal-dist-lds}. Then, for any $A^\star \in \Theta$ with $\rho(A^\star) \le 1$ and $X_0 \in \ball(0, R)$, we have
    \[
    \KL{\dist^{A^\star} (\cdot \mid X_0)}{\dist(\cdot \mid X_0)} \le O\left(n^3 \log\left(nBT(R/\sigma + 1)\right)\right).
    \]
\end{lemma}
\begin{proof}
Fix $A^\star \in \Theta$ with $\rho(A^\star) \le 1$ and $X_0 \in \ball(0, R)$ arbitrarily, and let $\eps \in (0, B]$. Define the localized neighborhood
\[
\ball_\infty(A^\star, \eps) := \{A \in \bbR^{n \times n} \colon \norm{A - A^\star}_\infty \le \eps\}.
\]
Under the Gaussian prior $\pi = \cN(0, B^2)^{n \times n}$, each coordinate $A_{ij}$ is independent and its density on $[A^\star_{ij} - \eps, A^\star_{ij} + \eps]$ is lower-bounded by $\frac{1}{\sqrt{2\pi} B} \exp\left(-\frac{(|A^\star_{ij}| + \eps)^2}{2B^2}\right)$, so
\[
\pi(\ball_\infty(A^\star, \eps)) \ge \prod_{i,j} \frac{2\eps}{\sqrt{2\pi} B} \exp\left(-\frac{(|A^\star_{ij}| + \eps)^2}{2B^2}\right) = \left(\frac{2\eps}{\sqrt{2\pi} B}\right)^{n^2} \exp\left(-\frac{1}{2B^2}\sum_{i,j}(|A^\star_{ij}| + \eps)^2\right).
\]
Since $A^\star \in [-B, B]^{n\times n}$ and $\eps \le B$, $(|A^\star_{ij}| + \eps)^2 \le 4B^2$, so the exponent is at most $2n^2$. Hence
\[
\log\frac{1}{\pi(\ball_\infty(A^\star, \eps))} \le n^2\log(B/\eps) + O(n^2).
\]
By a standard localization argument (e.g.,~\citep{polyanskiy2025information}),
\[
\KL{\dist^{A^\star}(\cdot \mid X_0)}{\dist(\cdot \mid X_0)} \le n^2\log(B/\eps) + O(n^2) + \sup_{A \in \ball_\infty(A^\star, \eps)} \KL{\dist^{A^\star}(\cdot \mid X_0)}{\dist^{A}(\cdot \mid X_0)}.
\]
For the supremum, since the initial state $X_0$ is fixed under both laws, the chain rule for KL together with the closed form \[\KL{\cN(\mu_1, \sigma^2 I_n)}{\cN(\mu_2, \sigma^2 I_n)} = \norm{\mu_1 - \mu_2}_2^2/(2\sigma^2)\] gives
\begin{equation}\label{eq:lds-chain-rule}
\KL{\dist^{A^\star}(\cdot \mid X_0)}{\dist^{A}(\cdot \mid X_0)} = \frac{1}{2\sigma^2}\sum_{t=1}^T \En_{\dist^{A^\star}(\cdot \mid X_0)}\norm{(A^\star - A) X_{t-1}}_2^2.
\end{equation}
For $A \in \ball_\infty(A^\star, \eps)$ we have $\norm{A^\star - A}_F \le \eps n$, so $\norm{(A^\star - A) X_{t-1}}_2^2 \le \eps^2 n^2 \norm{X_{t-1}}_2^2$, and it remains to bound $\sum_{t=1}^T \En\norm{X_{t-1}}_2^2$. It is well-known that the operator norm of powers of marginally stable matrices grows polynomially~\citep{horn2012matrix,sarkar2019near,simchowitz2018learning}. In particular, we will use the following statement, the proof of which is deferred to~\cref{lemma:matrix-power-growth} for self-containedness.
\[
\norm{(A^\star)^s}_{\mathrm{op}} \le n \cdot (\norm{A^\star}_{\mathrm{op}} + 1)^{n-1} \cdot s^{n-1}.
\]
Now, since $\norm{A^\star}_{\mathrm{op}} \le \norm{A^\star}_F \le Bn$, the above gives
\[
\norm{(A^\star)^s}_{\mathrm{op}} \le n (Bn + 1)^{n-1} s^{n-1} =: C_n \cdot s^{n-1},
\]
where $C_n := n(Bn+1)^{n-1}$ depends only on $B$ and $n$. Therefore $\norm{(A^\star)^s}_F^2 \le n \norm{(A^\star)^s}_{\mathrm{op}}^2 \le n C_n^2 s^{2(n-1)}$. Decomposing $X_t = (A^\star)^t X_0 + \sum_{s=0}^{t-1}(A^\star)^s\eta_{t-1-s}$ and using independence of the $\eta_s$,
\[
\En \norm{X_t}_2^2 = \norm{(A^\star)^t X_0}_2^2 + \sigma^2\sum_{s=0}^{t-1}\norm{(A^\star)^s}_F^2 \le C_n^2 \cdot t^{2(n-1)} R^2 + \sigma^2 n C_n^2 \cdot t^{2n-1}.
\]
Summing $\sum_{t=1}^{T-1} t^{2(n-1)} \le T^{2n-1}$ and $\sum_{t=1}^{T-1} t^{2n-1} \le T^{2n}$, and adding $\norm{X_0}_2^2 \le R^2$, yields
\begin{equation}
\label{eq:lds-second-moment-sum}
\sum_{t=1}^{T} \En\norm{X_{t-1}}_2^2 \le 2 C_n^2 T^{2n}(R^2 + \sigma^2 n).
\end{equation}
Combining~\cref{eq:lds-chain-rule,eq:lds-second-moment-sum}, we get
\[
\sup_{A \in \ball_\infty(A^\star, \eps)} \KL{\dist^{A^\star}(\cdot \mid X_0)}{\dist^{A}(\cdot \mid X_0)} \le \frac{\eps^2 n^2 C_n^2 T^{2n}(R^2 + \sigma^2 n)}{\sigma^2}.
\]
Thus,
\[
\KL{\dist^{A^\star}(\cdot \mid X_0)}{\dist(\cdot \mid X_0)} \le n^2\log(B/\eps) + O(n^2) + \frac{\eps^2 n^2 C_n^2 T^{2n}(R^2 + \sigma^2 n)}{\sigma^2}.
\]
Balancing in the above by choosing $\eps^2 = \sigma^2/(n^2 C_n^2 T^{2n}(R^2 + \sigma^2 n))$ gives:
\[
\KL{\dist^{A^\star}(\cdot \mid X_0)}{\dist(\cdot \mid X_0)} = O\left(n^2 \log\left(\frac{B C_n T^n (R + \sigma\sqrt{n})}{\sigma}\right)\right) = O\left(n^3 \log\left(nBT(R/\sigma + 1)\right)\right),
\] as desired.
\end{proof}

\begin{lemma}\label{lemma:matrix-power-growth}
For any $A \in \bbR^{n \times n}$ with $\rho(A) \le 1$ and any $s \ge 1$,
\[
\norm{A^s}_{\mathrm{op}} \le n \cdot (\norm{A}_{\mathrm{op}} + 1)^{n-1} \cdot s^{n-1}.
\]
\end{lemma}
\begin{proof}
We consider the Schur decomposition of $A$, namely, $A = QUQ^*$, where $Q \in \bbC^{n \times n}$ is unitary and $U \in \bbC^{n \times n}$ is upper triangular with the eigenvalues of $A$ on its diagonal. Since $Q$ is unitary, $\norm{A^s}_{\mathrm{op}} = \norm{U^s}_{\mathrm{op}}$ and $\norm{U}_{\mathrm{op}} = \norm{A}_{\mathrm{op}}$. Decompose $U = D + N$ where $D$ is the diagonal of $U$ and $N$ is strictly upper triangular. Since $\rho(A) \le 1$, the diagonal entries of $D$ have magnitude at most $1$, so $\norm{D}_{\mathrm{op}} \le 1$ and $\norm{N}_{\mathrm{op}} \le \norm{U}_{\mathrm{op}} + \norm{D}_{\mathrm{op}} \le \norm{A}_{\mathrm{op}} + 1$.

Now, we expand $U^s = (D+N)^s$ as follows:
\[
U^s = \sum_{j=0}^{s} \sum_{\substack{a_0, \ldots, a_j \ge 0 \\ a_0 + \cdots + a_j = s - j}} D^{a_0} N D^{a_1} N \cdots N D^{a_j}.
\]
We observe that any term above corresponding to $j \ge n$ vanishes. Indeed, since $N$ is strictly upper triangular, and $D$ is diagonal, the product of $\ge n$ matrices $N$ and any number of matrices $D$ in any order is equal to zero. Thus, the above sum can be rewritten in the following form:
\[
U^s = \sum_{j=0}^{n-1} \sum_{\substack{a_0, \ldots, a_j \ge 0 \\ a_0 + \cdots + a_j = s - j}} D^{a_0} N D^{a_1} N \cdots N D^{a_j}.
\]
Now, each summand above has operator norm at most $\norm{D}_{\mathrm{op}}^{s-j} \norm{N}_{\mathrm{op}}^j \le (\norm{A}_{\mathrm{op}} + 1)^j$. The number of compositions $(a_0, \ldots, a_j)$ with sum $s - j$ is $\binom{s}{j}$, and $\binom{s}{j} \le s^j \le s^{n-1}$ for $j \le n - 1$. Therefore
\[
\norm{U^s}_{\mathrm{op}} \le \sum_{j=0}^{n-1} \binom{s}{j}(\norm{A}_{\mathrm{op}} + 1)^j \le s^{n-1} \sum_{j=0}^{n-1} (\norm{A}_{\mathrm{op}} + 1)^j \le n \cdot (\norm{A}_{\mathrm{op}} + 1)^{n-1} \cdot s^{n-1},
\]
which concludes the proof.
\end{proof}

\begin{proposition}
    \label{prop:lds-cond-sampler}
    For any $X_0, \seq X_{1:t}$, exact conditional sampling from $\dist(\cdot \mid \seq X_{1:t}, X_0)$, where $\dist$ is as in~\cref{lemma:lds-compressor}, can be implemented in $\poly(n, T)$ time per sample.
\end{proposition}
\begin{proof}
We want to sample $\seq X_{t+1:T} \sim \dist(\cdot \mid \seq X_{1:t}, X_0)$. Recall that $\dist$ is defined as a mixture over the prior $\pi$:
\[
\dist(\cdot \mid X_0) = \En_{A \sim \pi} \dist^A (\cdot \mid X_0)
\]
Then,
\[
\dist(\cdot \mid \seq X_{1:t}, X_0) = \En_{A \sim \pi(\cdot \mid X_{1:t})} \dist^A (\cdot \mid \seq X_{1:t}, X_0).
\]
That is, to sample $\seq X_{t+1:T} \sim \dist(\cdot \mid \seq X_{1:t}, X_0)$, it suffices to sample $A$ from the posterior $\pi(\cdot \mid X_{1:t})$, and subsequently sample from $\dist^A (\cdot \mid \seq X_{1:t}, X_0).$ The second part can clearly be done efficiently. Thus, the problem reduces to sampling from the posterior $\pi(\cdot \mid X_{1:t})$.

By the Markov property and the Gaussian transition $\dist^{A}(X_{s+1} \in \cdot \mid X_s = x) = \cN(A x, \sigma^2 I_n)$, we have
\[
\log \pi(A \mid \seq X_{1:t}, X_0) = \log \pi(A) - \frac{1}{2\sigma^2} \sum_{s=0}^{t-1} \norm{X_{s+1} - A X_s}_2^2 + \mathrm{const}.
\]
Both terms on the right-hand side are quadratic in $A$, so the posterior is itself a Gaussian on $\bbR^{n \times n}$ with mean and covariance computable in closed form, and exact sampling is trivial.
\end{proof}

\ThmLDS*
\begin{proof}
We instantiate~\cref{thm:cond-oracle-eff} with the simulator $\dist$ defined in~\cref{sec:lds}. By~\cref{lemma:lds-compressor}, \[\sup_{\substack{A^\star \in \Theta,\, \rho(A^\star) \le 1, \\ X_0 \in \ball(0, R)}}\KL{\dist^{A^\star}(\cdot \mid X_0)}{\dist(\cdot \mid X_0)} = O\left(n^3 \log\left(nBT(R/\sigma + 1)\right)\right),\] and by~\cref{prop:lds-cond-sampler} the conditional sampler runs in $\poly(n, T)$ time per round. Substituting into~\cref{thm:cond-oracle-eff} gives expected regret
\[
\En \reg(\cF, T) \lesssim \sqrt{dT} \left(1 + \sqrt{n^3 \log\left(nBT(R/\sigma + 1)\right)} \right)
\]
in $\poly(n, T)$ time per round, as claimed.
\end{proof}

\subsection{Proofs from~\texorpdfstring{\cref{sec:glauber}}{Section \getrefnumber{sec:glauber}}}

Throughout this subsection, we use the notation introduced in~\cref{sec:glauber}.
\begin{lemma}
    \label{lemma:glauber-compressor}
    For any $T \ge 2$, $\theta^\star \in \Theta$, and $X_0 \in \{0,1\}^n$, we have
    \[
    \KL{\dist^{\theta^\star}(\cdot \mid X_0)}{\dist(\cdot \mid X_0)} \lesssim p \log(BTn).
    \]
\end{lemma}
\begin{proof}
Fix $\theta^\star \in \Theta$ and $X_0 \in \{0,1\}^n$ arbitrarily, and let $\varepsilon \in (0, B]$. By the Markov property, the trajectory likelihood factors as
\[
\dist^{\theta}(\seq I, \seq X \mid X_0) = \prod_{t=1}^T \dist^{\theta}(I_t, X_t \mid X_{t-1}).
\]
Define the localized neighborhood $\ball_\infty(\theta^\star, \varepsilon) := \{\theta \in \bbR^p \colon \norm{\theta - \theta^\star}_\infty \le \varepsilon\}$, so $\pi(\ball_\infty(\theta^\star, \varepsilon) \cap \Theta) \ge (\varepsilon/2B)^p$. By a standard localization argument (e.g.,~\citep{polyanskiy2025information}),
\begin{align*}
\KL{\dist^{\theta^\star}(\cdot \mid X_0)}{\dist(\cdot \mid X_0)}
&\le \log\frac{1}{\pi(\ball_\infty(\theta^\star, \varepsilon) \cap \Theta)} + \sup_{\substack{\theta \in \ball_\infty(\theta^\star, \varepsilon) \cap \Theta, \\ (\seq I, \seq X)}}\log\frac{\dist^{\theta^\star}(\seq I, \seq X \mid X_0)}{\dist^{\theta}(\seq I, \seq X \mid X_0)} \\
&\le p\log(2B/\varepsilon) + \sup_{\substack{\theta \in \ball_\infty(\theta^\star, \varepsilon) \cap \Theta, \\ (\seq I, \seq X)}} \log\frac{\dist^{\theta^\star}(\seq I, \seq X \mid X_0)}{\dist^{\theta}(\seq I, \seq X \mid X_0)}.
\end{align*}
For the last term, by the Markov factorization and additivity of $\log$,
\[
\log\frac{\dist^{\theta^\star}(\seq I, \seq X \mid X_0)}{\dist^{\theta}(\seq I, \seq X \mid X_0)} = \sum_{t=1}^T \log\frac{\dist^{\theta^\star}(I_t, X_t \mid X_{t-1})}{\dist^{\theta}(I_t, X_t \mid X_{t-1})}.
\]
The one-step kernel can be written as
\[
\dist^{\theta}(I_t, X_t \mid X_{t-1}) = \frac{1}{n} \cdot \sigma(\glfield_{I_t}^{\theta}(X_{t-1}))^{X_{t, I_t}}\left(1 - \sigma(\glfield_{I_t}^{\theta}(X_{t-1}))\right)^{1 - X_{t, I_t}} \cdot \ind\{X_{t, j} = X_{t-1, j}\ \forall j \neq I_t\}.
\]
Then, after some algebra, the per-step log-ratio can be written as:
\begin{multline*}
\log\frac{\dist^{\theta^\star}(I_t, X_t \mid X_{t-1})}{\dist^{\theta}(I_t, X_t \mid X_{t-1})} = X_{t, I_t}\left(\glfield_{I_t}^{\theta^\star}(X_{t-1}) - \glfield_{I_t}^{\theta}(X_{t-1})\right) \\
- \log\left(1 + \exp(\glfield_{I_t}^{\theta^\star}(X_{t-1}))\right) + \log\left(1 + \exp(\glfield_{I_t}^{\theta}(X_{t-1}))\right).
\end{multline*}
Summing over $t$ and bounding by the supremum over $(i, x', x)$,
\[
\log\frac{\dist^{\theta^\star}(\seq I, \seq X \mid X_0)}{\dist^{\theta}(\seq I, \seq X \mid X_0)} \le \sum_{t=1}^T \sup_{(i, x', x)}\left[x'_{i}(\glfield_{i}^{\theta^\star}(x) - \glfield_{i}^{\theta}(x)) - \log(1 + \exp(\glfield_{i}^{\theta^\star}(x))) + \log(1 + \exp(\glfield_{i}^{\theta}(x)))\right].
\]
Since $\lambda \mapsto \log(1 + \exp(\lambda))$ is $1$-Lipschitz, $x'_i \in \{0, 1\}$, and $|\glfield_{i}^{\theta^\star}(x) - \glfield_{i}^{\theta}(x)| \le n\norm{\theta^\star - \theta}_\infty \le n\varepsilon$ for $\theta \in \ball_\infty(\theta^\star, \varepsilon)$, each per-step term is $\lesssim n\varepsilon$, so the sum is $\lesssim Tn\varepsilon$. Combining,
\[
\KL{\dist^{\theta^\star}(\cdot \mid X_0)}{\dist(\cdot \mid X_0)} \lesssim p\log(B/\varepsilon) + Tn\varepsilon.
\]
We now choose $\varepsilon = B \land p/(Tn)$. If $p/(Tn) \le B$, then $\varepsilon = p/(Tn)$ and the bound becomes $p\log(BTn/p) + p \lesssim p\log(BTn)$. Otherwise $\varepsilon = B$, in which case the first term vanishes and $Tn\varepsilon = TnB \le p \lesssim p\log(BTn)$ (using $p/(Tn) > B$). In either case we obtain the desired bound.
\end{proof}

\begin{proposition}
    \label{prop:glauber-cond-sampler}
    For any $X_0 \in \{0,1\}^n$, conditional sampling from $\dist(\cdot \mid \seq I_{1:t}, \seq X_{1:t}, X_0)$, where $\dist$ is as in~\cref{lemma:glauber-compressor}, can be implemented in $\poly(n, T, B, 1/\delta)$ time per sample, to within total-variation error $\delta$.
\end{proposition}
\begin{proof}
We want to sample $(\seq I_{t+1:T}, \seq X_{t+1:T}) \sim \dist(\cdot \mid \seq I_{1:t}, \seq X_{1:t}, X_0)$. Recall that $\dist$ is defined as a mixture over the prior $\pi$:
\[
\dist(\cdot \mid X_0) = \En_{\theta \sim \pi} \dist^\theta(\cdot \mid X_0).
\]
Then,
\[
\dist(\cdot \mid \seq I_{1:t}, \seq X_{1:t}, X_0) = \En_{\theta \sim \pi(\cdot \mid \seq I_{1:t}, \seq X_{1:t}, X_0)} \dist^\theta(\cdot \mid \seq I_{1:t}, \seq X_{1:t}, X_0).
\]
That is, to sample $(\seq I_{t+1:T}, \seq X_{t+1:T}) \sim \dist(\cdot \mid \seq I_{1:t}, \seq X_{1:t}, X_0)$, it suffices to sample $\theta$ from the posterior $\pi(\cdot \mid \seq I_{1:t}, \seq X_{1:t}, X_0)$, and subsequently sample from $\dist^\theta(\cdot \mid \seq I_{1:t}, \seq X_{1:t}, X_0)$. The second part can clearly be done efficiently by forward-simulating Glauber dynamics from $X_t$ for $T-t$ steps, taking $O((T-t)n)$ time. Thus, the problem reduces to sampling from the posterior $\pi(\cdot \mid \seq I_{1:t}, \seq X_{1:t}, X_0)$.

Writing out the log-density of the posterior on $\Theta$, we have
\[
\log \pi(\theta \mid \seq I_{1:t}, \seq X_{1:t}, X_0) = \sum_{s=1}^{t} X_{s, I_s} \log \sigma(\glfield_{I_s}^\theta(X_{s-1})) + (1 - X_{s, I_s}) \log \left(1 - \sigma(\glfield_{I_s}^\theta(X_{s-1}))\right) + \mathrm{const}.
\]
The local field $\glfield_{i}^\theta(x) = \sum_{j \ne i} \glmat_{ij} x_j + \glvec_i$ is linear in $\theta = (\glmat, \glvec)$, so we can write $\glfield_{I_s}^\theta(X_{s-1}) = \langle \phi_s, \theta\rangle$ where $\phi_s \in \{0,1\}^p$ has at most $n$ non-zero entries (in particular, $\norm{\phi_s}_2 \le \sqrt{n}$). Using $\frac{\mathrm{d}}{\mathrm{d}\lambda}\log\sigma(\lambda) = 1 - \sigma(\lambda)$ and $\frac{\mathrm{d}}{\mathrm{d}\lambda}\log(1-\sigma(\lambda)) = -\sigma(\lambda)$, the gradient and Hessian of the log-posterior are
\[
\nabla_\theta \log \pi = \sum_{s=1}^t \left(X_{s, I_s} - \sigma(\langle \phi_s, \theta\rangle)\right)\phi_s, \qquad \nabla^2_\theta \log \pi = -\sum_{s=1}^t \sigma(\langle \phi_s, \theta\rangle)\left(1 - \sigma(\langle \phi_s, \theta\rangle)\right)\phi_s \phi_s^\top.
\]
Since $\sigma(1 - \sigma) \le 1/4$ and $\phi_s \phi_s^\top \preceq \norm{\phi_s}_2^2 \cdot I_p \le n \cdot I_p$, the triangle inequality for the operator norm gives
\[
\norm{\nabla^2_\theta \log \pi}_{\mathrm{op}} \le \sum_{s = 1}^t \frac{1}{4} \norm{\phi_s}_2^2 \le \frac{Tn}{4}.
\]
In particular, $\nabla^2_\theta \log\pi \preceq 0$, so the log-posterior is concave on $\bbR^p$ and the posterior is log-concave restricted to the bounded convex body $\Theta = [-B,B]^p$ (of Euclidean diameter at most $2B\sqrt{p}$). Thus, results for log-concave sampling with smooth densities~\citep{altschuler2022resolving} imply that we can sample from the above up to error $\delta$ in $\poly(p,B,T,1/\delta)$ time, as desired.
\end{proof}

\ThmGlauber*
\begin{proof}
For $T = 1$ the regret is trivially bounded by $1$, so assume $T \ge 2$. We instantiate~\cref{thm:cond-oracle-eff} with the simulator $\dist$ defined in~\cref{sec:glauber}. By~\cref{lemma:glauber-compressor}, \[\sup_{\theta^\star \in \Theta,\, X_0 \in \{0,1\}^n}\KL{\dist^{\theta^\star}(\cdot \mid X_0)}{\dist(\cdot \mid X_0)} \lesssim p \log(BTn) \lesssim n^2 \log(BTn).\] By~\cref{prop:glauber-cond-sampler}, for any $\delta > 0$ there is a conditional sampler running in $\poly(n, T, B, 1/\delta)$ time per call that samples to within total-variation distance $\delta$. The algorithm of~\cref{thm:cond-oracle-eff} issues $M = \poly(T)$ such calls over the horizon; coupling each approximate draw with an exact one, the two executions agree except with probability at most $M\delta$, so the expected regret changes by at most $2TM\delta$. Taking $\delta = 1/\poly(T)$ small enough makes this $O(1)$ while keeping the per-call running time $\poly(n, T, B)$. Substituting into~\cref{thm:cond-oracle-eff}, we obtain:
\[
\En \reg(\cF, T) \lesssim \sqrt{dT}\left(1 + \sqrt{n^2 \log(BTn)}\right),
\]
in $\poly(n, T, B)$ time per round, as claimed.
\end{proof}

\section{Learning with Adaptive Labels}
\label{sec:adaptive-labels}
\label{app:adaptive-labels}

\subsection{Discussion on~\texorpdfstring{\cref{assumption:f-star}}{Assumption \getrefnumber{assumption:f-star}}}

\label{sec:assumptions}

\cref{assumption:f-star}, intuitively, formalizes the following belief: nature's play may be complex, but it is not adversarial. More concretely, we assume that the data generated by nature may be non-i.i.d., but the ``ground truth'' $f^\star$ does not adversarially change over time or based on actions of the learner. From this point of view~\cref{assumption:f-star} can be seen as a modeling assumption that is suitable for many online learning use cases.

On the other hand,~\cref{assumption:f-star} is motivated by the fact that~\cref{assumption:p-star} alone is not sufficient to recover learnability of VC classes if the labels are allowed to depend on the realized covariates $(X_1,\ldots,X_T)$ in an arbitrary manner. Concretely, even with exact knowledge of the distribution of $(X_1,\ldots,X_T)$, it is possible to replicate the standard Littlestone tree hardness construction for learning thresholds. It suffices to let $\dist^\star$ pick out the uniform path in the Littlestone tree, and let $Y_t = 1$ if $X_{t+1}$ is the right child of $X_t$ and $0$ otherwise. This can be seen to force linear regret for the learner.

At the same time, milder forms of dependence of $Y_t$ on $X_{1:T}$ can be allowed if we have conditional sampling access. In particular, all of our positive results for this setting can be extended to the scenario where $Y_t$ is allowed to depend on $X_{1:t}$ but not on $X_{t+1:T}$.

In this section, we keep~\cref{assumption:p-star} and discuss alternative assumptions about the dependence on labels on realization of the covariates. First, we show that allowing \emph{arbitrary} dependence of labels on covariate realization still suffers from classical impossibility results in online learning literature, and essentially demands bounded Littlestone dimension for sublinear regret.

\subsection{Impossibility of Learning with Arbitrary Dependence}

Suppose that each label $Y_t$ is allowed to depend in an arbitrary fashion on the realization of $\seq X \sim \dist^\star$. Then, any class with infinite Littlestone dimension is not learnable.
\begin{proposition}
    Consider any class $\cF$ with infinite Littlestone dimension, and consider the Littlestone tree of depth $T$. Let $\dist^\star$ be a distribution that samples a uniform path $\seq X$ in the Littlestone tree, and, for each $t \le T-1$, set $Y_t := 1$ iff $X_{t+1}$ is the right child of its parent (and sample $Y_T$ randomly). This strategy of nature forces any learner to have $\ge T/2$ regret.
\end{proposition}
\begin{proof}
    We observe that the sampling procedure used by nature is equivalent to the one in~\cref{prop:littlestone-hardness}. Thus, any learner suffers $\ge T/2$ expected regret.
\end{proof}

We point out that, since we fix $\dist^\star$ in the hardness construction, the learner ``knows'' $\dist^\star$, so~\cref{assumption:p-star} and~\cref{assumption:p-star-cond} both hold.

\subsection{Labels Dependent on the History}

Another form of dependence we consider is allowing the labels to depend on \emph{past} covariates only, and, possibly, the actions of the learner. More formally, letting
\[
\cH_{t} := (\seq X_{1:t}, \seq Y_{1:t-1}, \hat {\seq Y}_{1:t-1})
\]
denote the history up to time $t$, we allow the distribution $Y_t$ to depend on $\cH_t$ but not on future covariates. Equivalently, the conditional law of future covariates must be conditionally independent from the realization of $Y_t$. We formalize this assumption as follows.
\begin{assumption}
    \label{assumption:adaptive-labels}
    For every $t \in [T]$, the distribution $\seq X_{t+1:T} \mid (\cH_t, Y_t, \hat Y_t)$ is given by $\dist^\star(\cdot \mid \seq X_{1:t})$.
\end{assumption}
In this situation, the picture is more complex. It turns out, even with $\dist^\star = \dist$, learning with unconditional samples is \emph{statistically impossible} here. However, with conditional sampling access to $\dist$, our positive results in~\cref{thm:cond-oracle-eff} still hold.

\begin{proposition}
\label[proposition]{prop:adaptive-labels-impossibility}
    Let $\cF$ be the class of thresholds on the $[0,1]$ interval. For any $N \ge 1$, and for any learning algorithm that draws at most $N$ unconditional samples, there exists a choice of a distribution $\dist^\star$, a choice of realizable labels satisfying~\cref{assumption:adaptive-labels} such that
    \[
    \En \left[\sum_{t=1}^T \ind\{f^{\star}(X_t) \neq \hat{y}_t\} - \inf_{f \in\cF} \sum_{t=1}^T  \ind\{f^{\star}(X_t) \neq f(X_t) \} \right] \ge T/8.
    \]
\end{proposition}
\begin{proof}
    The hardness construction we give here is the same as the one given in~\cref{thm:agnostic-impossibility}. WLOG, assume $N \ge 2$. Divide the interval $[0,1]$ into $N$ sub-intervals $\{A_{i}\}_{i \in [N]}$, and consider $N$ Littlestone trees in each respective interval. For each $i \in[N]$, let $\rho_i$ be a distribution that samples a random path in the Littlestone tree in interval $A_i$. For each $i \in [N],$ sample a path $\seq x^i \sim \rho_i$. Now, let $\dist^\star$ be defined in the following way. Set
    \[
    \dist^\star := \frac{1}{N} \sum_{i \in [N]} \delta_{\seq x^i}.
    \]
    On each interval $A_i$, $i \in [N]$, let $f_i$ be a threshold function consistent with the path $\seq x^i$ (see~\cref{def:littlestone-consistent-function}). Let $i_\star$ be a random index such that $\seq X = \seq x^{i_\star}$. Note that $i_\star$ is $X_1$-measurable. Then, for each $t \ge 1$, we may set $Y_t := f_{i_\star}(X_t)$. Then, with probability $1$ over the draw from $\seq X \sim \dist^\star$, the labels $\seq Y_{1:T}$ are realizable w.r.t. $\cF$ on $\seq X$, that is, with probability $1$:
    \[
    \inf_{f \in\cF} \sum_{t=1}^T \ind\{f^{\star}(X_t) \neq f(X_t)\} = 0.
    \]
    At the same time, with probability at least
    \[
    \left(1 - \frac{1}{N}\right)^{N} \ge \frac{1}{4},
    \]
    none of the samples the learner draws are equal to $\seq x^{i_\star}$. Let us condition on this event. Then, for any $t \in[T]$ the randomness of $\seq x^{i_\star}$ is independent of the randomness of the learner in round $t$, that is, learner suffers expected regret of $T/2$ by standard Littlestone hardness result (see~\cref{prop:littlestone-hardness}). Thus,
    \[
    \En \left[\sum_{t=1}^T  \ind\{f^{\star}(X_t) \neq \hat{y}_t\} - \inf_{f \in\cF} \sum_{t=1}^T  \ind\{f^{\star}(X_t) \neq f(X_t)\} \right] \ge \frac{1}{4} \cdot \frac{T}{2} = \frac{T}{8},
    \]
    as desired.
\end{proof}

In contrast, all positive results from~\cref{sec:cond} still hold in this setting. Indeed, since the proof of~\cref{thm:cond-oracle-eff} is phrased in terms of relaxations, it permits the dependence of $Y_t$ on $\cH_t$. Thus, the following result still holds in this setting.
\begin{theorem}
\label{thm:cond-oracle-eff-adaptive-labels}
Let $\cF$ be a class of VC dimension $d$. Suppose the labels satisfy~\cref{assumption:adaptive-labels}. Then, there exists an algorithm that, given access to an agnostic ERM oracle over $\cF$ and a conditional sampling oracle $\condoracle_{\dist}$, achieves expected regret bound:
    \[
    \En \reg(\cF,T) \lesssim \sqrt{dT (1 + \KL{\dist^\star}{\dist})},
    \]
    in both realizable and agnostic settings. Moreover, the algorithm runs in $\poly(T)$ time.
\end{theorem}

\end{document}